\documentclass[11pt]{article}
\usepackage{geometry} 
\usepackage{natbib, algorithm, algorithmic,hyperref}

\usepackage{microtype}
\usepackage{graphicx}
\usepackage{booktabs} 

\usepackage{graphicx} 
\usepackage{subcaption} 
\usepackage{multirow}
\usepackage{mathtools}
\usepackage{relsize}
\usepackage{url}
\usepackage{nicefrac}
\usepackage{xspace}
\usepackage{algorithm}
\usepackage{algorithmic}
\usepackage{multicol}
\usepackage{float}
\usepackage{bbm}

\usepackage[dvipsnames]{xcolor}
\usepackage{pifont}
\newcommand{\cmark}{{\color{PineGreen}\ding{51}}}%
\newcommand{\xmark}{{\color{BrickRed}\ding{55}}}%

\usepackage[utf8]{inputenc} 
\usepackage[T1]{fontenc}    
\usepackage{hyperref}       
\usepackage{url}            

\usepackage{nicefrac}       
\usepackage{microtype}      
 
\usepackage{layout}
\usepackage{multirow}

\usepackage{amsfonts}
\usepackage{amsmath}
\usepackage{amsthm}
\usepackage{amssymb} 

\usepackage{mdframed} 
\usepackage{thmtools}

\usepackage[font=normalsize, labelfont=bf]{caption}

\usepackage{makecell}
\usepackage{colortbl}
\definecolor{bgcolor}{rgb}{1, 0.96, 0.80}
\definecolor{bgcolor2}{rgb}{1, 0.96, 0.80}

\usepackage[dvipsnames]{xcolor}
\usepackage{pifont}

\newcommand{\makecellnew}[1]{{\renewcommand{\arraystretch}{0.8}\begin{tabular}{c} #1 \end{tabular}}}

\graphicspath{{plots/}}

\usepackage{amsmath,amsfonts,bm,mathtools}

\newcommand{\eqdef}{\coloneqq} 
\newcommand{\cO}{{\cal O}}
\newcommand{\cC}{{\cal C}}
\newcommand{\cD}{{\cal D}}

\newcommand{\R}{\mathbb{R}}

\newcommand{\E}{\mathbb{E}}

\newcommand{\cL}{\mathcal{L}}

\newcommand{\mI}{\mathrm{I}}

\newcommand{\ML}{{\bf L}}
\newcommand{\MI}{{\bf I}}
\newcommand{\Mdiag}{{\bf Diag}}
\newcommand{\MP}{{\bf P}}
\newcommand{\MC}{{\bf C}}
\newcommand{\MM}{{\bf M}}
\newcommand{\ME}{{\bf E}}
\newcommand{\MD}{{\bf D}}
\newcommand{\MU}{{\bf U}}
\newcommand{\MA}{{\bf A}}
\newcommand{\MS}{{\bf S}}
\newcommand{\MB}{{\bf B}}
\newcommand{\MZ}{{\bf Z}}
\newcommand{\MQ}{{\bf Q}}

\def\Pr{{\rm Prob}}

\def\prox{{\rm prox}}

\DeclareMathOperator*{\argmin}{arg\,min}

\DeclareMathOperator*{\diag}{{\rm Diag}}
\DeclareMathOperator*{\rank}{{\rm rank}}

\DeclareMathOperator*{\tr}{{\rm tr}}
\DeclareMathOperator*{\range}{{\rm Range}}

\newcommand{\ve}[2]{\langle #1 , #2 \rangle}
\newcommand{\overbar}[1]{\mkern 1.5mu\overline{\mkern-1.5mu#1\mkern-1.5mu}\mkern 1.5mu}
\newcommand{\half}{{\nicefrac{1}{2}}}
\newcommand{\phalf}{{\dagger\nicefrac{1}{2}}}

\def\<{\left\langle}
\def\>{\right\rangle}
\def\[{\left[}
\def\]{\right]}
\def\({\left(}
\def\){\right)}

\theoremstyle{plain}
\newtheorem{theorem}{Theorem}  
\newtheorem{lemma}[theorem]{Lemma} 
\newtheorem{proposition}[theorem]{Proposition} 
\newtheorem{remark}{Remark} 

\theoremstyle{definition}
\newtheorem{assumption}{Assumption} 
\newtheorem{definition}{Definition}

\theoremstyle{remark}

\begin{document}

\title{\bf Smoothness Matrices Beat Smoothness Constants: \\ Better  Communication Compression Techniques for Distributed Optimization}

\author{Mher Safaryan$^1$ \qquad Filip Hanzely$^2$ \qquad Peter Richt\'arik$^1$ \qquad \\
\phantom{XXX} \\
 $^1$King Abdullah University of Science and Technology (KAUST), Thuwal, Saudi Arabia \\
 $^2$Toyota Technological Institute at Chicago (TTIC), Chicago, USA
}

\date{February 14, 2021}

\maketitle

\begin{abstract}
Large scale distributed optimization has become the default tool for the training of supervised machine learning models with a large number of parameters and training data. Recent advancements in the field provide several mechanisms for speeding up the training, including {\em compressed communication}, {\em variance reduction} and {\em acceleration}. 
However, none of these methods is capable of exploiting the inherently rich data-dependent smoothness structure of the local losses beyond standard smoothness constants. In this paper, we argue that when training supervised models,  {\em smoothness matrices}---information-rich generalizations of the ubiquitous smoothness constants---can and should be exploited for further dramatic gains, both in theory and practice. In order to further alleviate the communication burden inherent in distributed optimization, we propose a novel communication sparsification strategy that can take full advantage of the smoothness matrices associated with local losses. To showcase the power of this tool, we describe how our sparsification technique can be adapted to three distributed optimization algorithms---DCGD \citep{KFJ}, DIANA \citep{MGTR} and ADIANA \citep{AccCGD}---yielding significant savings in terms of communication complexity.  The new methods always outperform the baselines, often dramatically so.
\end{abstract}

\tableofcontents

\section{Introduction}\label{sec:intro}

With the desire to build and train high quality machine learning models comes an increased appetite for larger models, both in terms of the number of parameters encoding them, and in the amount of data required to train them. In the big data regime, the data is partitioned among many parallel machines, which then cooperatively train  a single global model, usually orchestrated by a central server. Distributed training is cast as the distributed optimization problem \begin{equation}\label{main-opt-problem-dist}
\min\limits_{x\in\R^d} f(x) + R(x), \qquad  f(x)\eqdef \frac{1}{n}\sum \limits_{i=1}^n f_i(x),
\end{equation}
where $d$ is the number of parameters of model $x\in\R^d$,  $n$ is the number of machines participating in the training, $f_i(x)$ is the loss associated with the data stored on machine $i\in[n] \eqdef \{1,2,\dots,n\}$, $f(x)$ is the empirical loss, and $R(x)$ is a regularizer. 
Ample research over the past two decades has shown that first-order methods are highly scalable and as a result are the methods of choice for distributed optimization problems \citep{FOM-in-DS}. In particular, a substantial amount of work has been devoted to speeding up the training process by developing efficient methods empowered with techniques such as {\em compressed communication}, {\em variance reduction} and {\em acceleration}.

\subsection{Compressed communication}
In distributed training, compute nodes have to communicate with each other, often via a central server, in order to be able to maintain consensus and jointly train a single global model. However,  communication of the information pertaining to local progress, which is typically contained in gradient(s) distilled from local data, is almost invariably {\em the} key bottleneck in distributed training systems \citep{grace}.  One  popular way to address this issue is to reduce the number of bits encoding the vector/tensor to be transferred via the help of a lossy {\em compression operator}. Numerous {\em unbiased} gradient compression operators have been proposed for this purpose, including several types of sparsifications \citep{WSLCPW,99KFP,alistarh2018convergence} and quantizations \citep{AGLTV,ZLKALZ,HHHSCR,EC-QSGD}. 
Certain (classes of) {\em biased} compression operators have been proposed as well, including low-rank approximation \citep{vogels}, sign-based compressors \citep{SFDLY,BWAA,sign_descent_2019} and contractive compressors \citep{karimireddy2019error, stich2019, DoubleSqueeze2019, biased2020, GKKR}.

\subsection{Variance reduction}
A marked issue that needs to be addressed by successful distributed optimization methods has to do with the (potential) ``dissimilarity'' of the local loss functions $f_1,\dots,f_n$, which in turn is due to the heterogeneity of the training data defining these functions. The higher the dissimilarity, the harder it is for the devices to find the minimizer of \eqref{main-opt-problem-dist}. This issue exists even  in the unregularized case ($R\equiv 0$). Indeed, while in this case $\frac{1}{n} \sum_i \nabla f_i(x^*)=0$ if $x^*$ is a minimizer of $f$, this does not mean that the individual gradients, $\nabla f_1(x^*), \dots, \nabla f_n(x^*)$, are all zero. This shows that local gradient information alone  is not enough for any node to ``realize'' that a solution has been found, which encourages further, in this case unnecessary, iterations.  If unaddressed properly, an algorithm is forced to use smaller learning rates, and this leads to unnecessarily slow convergence. On the other hand, when a fixed learning rate is used, the rate is fast, but convergence stops in a potentially large neighborhood\footnote{In the $R\equiv 0$ case, this neighborhood is proportional to the {\em variance of the local gradients at the optimum}: $\frac{1}{n}\sum_{i=1}^n \|\nabla f_i(x^*)\|^2$.} of the optimum $x^*$. This issue is exacerbated further by the extra noise coming from gradient compression. Indeed, this noise prevents methods such as Distributed Compressed Gradient Descent (DCGD) \citep{KFJ} from converging to $x^*$ with a constant learning rate  even in the interpolation regime characterized by the identities $\nabla f_i(x^*) = 0$ for all $i$. Fortunately, these issues can be resolved via carefully designed variance reduction techniques \citep{VR-Review2020}. In particular, the first variance reduction mechanism for removing the variance coming from compression operators in distributed training is due to \citet{MGTR}, embodied in their DIANA algorithm.  The method was initially analyzed for ternary quantization only \citep{WXYWWCL}, and later generalized to handle a general class of unbiased compression operators \citep{DIANA-VR,GKKR}.

\subsection{Acceleration}
To speed up distributed training even further, it is often possible to employ Nesterov's acceleration technique \citep{NestAcc, Nest-ILCO} in concert with gradient compression and variance reduction. For instance, \citet{AccCGD} developed the ADIANA method, which adds acceleration on top of a variant of DIANA that relies on the computation of  full-batch gradients on all nodes. The resulting method offers provable speedups in convex and strongly convex regimes. Another example is the method ECLK of \citet{Qian2020ErrorCD}, which
employs compressed communication via any (possibly biased) compressor satisfying a certain contraction property in combination with a slightly different variance reduction technique known as error compensation \citep{stich2019, karimireddy2019error}, while acceleration is offered by a loopless variant of the accelerated method Katyusha \citep{Katyusha,L-Katyusha}.

\subsection{Further tricks}
Numerous other techniques are often used to improve some other aspects of  distributed training, including implementing multiple local gradient steps before communication \citep{Stich-LocalSGD,SCAFFOLD,WPSDBMSS}, asynchronous communication protocols \citep{AgarwalDuchi,LHLL,RRWN}, in-network aggregation \citep{switchML}, and performing the distributed training in a decentralized peer-to-peer manner without the reliance on an orchestrating server \citep{KSJ,AYS}. However, in this work, we do not explore these directions and focus on the three techniques described before, namely, compressed communication, variance reduction and acceleration.

\section{Mining for Smoothness Information}

\subsection{One size fits all}
Arguably, one of the most ubiquitous, if not {\em the} most ubiquitous, assumptions used in the literature on first-order optimization methods is that of {\em $L$-smoothness} \citep{Nest-ILCO}. A differentiable function $\phi:\R^d\to \R$ is said to be $L$-smooth if there exists a constant $L\geq 0$ such that  \begin{equation}\label{eq:L-smoothness}\phi(x) \leq \phi(y) + \langle \nabla \phi(x), x-y \rangle + \frac{L}{2} \|x-y\|^2\end{equation} holds for all $x,y\in \R^d$. However, most works in the area of finite-sum distributed optimization use it very crudely: they assume that all local loss functions $f_i$ as well as their average, $f=\frac{1}{n}\sum_i f_i$, share the same smoothness constant $L$ \citep{DoubleSqueeze2019,MB-vs-Local-SGD,Stich-LocalSGD}. This is crude because much information is lost this way. Indeed, assuming that each $f_i$ is $L_i$-smooth, it is well known that $f$ is $L_f$-smooth with $L_f$ satisfying the inequality $L_f \leq \frac{1}{n}\sum_i L_i$. In the light of this, the above assumption is crude as it effectively replaces the values $L_1, \dots, L_n$ and $L_f$ with a single parameter $L$ satisfying $L \geq \max \{L_1, \dots, L_n\}$. Since the stepsizes and convergence rates of first-order methods depend on the smoothness constant(s) employed, convergence analysis relying on such crude approximation may be significantly suboptimal, and the methods too slow when implemented following the theory.

\subsection{``According to the work of their hands'' (Lam 3:64)}
Significant theoretical and practical improvement can often be obtained when taking account of all the smoothness constants involved, avoiding the practice of replacing them all with a single crude bound. Such analyses are more rare, but fairly common. For example, \citep{Nsync,ACD-HanzRich}.

\subsection{``Like treasure hidden in a field, which a man found and covered up'' (Mat 13:44)}

The starting point of this paper is the observation that there is a hitherto untapped richness of smoothness information that {\em can} be used to construct {\em better distributed optimization algorithms and obtain better theory. } This information is available, but hidden from sight, and is based on the notion of {\em matrix smoothness}. 

\begin{definition}[Matrix Smoothness] \label{def:SM} We say that a differentiable function $\phi:\R^d\to \R$ is $\ML$-smooth if there exists a symmetric positive semidefinite matrix  $\ML\succeq 0$ such that
\begin{equation}\label{eq:ML-smoothness}
\phi(x) \le \phi(y) + \<\nabla \phi(y), x-y\> + \frac{1}{2}\|x-y\|^2_{\ML}
\end{equation}
holds for all $x,y\in \R^d$. 
\end{definition}
The standard $L$-smoothness condition \eqref{eq:L-smoothness} is obtained as a special case of \eqref{eq:ML-smoothness} for matrices of the form $\ML=L \MI$, where
$\MI$ is the identity matrix. Function $f_i$ appearing in \eqref{main-opt-problem-dist} is often the average loss over the training data stored on node $i$, i.e., 
\begin{equation}\label{eq:89fd8g90f8h0fd}
f_i(x) = \frac{1}{m_i}\sum \limits_{m=1}^{m_i} \phi_{im}(\MA_{im}x),
\end{equation}
where $\MA_{im}\in \R^{d_{im} \times d}$ is a data matrix, and $\phi_{im}:\R^{d_{im}}\to \R$ is a differentiable function (e.g., the loss over all but the last linear layer of a NN). The following simple result from \citet{QuRich16-2},  used therein in the context of randomized coordinate descent methods, states that if the loss functions $\phi_{im}$ are smooth in the standard scalar sense, then $f_i$ is smooth in the matrix sense.
\begin{lemma}\label{ex:matrix-smooth}
Assume that each $\phi_{im}$ is $\lambda_{im}$-smooth. Then the function $f_i$ defined in  \eqref{eq:89fd8g90f8h0fd} is $\ML_i$-smooth with
\begin{equation} \label{eq:ub98fd_090f9dhfi}
\ML_i = \frac{1}{m_i} \sum \limits_{m=1}^{m_i} \lambda_{im} \MA_{im}^\top \MA_{im}.
\end{equation}
\end{lemma}

In cases where the local functions $f_i$ are of the form  \eqref{eq:89fd8g90f8h0fd}---and it is clear this structure is ubiquitous---there is a lot of potentially useful information contained in the matrix smoothness ``constant'' $\ML_i$. If we were to use the scalar smoothness constant of $f_i$ instead, we would be effectively tossing this richness away, and replacing it with $L_i = \lambda_{\max}(\ML_i)$; the largest eigenvalue of $\ML_i$.   This seems wasteful. As we show in this work, it is. However, we offer a fix. 

\section{Motivation and Contributions} To the best of our knowledge, {\em none} of the current distributed optimization methods, including the methods DCGD \citep{KFJ}, DIANA \citep{MGTR} and \mbox{ADIANA} \citep{AccCGD} discussed in Section~\ref{sec:intro}, are  capable of exploiting the inherently rich data-dependent smoothness structure of the local losses beyond standard smoothness constants.
To this effect, we impose the following assumption throughout the paper:
\begin{assumption}\label{asm:Li-smooth-convex}
The  functions $f_i\colon\R^d\to\R$ are differentiable, convex, lower bounded\footnote{Lower boundedness of $f_i(x)$ can be dropped if $\ML_i\succ0$ is positive definite. This part of the assumption is not a restriction in applications as all loss function are lower bounded.} and $\ML_i$-smooth. Moreover, $f$ is $\ML$-smooth. Let $L \eqdef \lambda_{\max}(\ML)$ be the (standard) smoothness constant of $f$.\end{assumption}

In this paper, we  argue that when training supervised models,  {\em smoothness matrices} (see Definition~\ref{def:SM})---information-rich generalizations of the classical and ubiquitous smoothness constants---can and should be exploited for further dramatic gains, both in theory and practice.  

\begin{table}[t]
\caption{Original and proposed new methods.}
\label{tbl:new-methods}
\begin{center}
\scriptsize
\begin{sc}
\begin{tabular}{@{\hskip 0.01in}cccc@{\hskip 0.01in}}
\toprule
\makecell{Original} & DCGD   & DIANA & ADIANA \\
\midrule
\makecell{{\bf NEW}}
& \makecell{DCGD+ \\ (Alg.\ref{alg:dskgd})}
& \makecell{DIANA+ \\ (Alg.\ref{alg:DIANA+})}
& \makecell{ADIANA+ \\ (Alg.\ref{alg:aDIANA+})}
\\
\midrule
Proximal
& \cmark
& \cmark
& \cmark
\\
Distributed
& \cmark
& \cmark
& \cmark
\\
\makecell{Variance  Reduced}
& \xmark
& \cmark
& \cmark
\\
Accelerated
& \xmark
& \xmark
& \cmark
\\
\bottomrule
\end{tabular}
\end{sc}
\end{center}
\end{table}

{ 

    \begin{table*}[t]
    \scriptsize
    \addtolength{\tabcolsep}{-3pt} 
        \centering
        \caption{Summary of theoretical results obtained in this work with hidden $\log\frac{1}{\varepsilon}$ factors and constants. Below $n$ is the number of machines, $d$ is the number of parameters of model, $L_{\max} = \max_i L_i,\; L_i = \lambda_{\max}(\ML_i)$ and the expected smoothness constant $\widetilde{\cL}_{\max}$ is defined in (\ref{def:tilde-cL-max}). The variance of generic compression operator used in the original methods is denoted by $\omega$. In case of sparsification, we have $\omega=\nicefrac{d}{\tau}-1 = \cO(n)$ when the expected size of selected coordinates is $\tau=\cO(\nicefrac{d}{n})$. Parameters $\nu_1,\nu_2$ and $\nu$ describing distribution of matrices $\ML_i$ are defined in (\ref{def:nu}).}
        \label{tbl:summary}
        \renewcommand{\arraystretch}{1.7}
        \begin{tabular}{|c|c|c|c|}
        \hline
        \makecellnew{Regime}
            & \makecellnew{$\nabla f_i(x^*) \equiv 0$}
            & \makecellnew{arbitrary $\nabla f_i(x^*)$}
            & \makecellnew{arbitrary $\nabla f_i(x^*)$} \\
        \hline
            \hline
            \makecellnew{\bf Original \\ \bf Methods}
            & \makecellnew{\bf DCGD \\ \citep{KFJ}}
            & \makecellnew{\bf DIANA \\ \citep{MGTR}}
            & \makecellnew{\bf ADIANA \\ \citep{AccCGD}} \\
\hline
\makecellnew{Iteration \\ Complexity}
&
$\frac{L}{\mu} + \frac{\omega L_{\max}}{n\mu}$
&
$\omega + \frac{L_{\max}}{\mu} + \frac{\omega L_{\max}}{n\mu}$
&
$
\left\{\begin{smallmatrix}
\omega + \omega\sqrt{\frac{L_{\max}}{n\mu}} & \;\text{if}\; n \le \omega \\
\omega + \sqrt{\frac{L_{\max}}{\mu}} + \sqrt{\omega\sqrt{\frac{\omega L_{\max}}{n\mu}}\sqrt{\frac{L_{\max}}{\mu}}} & \;\text{if}\; n > \omega
\end{smallmatrix}\right.
$ \\
\hline
\makecellnew{Iteration \\ Complexity \\ $\omega = \cO(n)$}
&
$\frac{L_{\max}}{\mu}$
&
$n + \frac{L_{\max}}{\mu}$
&
$
n + n\sqrt{\frac{L_{\max}}{n\mu}} \equiv n + \sqrt{n \frac{L_{\max}}{\mu}}
$ \\
\hline
            \hline
            \rowcolor{bgcolor}
            \makecellnew{\bf New \\ \bf Methods}
            & \makecellnew{\bf DCGD+ \\ (Algorithm \ref{alg:dskgd})}
            & \makecellnew{\bf DIANA+ \\ (Algorithm \ref{alg:DIANA+})}
            & \makecellnew{\bf ADIANA+ \\ (Algorithm \ref{alg:aDIANA+})} \\
\hline
            \rowcolor{bgcolor}
\makecellnew{Iteration \\ Complexity}
&
$\frac{L}{\mu} + \frac{\widetilde{\cL}_{\max}}{n\mu}$
&
$\omega_{\max} + \frac{L}{\mu} + \frac{\widetilde{\cL}_{\max}}{n\mu}$
&
$
\left\{\begin{smallmatrix}
\omega_{\max} + \sqrt{\omega_{\max}\frac{\widetilde{\cL}_{\max}}{n\mu}} & \;\text{if}\; nL \le \widetilde{\cL}_{\max} \\
\omega_{\max} + \sqrt{\frac{L}{\mu}} + \sqrt{\omega_{\max}\sqrt{\frac{\widetilde{\cL}_{\max}}{n\mu}}\sqrt{\frac{L}{\mu}}} & \;\text{if}\; nL > \widetilde{\cL}_{\max}
\end{smallmatrix}\right.
$ \\
\hline
         \rowcolor{bgcolor}
\makecellnew{Iteration \\ Complexity \\ $\omega = \cO(n)$}
&
\makecellnew{
    \\
    $\frac{L_{\max}}{{\color{red} n}\mu} + \frac{L_{\max}}{{\color{red} d}\mu}$ \\
    \\[-8pt]
    $\begin{smallmatrix} (\text{if}\; \nu,\;\nu_1 \;\text{are}\; \cO(1)) \end{smallmatrix}$
}
&
\makecellnew{
    \\
    $n + \frac{L_{\max}}{{\color{red} n}\mu} + \frac{L_{\max}}{{\color{red} d}\mu}$ \\
    \\[-8pt]
    $\begin{smallmatrix} (\text{if}\; \nu,\;\nu_1 \;\text{are}\; \cO(1)) \end{smallmatrix}$
}
&
\makecellnew{
    $
    \left\{\begin{smallmatrix}
    n + n\( \frac{L_{\max}}{n\mu} \)^{{\color{red} \nicefrac{1}{4}}} & \;\text{if}\; nL \le \widetilde{\cL}_{\max} \\
    n + \sqrt{\frac{L_{\max}}{{\color{red}n}\mu}} + \( n\frac{L_{\max}}{\mu} \)^{{\color{red} \nicefrac{3}{8}}}   & \;\text{if}\; nL > \widetilde{\cL}_{\max}
    \end{smallmatrix}\right.
    $ \\
    $\begin{smallmatrix} (\text{if}\; \nu,\nu_2 \;\text{are}\; \cO(1) \;\text{and}\; {\color{blue}\nicefrac{L_{\max}}{\mu} \;\text{is}\; \cO(n d^2)}) \end{smallmatrix}$
}
\\
\hline
         \rowcolor{bgcolor}
\makecellnew{Reference}
&
{\small Theorem~\ref{thm:dist-prox-skgd-better}, Remark~\ref{rem:ibcd}}
&
{\small Theorem~\ref{thm:DIANA+}, Remark~\ref{rem:isega}}
&
{\small Theorem~\ref{thm:aDIANA+}, Remark~\ref{rem:adiana}} \\
\hline
\hline
         \rowcolor{bgcolor2}
\makecellnew{Speedup \\ factor (up to)}
&
${\color{red} \min(n,d)}$
&
${\color{red} \min(n,d)}$
&
\makecellnew{
    $
    \left\{\begin{smallmatrix}
    {\color{red}\sqrt{d}} & \;\text{if}\; nL \le \widetilde{\cL}_{\max} \;\text{and}\; {\color{blue}\nicefrac{L_{\max}}{\mu} = \cO(n d^2)} \\
    {\color{red}\sqrt{\min(n,d)}}   & \;\text{if}\; nL > \widetilde{\cL}_{\max} \;\text{and}\; {\color{blue}\nicefrac{L_{\max}}{\mu} = \cO(n d^2)}
    \end{smallmatrix}\right.
    $
}
\\
        \hline
        \end{tabular}   
    \end{table*}
}

\subsection{Unbiased diagonal sketches}
We study unbiased diagonal sketches, defined as follows:
\begin{definition}[Unbiased diagonal sketch] Let $S$ be a random subset of the set of coordinates/features of the model $x\in \R^d$ we wish to train, i.e., $S\subseteq [d]\eqdef \{1,2,\dots,d\}$. Let $S$ be {\em proper}, i.e., $p_j \eqdef \Pr(j\in S)>0$ for all coordinates $j\in[d]$. We now define a random diagonal matrix (sketch) $\MC = \MC_S \in \R^{d\times d}$  via
\begin{equation}\label{sketch-matrix-C}
\MC = \diag(\mathrm{c}_1, \dots, \mathrm{c}_d), \quad 
\mathrm{c}_j =
\begin{cases} \nicefrac{1}{p_j} & \;\text{if}\;j\in S,\phantom{00} \\ 0\phantom{00}   & \text{otherwise}. \end{cases}
\end{equation}

\end{definition} 

Note that given a vector $x = (x_1,\dots, x_d)\in \R^d$, we have
$$
(\MC x)_j = \begin{cases} x_j/p_j & \text{if} \qquad j\in S \\ 0 & \text{if} \qquad j\notin S \end{cases}.
$$
So, we can control the sparsity level of the product $\MC x$ by engineering the properties of the random set $S$. Also note that $\E [\MC x] = x$ for all $x$.

\subsection{Data-dependent sparsification operators}
In order to further alleviate the communication burden inherent in distributed optimization, we further propose  {\em data-dependent sparsification operators}  that can take full advantage of the smoothness matrices $\ML_i$ associated with the local losses $f_i$. To the best of our knowledge, this is in sharp contrast with the design of all existing tractable compression techniques used in distributed training, which are proposed independently of the training data, and typically based on intuitive or information-theoretic principles.

With each node $i$ we associate an   unbiased diagonal matrix $\MC_i$ of the form~\eqref{sketch-matrix-C}. We use this and the smoothness matrix of $f_i$ to define a sparsification technique, described next. 

\begin{definition}[Data-dependent sparsification] 
In situations when the $i$-th node wished to communicate local gradient $\nabla f_i(x)$, we ask the node to send the sparse (=compressed) vector $\MC_i\ML_i^\phalf \nabla f_i(x)$ to the server instead.  The server then constructs (=decompresses) an unbiased estimator of  $\nabla f_i(x)$  as follows: 
\begin{equation}\label{compression_our}
g_i(x) = \ML_i^\half\MC_i\ML_i^\phalf \nabla f_i(x),
\end{equation}
where $\ML_i^\phalf$ denotes the square root of the Moore-Penrose pseudoinverse of $\ML_i$.
\end{definition}

Notable differences of  our proposed communication protocol  when compared with standard sparsification techniques are: i) we use the smoothness matrix $\ML_i$, ii) the compressed vector $\MC_i\ML_i^\phalf \nabla f_i(x)$ is not unbiased, iii) we devise a separate decompression mechanism \eqref{compression_our}, also involving $\ML_i$, and this enforces effective unbiasedness.

%

%

\subsection{Matrix-smoothness-aware redesign of 3 distributed methods}
To showcase the power of our approach, we demonstrate how our matrix-smoothness-aware sparsification technique \eqref{compression_our}  can be adapted to DCGD, DIANA and ADIANA, in each case leading to significant communication savings. 
 By doing so, we show that matrix smoothness can be effectively used to speed up communication compression, variance reduction and acceleration, respectively. 
This results  in three novel methods: DCGD+, DIANA+, and ADIANA+; see Table~\ref{tbl:new-methods}.

\subsection{Dramatic improvements in complexity results}
We perform complexity analyses for our methods and derive convergence rates under matrix smoothness\footnote{The closest to our result is work of \citet{GJS-HR} and their ISEGA method which is able to  exploit {\em diagonal} smoothness matrices.  To the best of our knowledge, we are the first to fully exploit smoothness matrices of arbitrary structure, and elevate them as a new tool at the disposal of algorithm designers.}
(see Assumption \ref{eq:ML-smoothness}) and strong convexity assumptions (see Theorems \ref{thm:dist-prox-skgd-better}, \ref{thm:DIANA+} and \ref{thm:aDIANA+}). We show that new methods always outperform the originals/baselines, and often dramatically so. 

To illustrate the potential of our sparsification technique \eqref{compression_our} embedded in the new methods, let all machines $i\in[n]$ use sketches $\MC_i$ induced by independent\footnote{Sampling $S_i$ is called independent if $p_{i;jl} \eqdef \Pr(\{j,l\}\subseteq S_i) = p_{i;j} p_{i;l}$ for all $j,l\in[d]$.} samplings $S_i$ with probabilities $p_{i;j} \eqdef \Pr(j\in S_i)$. Then we show that, with optimized probabilities $p_{i;j}$, DCGD+ can be $\cO(\min(n,d))$ times faster then DCGD (see Remark ~\ref{rem:ibcd}) and DIANA+ can be $\cO(\min(n,d))$ times faster than DIANA (see Remark~\ref{rem:isega}), depending on the distribution of $\ML_i$.
For the accelerated method, we highlight improvements when condition numbers of subproblems are $\cO(n d^2)$. We show that \mbox{ADIANA+} can be faster than the original ADIANA by a factor of $\cO(\sqrt{d})$ in high compression regime, and by a factor of $\cO(\sqrt{\min(n,d)})$ in low compression regime (see Remark~\ref{rem:adiana}). Main theoretical results are summarized in Table \ref{tbl:summary}.


\subsection{Single node case} 
Specializing our theory to the single machine setting ($n=1$), we design new non-distributed algorithms providing an alternative viewpoint to randomized coordinate descent methods (see Appendix \ref{rcd-as-skgd}).

\subsection{Lower bounds} 
Using matrices as linear compression operators, we further investigate the trade-off between communicated bits and variance induced by the compression (see Appendix \ref{sec:lower-bounds}).

\subsection{Experiments} 
We conduct numerical experiments using LibSVM datasets~\citep{chang2011libsvm}, confirming the effectiveness and superiority of our sparsification protocol (\ref{compression_our}) over the standard sparsification scheme (see Section~\ref{sec:experiments}).

\section{New Communication-Efficient Distributed Methods \\ Exploiting Matrix Smoothness}

Consider the distributed optimization problem (\ref{main-opt-problem-dist}) with the smoothness Assumption  \ref{asm:Li-smooth-convex} and for strongly convex $f$.

\begin{assumption}[$\mu$-convexity]\label{asm:mu-convex}
$f\colon\R^d\to\R$ is $\mu$-convex for some $\mu>0$, i.e., $$
f(x) \ge f(y) + \<\nabla f(x), x-y\> + \frac{\mu}{2}\|x-y\|^2
$$ for all $x,y\in\R^d$.
\end{assumption}




Below we present our new distributed methods, redesigned for matrix smoothness, and their convergence guarantees. Each node $i\in[n]$ generates diagonal sketches $\MC_i$ independently from others via an arbitrary sampling $S_i$ and, togther with its smoothness matrix $\ML_i$, composes the compression matrix $\MC_i\ML_i^{\dagger\nicefrac{1}{2}}$. Probability matrices $\MP_i$ and $\widetilde{\MP}_i$ associated with the sampling $S_i$ and sketch $\MC_i$ are defined as follows
\begin{align}
\begin{split}\label{def:P}
\MP_i &= (p_{i;jl})_{jl=1}^d, \qquad p_{i;jl} = \Pr(\{j,l\}\subseteq S_i), \\
\widetilde{\MP}_i &  = (\widetilde{p}_{i;jl})_{jl=1}^d, \qquad \widetilde{p}_{i;jl} = \frac{p_{i;jl}}{p_{i;jj}p_{i;ll}}-1.
\end{split}
\end{align}

Next, we introduce the key quantity, $\widetilde{\cL}_{\max}$, describing the joint contribution of our sparsification~\eqref{compression_our} to the complexities of the three proposed methods: 
\begin{equation}\label{def:tilde-cL-max}
\widetilde{\cL}_{\max} = \max_{1\le i\le n} \widetilde{\cL}_i,
\qquad
\widetilde{\cL}_i = \lambda_{\max}(\widetilde{\MP}_i\circ\ML_i),
\end{equation}
Above,  $\circ$ stands for Hadamard (i.e. element-wise) product.

\subsection{DCGD+}

We now present our matrix-smoothness-aware sparsification technique by adapting DCGD algorithm \citep{KFJ}.

Upon receiving the current model $x^k$ from the server, each node computes $\ML_i^{\phalf} \nabla f_i(x^k)$ based on local training data and smoothness matrix. Next, sparsified updates $\MC_i^k\ML_i^{\phalf} \nabla f_i(x^k)$ are sent back to the server, which then averages decompressed updates $\ML_i^{\half}\MC_i^k\ML_i^{\dagger\nicefrac{1}{2}} \nabla f_i(x^k)$ and performs proximal step to get a new model $x^{k+1}$.

\begin{algorithm}[H]
\begin{algorithmic}[1]
\STATE \textbf{Input:} Initial point $x^0\in\R^d$, current point $x^k$, step size $\gamma$, diagonal sketch $\MC_i^k$
\STATE \textbf{on} server
\STATE \quad send $x^k$ to all nodes
\STATE \quad get sparse updates $\MC_i^k\ML_i^{\phalf} \nabla f_i(x^k)$ from each node
\STATE \quad $g^k = \frac{1}{n}\sum_{i=1}^n \ML_i^{\half}\MC_i^k\ML_i^{\phalf} \nabla f_i(x^k)$
\STATE \quad $x^{k+1} = \prox_{\gamma R}(x^k - \gamma g^k)$
\end{algorithmic}
\caption{\sc DCGD+}
\label{alg:dskgd}
\end{algorithm}

With this method we get convergence up to a neighborhood.

\begin{theorem}[see \ref{apx-thm:dist-prox-skgd-better}]\label{thm:dist-prox-skgd-better}
Let Assumptions  \ref{asm:Li-smooth-convex} and \ref{asm:mu-convex} hold and assume that each node generates its own diagonal sketch $\MC_i$ independently from others. Then, for the step-size 
$$
0<\gamma
\le \frac{1}{L+\frac{2}{n} \widetilde{\cL}_{\max}},
$$
the iterates $\{x^k\}$ of Algorithm \ref{alg:dskgd} satisfy
\begin{equation}\label{rate-dskgd}
\E\[\|x^k - x^*\|^2\] \le \(1-\gamma\mu\)^k\|x^{0}  - x^*\|^2 + \frac{2\gamma\sigma^*}{\mu n},
\end{equation}
where $\sigma^* \eqdef \frac{1}{n}\sum_{i=1}^n \widetilde{\cL}_i \|\nabla f_i(x^*)\|^2_{\ML_i^{\dagger}}$.
\end{theorem}


{\bf Proof technique.} 
First we show the unbiasedness of $g^k$. As smoothness matrices $\ML_i$ are not necessarily invertible, terms like $\ML_i^\half\ML_i^\phalf$ show up in the analysis and block chains of cancellations. This part is handled by the fact that gradients $\nabla f_i(x)$ of an $\ML_i$-smooth function are constraint to remain in $\range{\ML_i}$ and the mapping associated with the matrix $\ML_i^\half\ML_i^\phalf$ is identity on the subspace $\range(\ML_i)$.
Second part is the tight estimation of $\E_k\|g^k - \nabla f(x^*)\|^2$, which describes {the progress of the method} in the presence of stochasticity. Key part is getting the decomposition
\begin{equation}\label{moment-decomposition__meta}
\E_k \left[\|g^k - \nabla f(x^*)\|^2 \right]
= \|\nabla f(x^k) - \nabla f(x^*)\|^2  + \frac{1}{n^2}\sum_{i=1}^n \left\|\nabla f_i(x^k) \right\|^2_{ \ML_i^\phalf(\widetilde{\MP}_i\circ\ML_i)\ML_i^\phalf }, 
\end{equation}
which shows the exact interaction between random sketches and local smoothness. We complete the proof using the unified convergence theory of \citet{sigma_k}.

\subsection{Variance reduction: DIANA+}

Next, we apply our sparsification technique to the variance reduced method DIANA \citep{MGTR}.

In this method, each node maintains an auxiliary control vector $h_i^k$, called shift, which helps to reduce the variance coming from the sparsification. Moreover, the central server keeps track of only the averaged shift $h^k$. Then, the model $x^k$ as well as control vectors $h_i^k,\; h^k$ are updated by decompressing sparse information $\Delta_i^k$ using matrices $\ML_i$.

\begin{algorithm}[H]
\begin{algorithmic}[1]
\STATE \textbf{Input:} Initial point $x^0\in\R^d$, initial shifts $h_i^0\in\range(\ML_i)$, current point $x^k$, step size parameter $\gamma$ and $\alpha$, sketch $\MC_i^k$ and $\overbar{\MC}_i^k \eqdef \ML_i^\half\MC_i^k\ML_i^\phalf$, current shifts $h_1^k,\dots,h_n^k$ and $h^k \eqdef \frac{1}{n}\sum_{i=1}^n h_i^k$.
\STATE \textbf{on} each node
\STATE \quad get $x^k$ from the server
\STATE \quad send sparse update $\Delta_i^k = \MC_i^k\ML_i^\phalf (\nabla f_i(x^k) - h_i^k)$
\STATE \quad $\overbar{\Delta}_i^k = \ML_i^\half\Delta_i^k,\; g_i^k = h_i^k + \overbar{\Delta}_i^k, h_i^{k+1} = h_i^k + \alpha\overbar{\Delta}_i^k$
\STATE \textbf{on} server
\STATE \quad get sparse updates $\Delta_i^k$ from each node
\STATE \quad $\overbar{\Delta}^k = \frac{1}{n}\sum_{i=1}^n \overbar{\Delta}_i^k = \frac{1}{n}\sum_{i=1}^n \ML_i^\half\Delta_i^k$
\STATE \quad $g^k = \overbar{\Delta}^k + h^k = \frac{1}{n}\sum_{i=1}^n \overbar{\MC}_i^k \(\nabla f_i(x^k) - h_i^k\) + h_i^k$
\STATE \quad $x^{k+1} = \prox_{\gamma R}(x^k - \gamma g^k)$
\STATE \quad $h^{k+1} = h^k + \alpha\overbar{\Delta}^k$
\end{algorithmic}
\caption{\sc DIANA+}
\label{alg:DIANA+}
\end{algorithm}

In this case we get rid of the neighborhood and provide linear convergence to the exact solution $x^*$. We use $\widetilde{\cO}$ notation to ignore $\log\frac{1}{\varepsilon}$ factors and constants.

\begin{theorem}[see \ref{apx-thm:DIANA+}]\label{thm:DIANA+}
Let Assumptions \ref{asm:Li-smooth-convex} and \ref{asm:mu-convex} hold and assume that each node generates its own diagonal sketch $\MC_i$ independently from others. Then, for the step-size 
$$
\gamma = \frac{1}{L+\frac{6}{n} \widetilde{\cL}_{\max}},
$$
Algorithm \ref{alg:DIANA+} guarantees $\E\[\|x^k-x^*\|^2\]\le\varepsilon$ after
\begin{equation}\label{DIANA+-complexity}
\widetilde{\cO}\( \omega_{\max} + \frac{L}{\mu} + \frac{\widetilde{\cL}_{\max}}{n\mu}\)
\end{equation}
 iterations, where $\omega_{\max} = \max_{1\le i\le n}\omega_i$ and $\omega_i = \max_{1\le j\le d}\frac{1}{p_{i;j}}-1$ is the variance of compression operator induced by sketch $\MC_i$.
\end{theorem}

{\bf Proof technique.}
The structure of the proof resembles the one for DCGD+.
With the introduced shift vectors, the unbiasedness of $g^k$ additionally requires $h_i^k\in\range(\ML_i)$. This is resolved by the initialization $h_i^0\in\range(\ML_i)$ and linear update rule for $h_i^{k+1}$ in line 5. The proof develops a decomposition similar to (\ref{moment-decomposition__meta}) with modified second term $\sigma^k \eqdef \frac{1}{n}\sum_{i=1}^n\|h_i^k - \nabla f(x^*)\|^2_{\ML_i^\dagger}$ involving shifts $h_i^k$. To avoid the neighborhood term in (\ref{rate-dskgd}) and guarantee a linear convergence for $x^k$, we make $\sigma^k$ converge linearly too.
Key technical part of the proof is to establish contracting recurrence relation for $\sigma^k$ which boils down to $\E[\overbar{\MC}_i^\top\ML_i^{\dagger}\overbar{\MC}_i] \preceq (\omega_i+1)\ML_i^{\dagger}$. The latter bound justifies the structure of $\overbar{\MC}_i$ as it filters the interaction between compression and smoothness mixed in the expectation and separates variance $\omega_i$ of compression from smoothness matrix $\ML_i$.

\begin{remark}[Variance Reduction: ISEGA+]
In Appendix \ref{apx-sec:ISEGA} we apply our redesign to another variance reduced method called ISEGA \citep{99KFP,GJS-HR}. At the core of ISEGA, the mechanism for variance reduction is based on SEGA method \citep{SEGA-FP}. The key difference between ISEGA and DIANA is that ISEGA updates the control variates $h$ more aggressively using projection instead of the mere $\alpha$-step towards the projection used in DIANA.  Formally, adapting our matrix-smoothness-aware sparsification to ISEGA, we define the update rule of control vectors $h_i^k$ as follows
\begin{equation*}
h_i^{k+1}
= \argmin_{\substack{h\in\range(\ML_i) \\ \MC_i^k\ML_i^\phalf\nabla f_i(x^k) = \MC_i^k\ML_i^\phalf h}} \|h-h_i^k\|^2_{\ML_i^\dagger} 
= h_i^k + \ML_i^\half \Mdiag(\MP_i) \MC_i^k\ML_i^\phalf (\nabla f_i(x^k) - h_i^k).
\end{equation*}
On the other hand, notice that the update rule in DIANA+ has the form
$$
h_i^{k+1} = h_i^k + \alpha \ML_i^\half \MC_i^k\ML_i^\phalf (\nabla f_i(x^k) - h_i^k)
$$
for some fixed scalar $\alpha>0$, and thus is more conservative.  Note that we choose the gradient estimator for ISEGA+ to be the same $g_i^k = h_i^k + \ML_i^\half \MC_i^k\ML_i^\phalf (\nabla f_i(x^k) - h_i^k)$.  The method is presented as Algorithm~\ref{alg:ISEGA+} in Appendix \ref{apx-sec:ISEGA}. 

In contrast to DIANA+, we can not obtain the convergence rate of ISEGA+ directly from the framework of~\citet{sigma_k}. Instead, to get the tight convergence rate, we shall cast it as an instance of GJS method~\citep{GJS-HR}.  Theorem~\ref{thm:ISEGA} provides the result -- we can see that the worst case complexity is identical to DIANA+. However, in terms of the practical performance, we expect ISEGA+ to outperform DIANA+ due to the more aggressive update rule of control variates. 
\end{remark}

\begin{remark}[Variance Reduction with Bi-directional Compression: DIANA++]
As an extension to DIANA+, in Appendix \ref{apx-thm:DIANA++} we apply our sparsification technique both for nodes and for the central server, thus compressing gradients in both directions of communication. We develop and analyze DIANA++ method (see Algorithm \ref{alg:DIANA++}), for which the central server applies compression in its turn with sketch $\MC$ independently. To converge in a linear rate, DIANA++ maintains an additional control vector, which helps to reduce the variance coming from the master's sparsification. Theorem \ref{thm:DIANA++} provides complexity result for DIANA++, which recovers the same complexity (\ref{DIANA+-complexity}) of DIANA+ if no compression is applied by the master.
\end{remark}

\subsection{Acceleration with variance reduction: ADIANA+}

Finally, we redesign the accelerated method ADIANA \citep{AccCGD} to effectively exploit local smoothness matrices.

The algorithm develops four sequences $\{x^k,y^k,z^k,w^k\}$ of models, which are layered via convex combinations, proximal steps and probabilistic assignments. In each iteration, nodes receive models $x^k$ and $w^k$ from the server, and send back sparse updates $\Delta_i^k$ and $\delta_i^k$ using local data and control vectors $h_i^k$. Then, decompressing these sparse vectors with matrices $\ML_i$, nodes update their shifts $h_i^k$ and the server updates all four models along with averaged shift $h^k$.

\begin{algorithm}[H]
\begin{algorithmic}[1]
\STATE \textbf{Input:} Initial points $x^0=y^0=z^0=w^0\in\R^d$, initial shifts $h_i^0\in\range(\ML_i)$, current point $x^k$, parameters $\gamma,\alpha,\beta,\eta,\theta_1,\theta_2,q$, sketch $\MC_i^k$ and $\overbar{\MC}_i^k \eqdef \ML_i^\half\MC_i^k\ML_i^\phalf$, current shifts $h_1^k,\cdots,h_n^k$ and $h^k=\frac{1}{n}\sum_{i=1}^n h_i^k$
\STATE \textbf{on} server
\STATE \quad $x^k = \theta_1 z^k + \theta_2 w_k + (1-\theta_1-\theta_2)y^k$
\STATE \quad send $x^k$ and $w^k$ to all nodes
\STATE \textbf{on} each node
\STATE \quad send sparse update $\Delta_i^k = \MC_i^k\ML_i^\phalf (\nabla f_i(x^k) - h_i^k)$
\STATE \quad send sparse update $\delta_i^k = \MC_i^k\ML_i^\phalf (\nabla f_i(w^k) - h_i^k)$
\STATE \quad update local gradient $\overbar{\Delta}_i^k = \ML_i^\half\Delta_i^k,\; g_i^k = h_i^k + \overbar{\Delta}_i^k$
\STATE \quad update local shift $\overbar{\delta}_i^k = \ML_i^\half\delta_i^k,\; h_i^{k+1} = h_i^k + \alpha\overbar{\delta}_i^k$
\STATE \textbf{on} server
\STATE \quad get sparse updates $\Delta_i^k$ and $\delta_i^k$ from each node
\STATE \quad $\overbar{\Delta}^k = \frac{1}{n}\sum_{i=1}^n \ML_i^\half\Delta_i^k,\; \overbar{\delta}^k = \frac{1}{n}\sum_{i=1}^n \ML_i^\half\delta_i^k$
\STATE \quad $g^k = \overbar{\Delta}^k + h^k = \frac{1}{n}\sum_{i=1}^n \overbar{\MC}_i^k \(\nabla f_i(x^k) - h_i^k\) + h_i^k$
\STATE \quad $h^{k+1} = h^k + \alpha\overbar{\delta}^k$
\STATE \quad $y^{k+1} = \prox_{\eta R}(x^k - \eta g^k)$
\STATE \quad $z^{k+1} = \beta z^k + (1-\beta)x^k + \frac{\gamma}{\eta}(y^{k+1}-x^k)$
\STATE \quad
$w^{k+1} =
\begin{cases}
y^k & \text{with probability}\quad q,\\
w^k & \text{with probability}\quad 1-q.
\end{cases}$
\end{algorithmic}
\caption{\sc ADIANA+}
\label{alg:aDIANA+}
\end{algorithm}

Clearly, the new method ADIANA+ enjoys the accelerated rate, which is strictly better then the one for DIANA+.

\begin{theorem}[see \ref{apx-thm:aDIANA+}]\label{thm:aDIANA+}
Let Assumptions   \ref{asm:Li-smooth-convex} and \ref{asm:mu-convex} hold and assume that each node generates its own diagonal sketch $\MC_i$ independently from others. Then, the iteration complexity of Algorithm \ref{alg:aDIANA+} guaranteeing $\E\[\|z^k-x^*\|^2\]\le\varepsilon$ is
\begin{equation}\label{ADIANA+-complexity}
\begin{cases}
\widetilde{\cO}\(\omega_{\max} + \sqrt{\omega_{\max}\frac{\widetilde{\cL}_{\max}}{\mu n}}\) & \text{if}\quad nL \le \widetilde{\cL}_{\max}\\
\widetilde{\cO}\(\omega_{\max} + \sqrt{\frac{L}{\mu}} + \sqrt{\omega_{\max}\sqrt{\frac{\widetilde{\cL}_{\max}}{\mu n}}\sqrt{\frac{L}{\mu}}}\) & \text{if}\quad  nL > \widetilde{\cL}_{\max}.
\end{cases}
\end{equation}
\end{theorem}

{\bf Proof technique.}
The additional difficulty that acceleration brings on top of variance reduction is the modified term $H^k \eqdef \frac{1}{n}\sum_{i=1}^n \|h_i^k - \nabla f_i(w^k)\|^2_{\ML_i^{\dagger}}$ controlling variance reduction process. The subtlety of $H^k$ in contrast to $\sigma^k$ is gradients $\nabla f_i(w^k)$ which are not fixed. Key technical part is to reduce contracting property of $H^k$ into upper bounding $\E[(\MI-\alpha\overbar{\MC}_i)^{\top}\ML_i^{\dagger}(\MI-\alpha\overbar{\MC}_i)]$ by $(1-\alpha)\ML_i^\dagger$ as quadratic forms in the subspace $\range(\ML_i)$.

\section{Improvements Over the Original Methods}\label{sec:compare}

To compare the proposed methods with originals and highlight improvement factors, we choose independent sampling for all nodes. For Algorithms \ref{alg:dskgd} and \ref{alg:DIANA+}, we optimize probabilities of the samplings based on the complexities we found.

\subsection{Parameters describing distribution of $\ML_{i}$}
Define parameters $\nu$ and $\nu_{s}$ describing the distribution of local smoothness matrices $\ML_i$ as follows
\begin{equation}\label{def:nu}
\nu \eqdef \frac{\sum_{i=1}^n L_i}{\max_{i\in[n]} L_i}, \quad
\nu_s \eqdef \max_{i\in[n]}\frac{\sum_{j=1}^d \ML^{\nicefrac{1}{s}}_{i;j}}{\max_{j\in[d]}\ML^{\nicefrac{1}{s}}_{i;j}},
\end{equation}
where $L_i = \lambda_{\max}(\ML_i)$ and $s=1$ or $s=2$. Let $L_{\max} \eqdef \max_{1\le i\le n} L_i$.
Note that parameters $\nu\in[1,n]$ and $\nu_s\in[1,d]$ describe the distribution over the nodes and coordinates respectively.
If $\ML_i$ are distributed uniformly, then $\nu=n$ and $\nu_s=d$. On the other extreme, when the distribution is extremely non-uniform, we have $\nu\ll n$ and $\nu_s\ll d$. These parameters are used to highlight the range of iteration complexities new methods can provide.

\subsection{Importance sampling for DCGD+}
Let $\tau=\E\[|S_i|\] = \sum_{j=1}^d p_{i;j}$ be the expected mini-batch size for the samplings $S_i$, where $p_{i;j} = p_{i;jj}$.
Notice that convergence rate of Algorithm \ref{alg:dskgd} depends on $\widetilde{\cL}_{\max} = \max_{1\le i\le n} \widetilde{\cL}_i$. Since each node $i\in[n]$ generates its own diagonal sketch $\MC_i$ independently from others, each node can optimize $\widetilde{\cL}_i = \lambda_{\max}(\widetilde{\MP}_i\circ\ML_i)$ independently based on local smoothness matrix $\ML_i$. In general, minimizing $\lambda_{\max}(\widetilde{\MP}_i\circ\ML_i)$ with respect to probability matrix $\widetilde{\MP}_i$ is hard. However, when each node uses an independent sampling, which means $p_{i;jl} = p_{i;j}p_{i;l}$ if $j\ne l$, then
\begin{equation}\label{cL-ind-sampling}
\lambda_{\max}(\widetilde{\MP}_i\circ\ML_i) = \max_{1\le j\le d}\(\frac{1}{p_{i;j}}-1\)\ML_{i;j},
\end{equation}
for which we can find the optimal probabilities $p_{i;j}$. To minimize the maximum term in (\ref{cL-ind-sampling}), we should have $\(\nicefrac{1}{p_{i;j}}-1\)\ML_{i;j} = \rho_i$ for some $\rho_i\ge0$. Then the solution is
\begin{equation}\label{ind-sampling-probs}
p_{i;j} = \frac{\ML_{i;j}}{\ML_{i;j}+\rho_i},
\end{equation}
where $\rho_i\ge0$ is the unique solution to $\sum_{j=1}^d\frac{\ML_{i;j}}{\ML_{i;j}+\rho_i}=\tau$.
The latter does not allow closed form solution for $\rho_i$. However, since $\rho_i$ is the root of strictly monotone and one dimensional function, it can be computed numerically using one dimensional solvers. Thus, we can efficiently compute the optimal probabilities (\ref{ind-sampling-probs}). 

\begin{proposition}[Optimality]
The independent sampling with probabilities (\ref{ind-sampling-probs}) is the optimal independent sampling for the rate (\ref{rate-dskgd}).
\end{proposition}

\begin{remark}[Improvement over DCGD \citep{KFJ}]\label{rem:ibcd}
With probabilities (\ref{ind-sampling-probs}) we show in Appendix \ref{apx-rem:ibcd} that
\begin{align}\label{L-bound-01}
\begin{split}
\frac{L}{\mu} + \frac{\widetilde{\cL}_{\max}}{n\mu}
\le
    \(\frac{\nu}{n} + \frac{\nu_1}{\tau n}\) \frac{L_{\max}}{\mu}.
\end{split}
\end{align}
In the interpolation regime (i.e. $\nabla f_i(x^*)=0$ for all $i\in[n]$), the iteration complexity of DCGD is $\widetilde{\cO}(\frac{L}{\mu}+\frac{\omega L_{\max}}{n\mu})$ for general compression operator with variance parameter $\omega$. If we specialize compression to sparsification with $\tau=\nicefrac{d}{n}$ entries (which gives $\omega = \nicefrac{d}{\tau}-1=n-1$), we get $\widetilde{\cO}(\frac{L_{\max}}{\mu})$.
Notice that, in this regime, Theorem \ref{thm:dist-prox-skgd-better} also provides linear convergence with iteration complexity $\widetilde{\cO}( \frac{L}{\mu} + \frac{\widetilde{\cL}_{\max}}{n\mu})$. Based on (\ref{L-bound-01}), it is bounded by $\widetilde{\cO}( (\frac{\nu}{n} + \frac{\nu_1}{d}) \frac{L_{\max}}{\mu})$, which is always better than $\widetilde{\cO}(\frac{L_{\max}}{\mu})$ and can be as small as $\widetilde{\cO}(\frac{L_{\max}}{\min(n,d)\mu})$. Hence, for mini-batch $\tau=\nicefrac{d}{n}$, DCGD+ (Algorithm \ref{alg:dskgd}) guarantees the same $\widetilde{\cO}(\frac{L_{\max}}{\mu})$ complexity in the worst case, but could provide up to $\min(n,d)$ times speedup.
\end{remark}

\subsection{Importance sampling for DIANA+}
To find optimal probabilities for DIANA+, we minimize $\omega_{\max}+\frac{\widetilde{\cL}_{\max}}{\mu n}$ part of the complexity (\ref{DIANA+-complexity}). Definitions of $\widetilde{\cL}_{\max}$ and $\omega_{\max}$ imply that it is equivalent to minimize
\begin{equation}\label{opt-prob-diana}
\max_{1\le j\le d}\(\frac{1}{p_{i;j}}-1\)\ML'_{i;j}, \quad \ML'_{i;j} \eqdef \frac{\ML_{i;j}}{\mu n} + 1,
\end{equation}
which can be solved in the same way as (\ref{cL-ind-sampling}) yielding
\begin{equation}\label{ind-sampling-probs-diana}
p_{i;j} = \frac{\ML'_{i;j}}{\ML'_{i;j}+\rho'_i} = \frac{\ML_{i;j} + \mu n}{\ML_{i;j}+(1+\rho'_i)\mu n}.
\end{equation}
\begin{proposition}[Optimality]
The independent sampling with probabilities (\ref{ind-sampling-probs-diana}) is the optimal\footnote{In the sense that it minimizes a quantity, which is the complexity of DIANA+ up to some constant factor.} independent sampling for the complexity (\ref{DIANA+-complexity}).
\end{proposition}
\begin{remark}[Improvement over DIANA \citep{MGTR,DIANA-VR}]\label{rem:isega}
Here we compare DIANA+ against the original DIANA method, which has iteration complexity $\widetilde{\cO}(n+\frac{L_{\max}}{\mu})$ when each node sparsifies with $\tau=\nicefrac{d}{n}$ entries. With probabilities (\ref{ind-sampling-probs-diana}) we upper bound the complexity (\ref{DIANA+-complexity}) in Appendix \ref{apx-rem:isega} as follows
\begin{align}\label{L-bound-05}
\begin{split}
\omega_{\max} + \frac{L}{\mu} + \frac{\widetilde{\cL}_{\max}}{\mu n}
\le
    \frac{2d}{\tau} + \(\frac{\nu}{n} + \frac{2\nu_1}{\tau n}\) \frac{L_{\max}}{\mu}.
\end{split}
\end{align}
Therefore, with $\tau=\nicefrac{d}{n}$, DIANA+ (Algorithm \ref{alg:DIANA+}) guarantees the same $\widetilde{\cO}(n+\frac{L_{\max}}{\mu})$ complexity in the worst case, but could provide up to $\min(n,d)$ times speedup with iteration complexity $\widetilde{\cO}(n+\frac{L_{\max}}{\min(n,d)\mu})$.
\end{remark}

\subsection{Independent sampling for ADIANA+}
Clearly, if we sparsify with uniform probabilities $p_{i;j} = \nicefrac{\tau}{d}$, then Algorithm \ref{alg:aDIANA+} recovers the rate of ADIANA. 

\begin{remark}[Improvement over ADIANA \citep{AccCGD}]\label{rem:adiana}
To show that the rate could be significantly better in some cases, consider the following choice

\begin{equation}\label{ind-sampling-probs-adiana}
p_{i;j}
= \sqrt{\frac{ \ML'_{i;j} }{\ML'_{i;j} + \rho''_i}},
\quad
\ML'_{i;j} = \frac{\ML_{i;j}}{\mu n} + 1,
\end{equation}
where $\rho''_i$ is determined uniquely from $\sum_{j=1}^d p_{i;j} = \tau$. Then, with these probabilities and for $\nicefrac{L_{\max}}{\mu} = \cO(n d^2)$, we show in Appendix \ref{apx-rem:adiana} that
$$
\frac{L}{\mu} \le \frac{\nu L_{\max}}{n\mu},
\quad
\omega_{\max} = \cO\(\frac{\nu_2 d}{\tau}\),
\quad
\frac{\cL_{\max}}{\mu n} = \cO\(\frac{\nu_2 d}{\tau} \sqrt{\frac{L_{\max}}{n\mu}}\).
$$
Furthermore, assuming both $\nu$ and $\nu_2$ are $\cO(1)$, choosing $\tau=\nicefrac{d}{n}$ we get
$$
\frac{L}{\mu} \le \cO\(\frac{L_{\max}}{n\mu}\),
\quad
\omega_{\max} = \cO\(n\),
\quad
\frac{\cL_{\max}}{\mu n} = \cO\(\sqrt{\frac{n L_{\max}}{\mu}}\).
$$
Then, the complexity (\ref{ADIANA+-complexity}) of ADIANA+ reduces to
$$
\begin{cases}
    n + n\( \frac{L_{\max}}{n\mu} \)^{\nicefrac{1}{4}} & \;\text{if}\; nL \le \widetilde{\cL}_{\max}, \\
    n + \sqrt{\frac{L_{\max}}{n\mu}} + \( n\frac{L_{\max}}{\mu} \)^{\nicefrac{3}{8}}   & \;\text{if}\; nL > \widetilde{\cL}_{\max},
\end{cases}
$$
which, compared to the complexity of ADIANA with $\omega=\cO(n)$ compression, gives $\sqrt{d}$ times improvement in the first case and $\sqrt{\min(n,d)}$ times improvement in the second case (ignoring the first summand $n$ of the complexities).
\end{remark}

\section{Experiments} \label{sec:experiments}

In this section we numerically compare the proposed matrix-smoothness-aware sparsification strategy (\ref{compression_our}) with the usual sparsification scheme.

\subsection{Experimental Setup}
We devise three different experiments on logistic regression with LibSVM data~\citep{chang2011libsvm}. In particular,  the objective is given as 

\begin{equation*}
f_{i}(x) \eqdef  \frac{1}{m_i} \sum_{j=1}^{m_i}\log \left(1+\exp\left((\MA_{im})_{j,:}x\cdot  (b_{im})_{j}\right) \right)
 +\frac{\mu}{2} \| x\|^2,
\end{equation*}
where $\MA_{im}\in \R^{d_{im}\times d}$ is the data matrix with corresponding labels $b_{im}\in \R^{d_{im}}$. In our case, we did split the randomly reshuffled datasets into equal chunks among workers in each case so that $m_i = m_j$ for all $i,j\leq n$.  The data matrix $\MA$ was normalized so that each datapoint has a norm equal to $\frac12$.  Lastly, we have chosen $\mu = 10^{-3}$ for all experiments.

For each of the datasets, we have selected a specific number of workers given by Table~\ref{tbl:datasets}.  Each of the method was run with theory supported parameters with an exception of the ADIANA+, where we have omitted several constant factors for the sake of practicality.

{\footnotesize
\begin{table}[!h]
\caption{Datasets.}
\label{tbl:datasets}
\begin{center}
\begin{tabular}{|c|c|c|c|c|}
\hline
Dataset & \# datapoints & $d$ &  $n$ & $m_i$  \\
\hline
\hline
\texttt{a1a} & 1 605    & 123 &  107  & 15\\ \hline
\texttt{mushrooms} & 8 124 & 112 & 12  & 677 \\ \hline
\texttt{phishing} & 11 055  &   68 & 11  &  1 005  \\ \hline
\texttt{madelon} & 2 000& 500& 4 & 500 \\ \hline
\texttt{duke} &44 & 7 129& 4 & 11\\ \hline
\texttt{a8a} &22 696 & 123& 8 & 2837\\ \hline
\end{tabular}
\end{center}
\end{table}
}


\subsection{Variance reduction with new sparsification and importance sampling}

We now comment on the experiment illustrated in Figure~\ref{fig:sampling}.
We examine three sparsification schemes (two variants of our strategy and the usual sparsification not aware of smoothness matrices) and their influence on convergence using six different datasets.
Considered schemes are
i) DIANA+ with importance sampling (\ref{ind-sampling-probs-diana}),
ii) DIANA+ with uniform sampling,
and
iii) DIANA with uniform sampling, i.e., uniform sparsification unaware of smoothness matrices.
In all three cases we fixed the sampling size $\tau=1$.

As expected, Figure~\ref{fig:sampling} confirms our theoretical findings.  First,  it demonstrates that our sparsification~\eqref{compression_our} always outperforms the naive/direct sparsification, sometimes by a large margin. Second, it shows the benefit of importance sampling (\ref{ind-sampling-probs-diana}) over the uniform sampling.

\begin{figure}[!h]
\centering
\begin{minipage}{0.3\textwidth}
  \centering
\includegraphics[width =  \textwidth ]{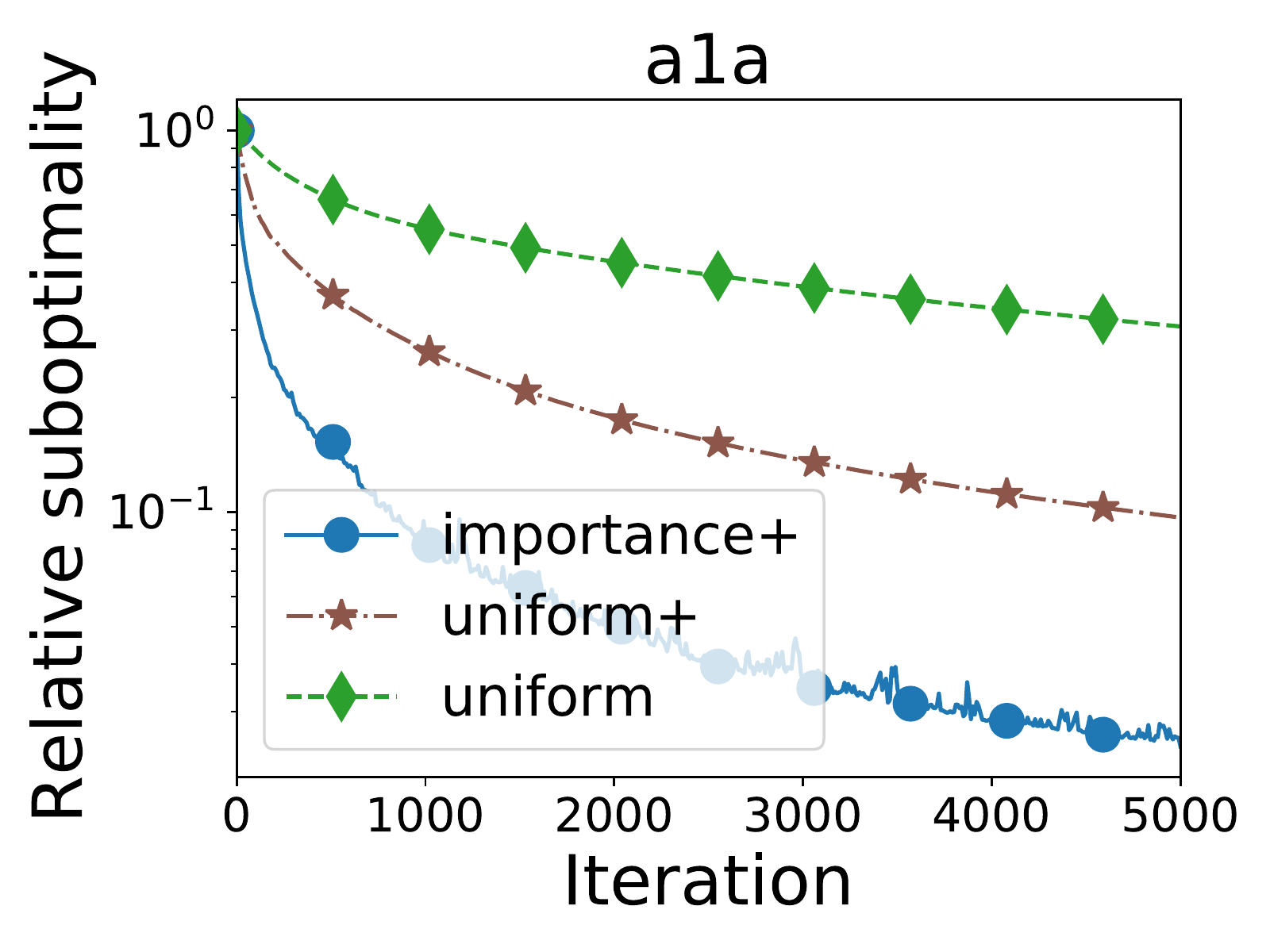}
\end{minipage}
\begin{minipage}{0.3\textwidth}
  \centering
\includegraphics[width =  \textwidth ]{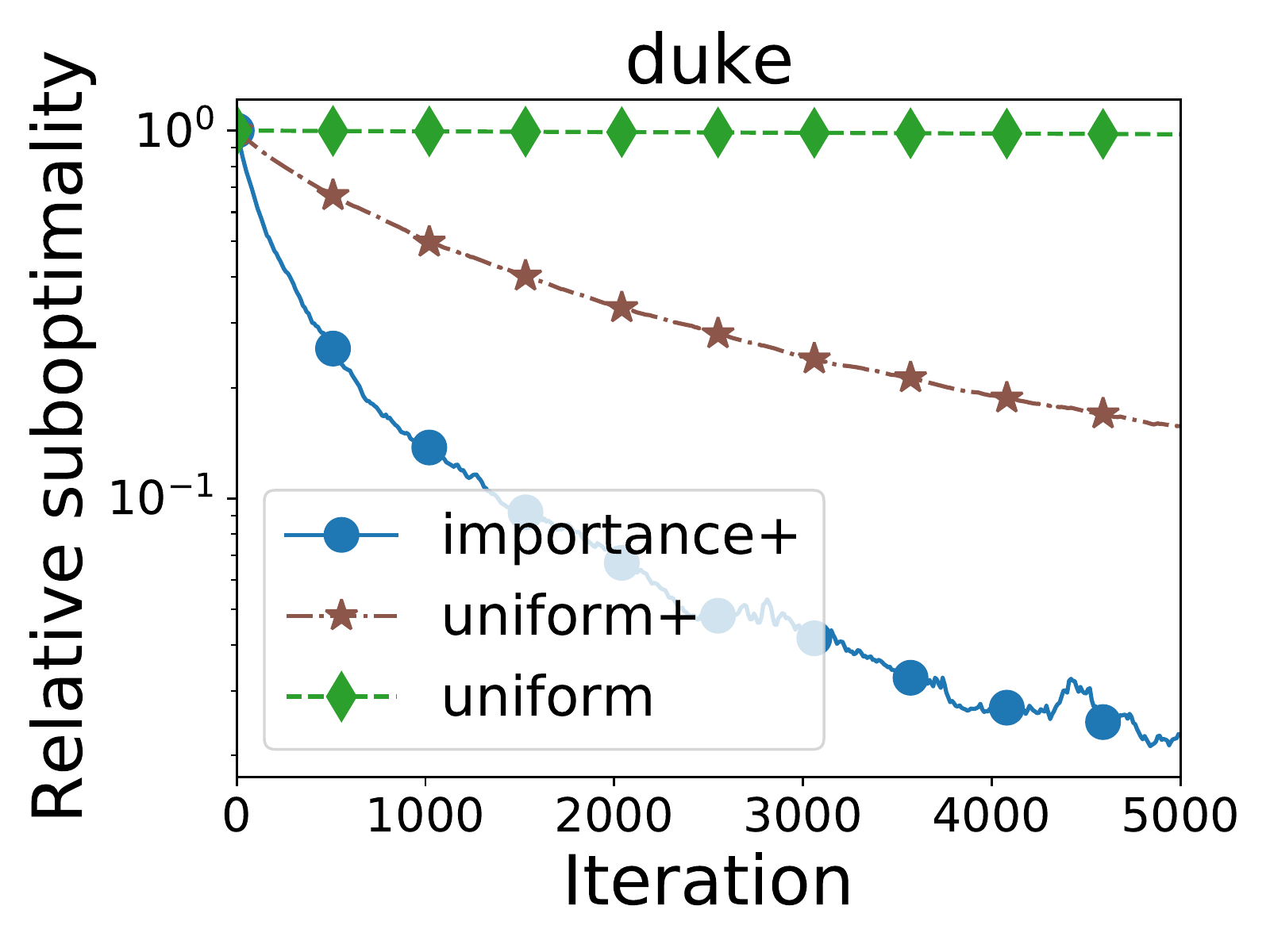}
\end{minipage}%
\begin{minipage}{0.3\textwidth}
  \centering
\includegraphics[width =  \textwidth ]{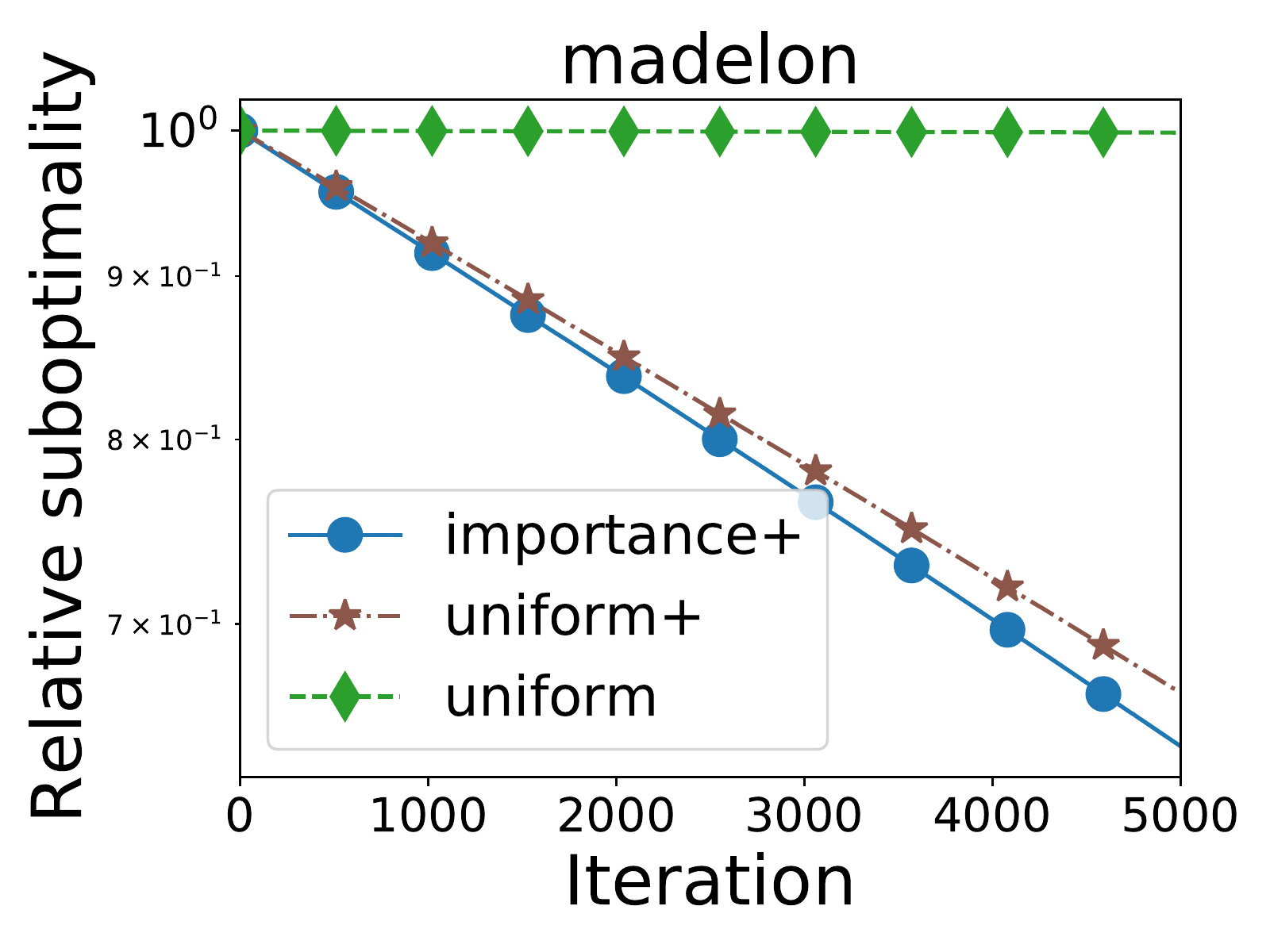}
\end{minipage}
\\
\begin{minipage}{0.3\textwidth}
  \centering
\includegraphics[width =  \textwidth ]{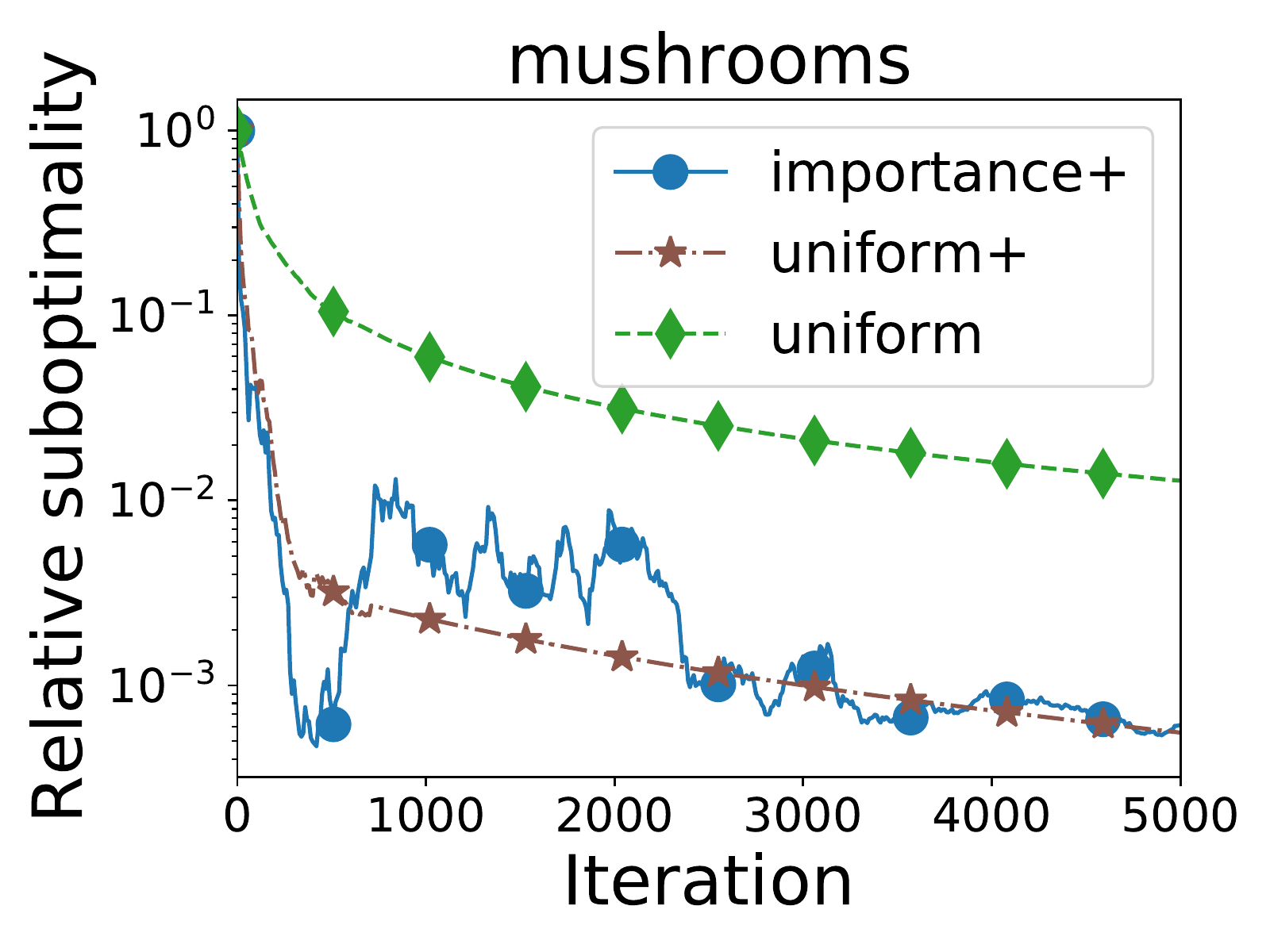}
\end{minipage}%
\begin{minipage}{0.3\textwidth}
  \centering
\includegraphics[width =  \textwidth ]{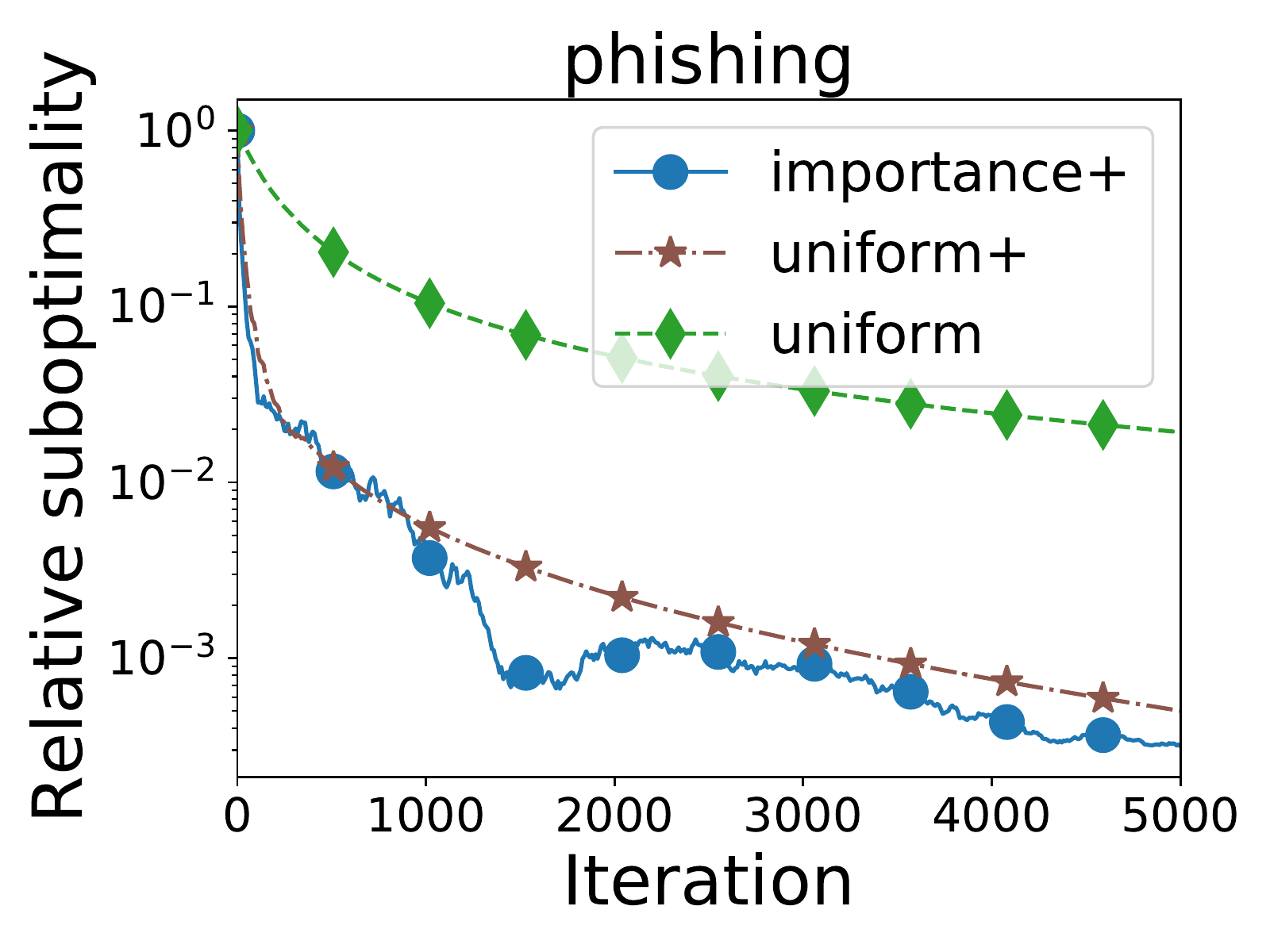}
\end{minipage}
\begin{minipage}{0.3\textwidth}
  \centering
\includegraphics[width =  \textwidth ]{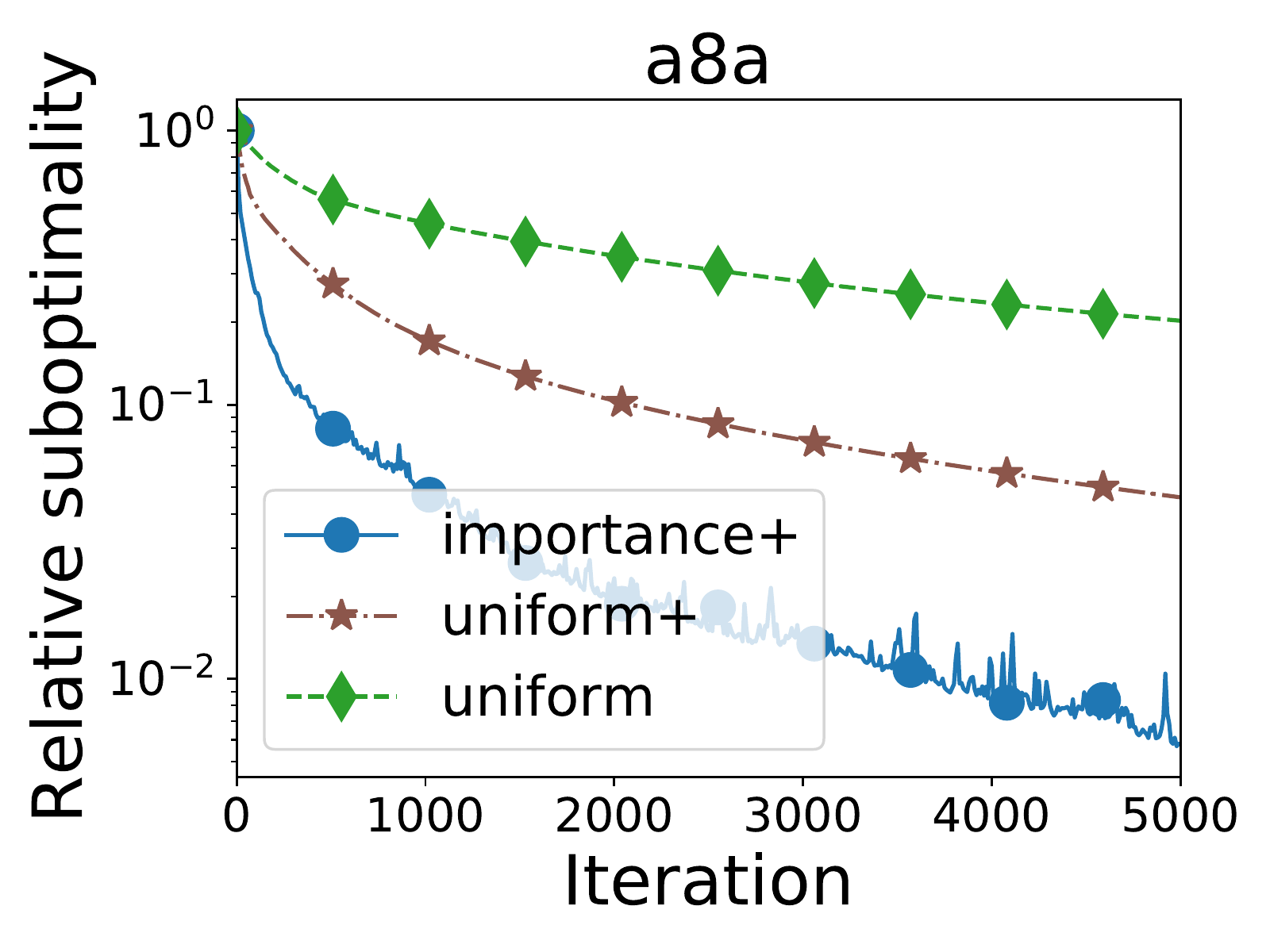}
\end{minipage}%
\caption{
Comparison of our sparsification strategy of size $\tau=1$ for DIANA+ (Algorithm~\ref{alg:DIANA+}) using  i) importance sampling with probabilities (\ref{ind-sampling-probs-diana}),  ii) uniform sampling with $p_i =(\frac1d, \frac1d,  \dots \frac1d)^\top$ and iii) DIANA~\citep{MGTR} using standard sparsification scheme with uniform sampling. All methods are run with stepsizes as dictated by  theory.}
\label{fig:sampling}
\end{figure}

\subsection{The proposed and usual sparsification techniques for the 3 distributed methods}

In the second experiment depicted in Figure~\ref{fig:acc_vr2},  we compare six different methods: well-established DCGD, DIANA, ADIANA and our methods DCGD+, DIANA+, ADIANA+, all with uniform sampling for $\tau=1$.  In order to highlight the importance of the variance reduction, in this experiment we choose the starting point to be close to the optimum.

\begin{figure}[!h]
\centering
\begin{minipage}{0.3\textwidth}
  \centering
\includegraphics[width =  \textwidth ]{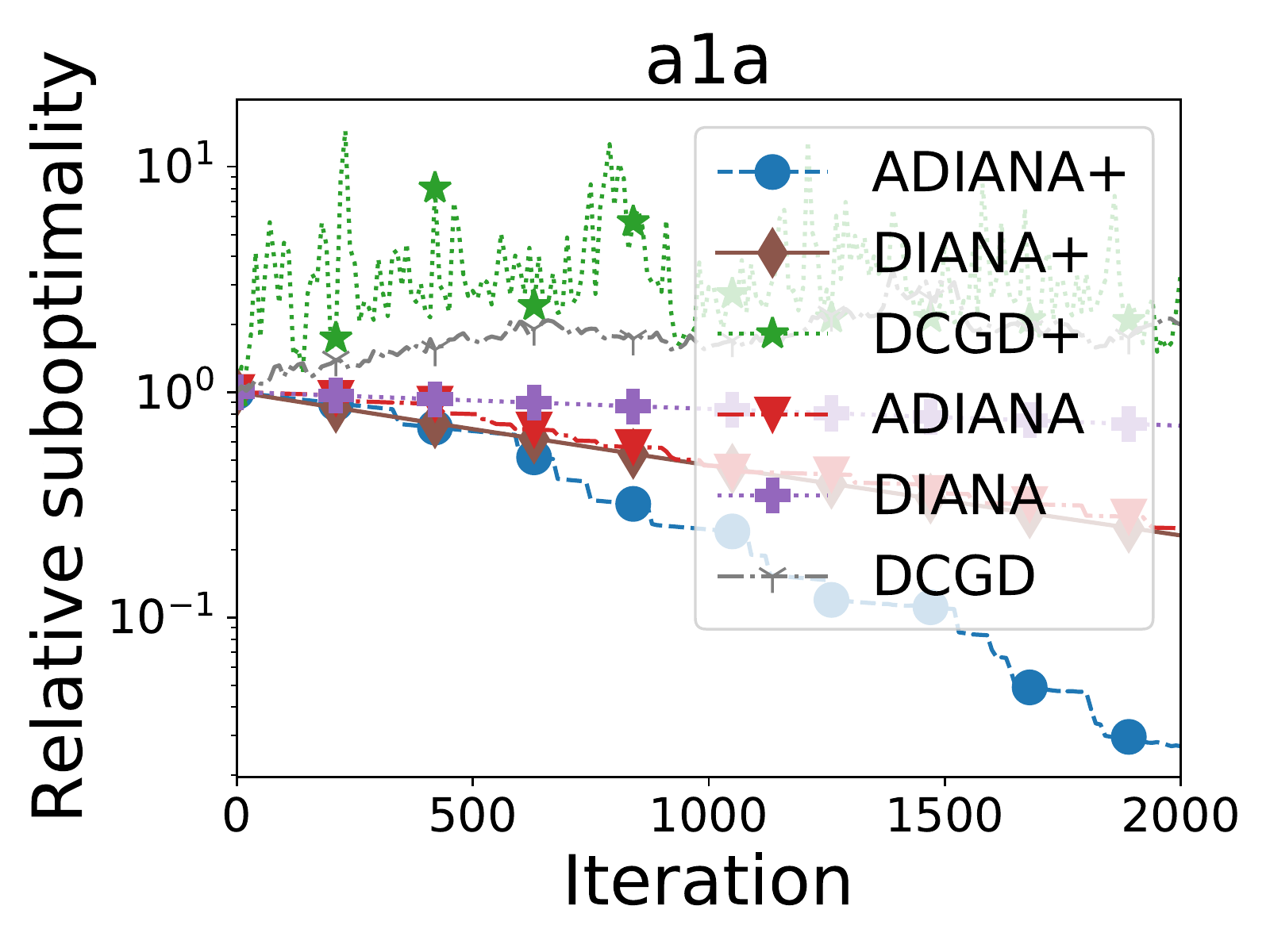}
\end{minipage}
\begin{minipage}{0.3\textwidth}
  \centering
\includegraphics[width =  \textwidth ]{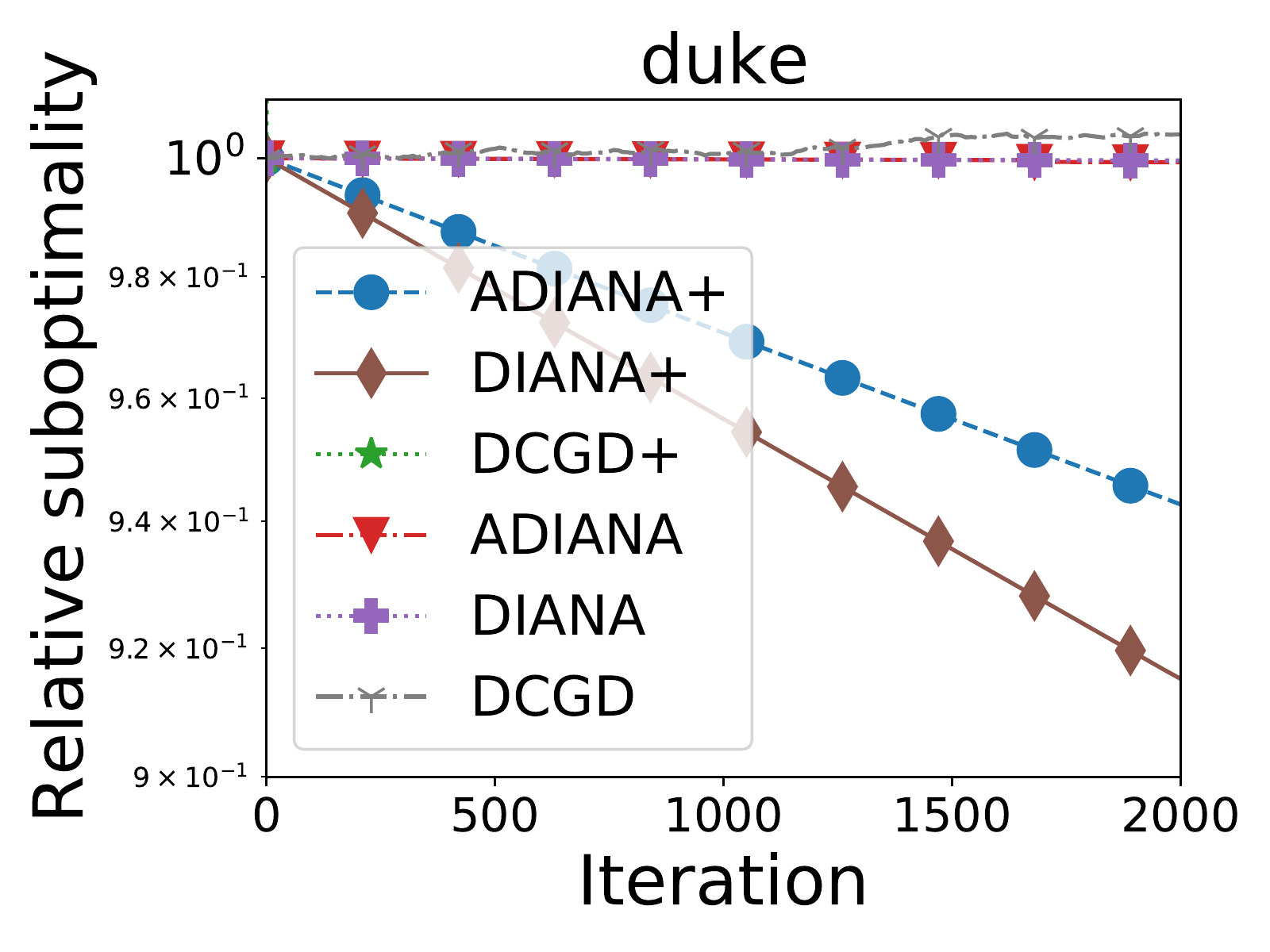}
\end{minipage}
\begin{minipage}{0.3\textwidth}
  \centering
\includegraphics[width =  \textwidth ]{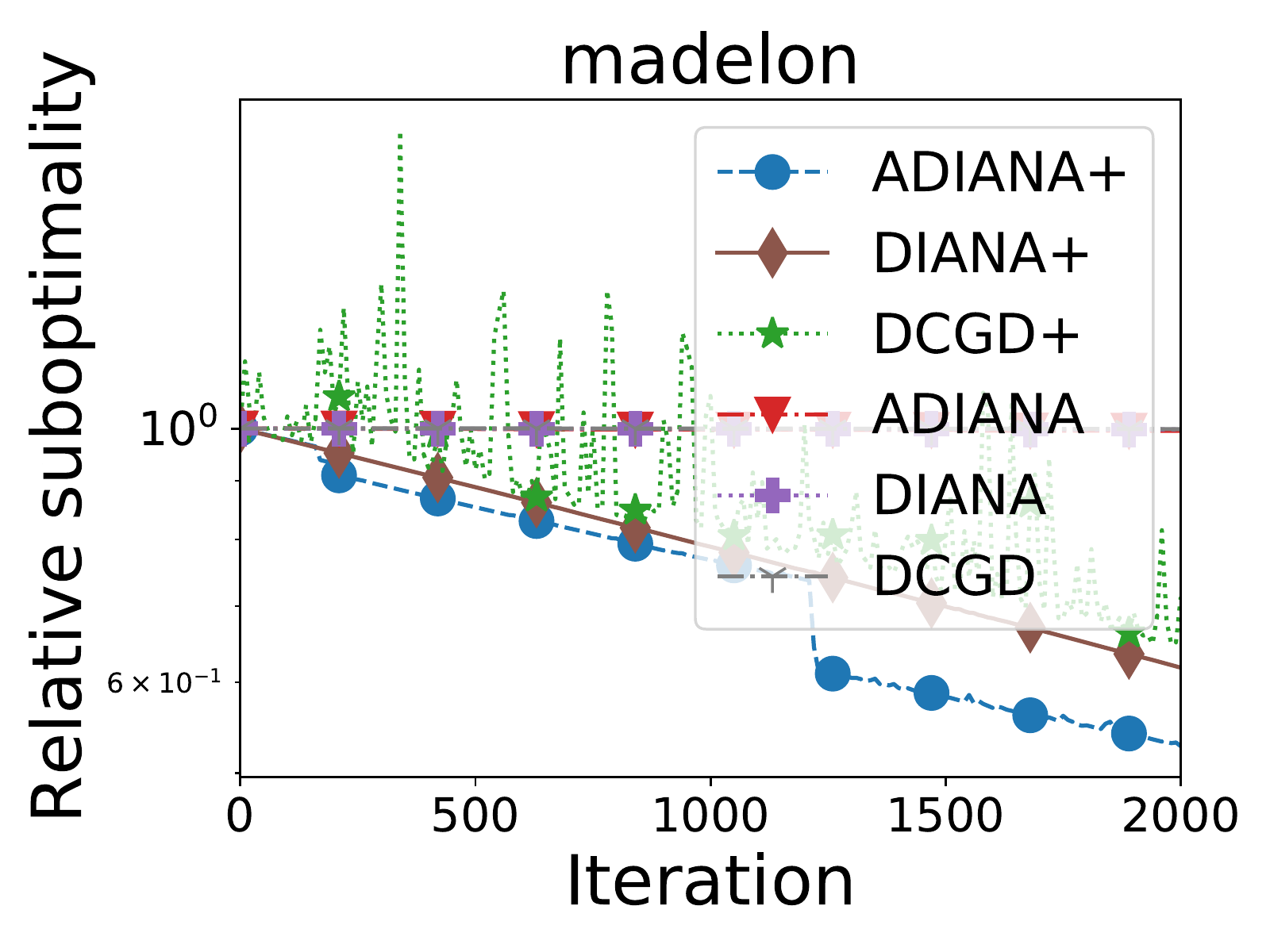}
\end{minipage}
\\
\begin{minipage}{0.3\textwidth}
  \centering
\includegraphics[width =  \textwidth ]{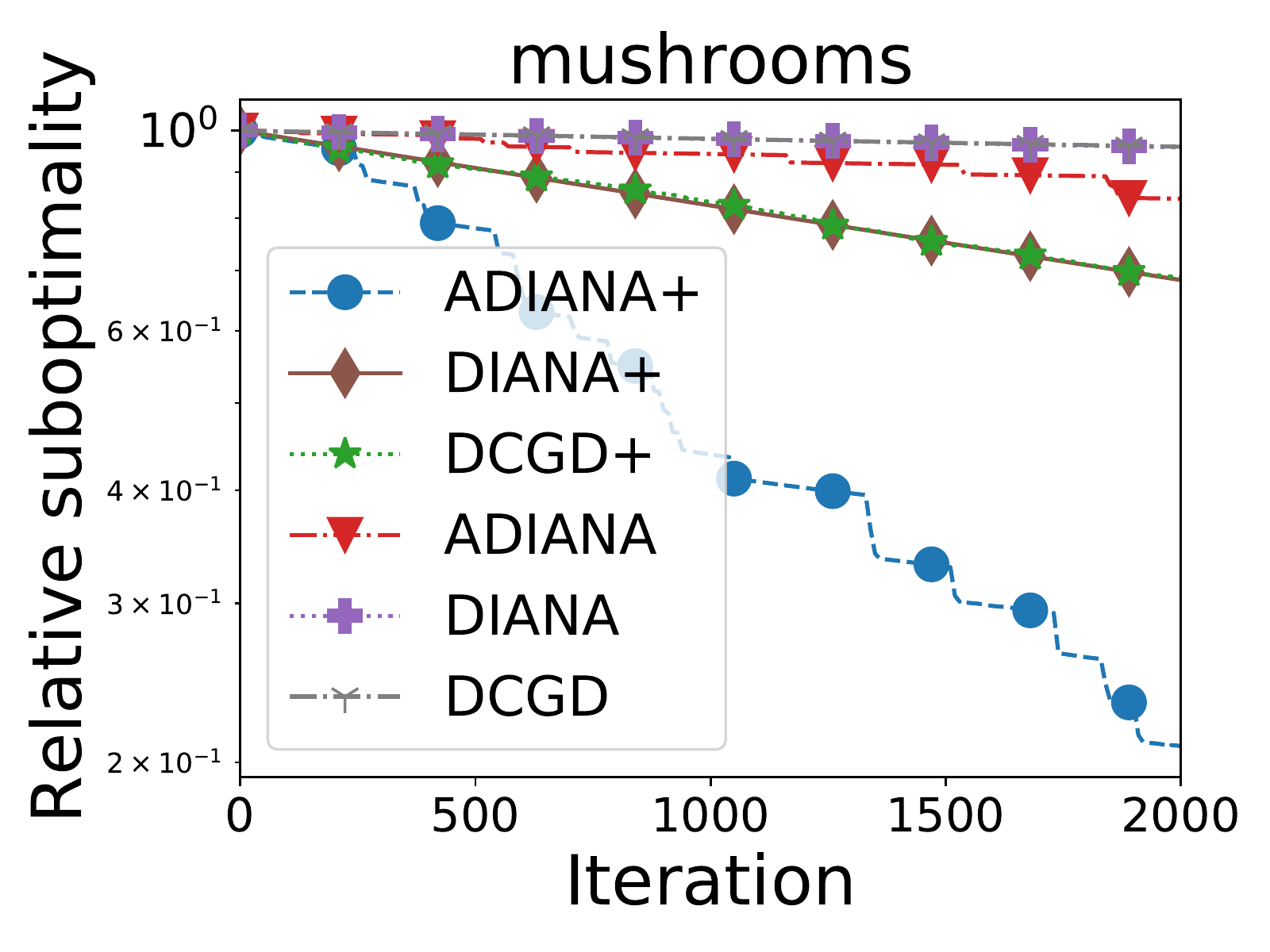}
\end{minipage}
\begin{minipage}{0.3\textwidth}
  \centering
\includegraphics[width =  \textwidth ]{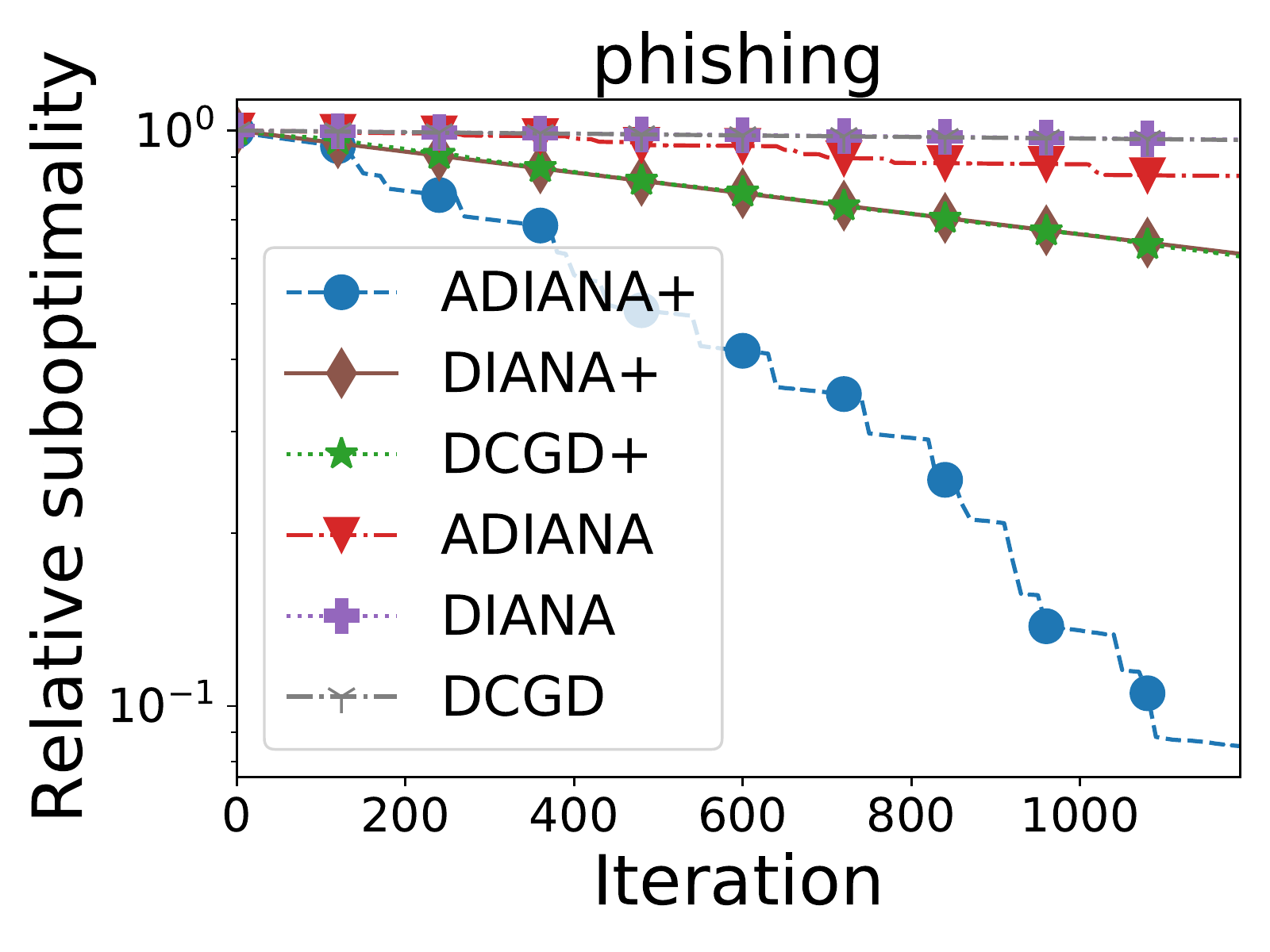}
\end{minipage}
\begin{minipage}{0.3\textwidth}
  \centering
\includegraphics[width =  \textwidth ]{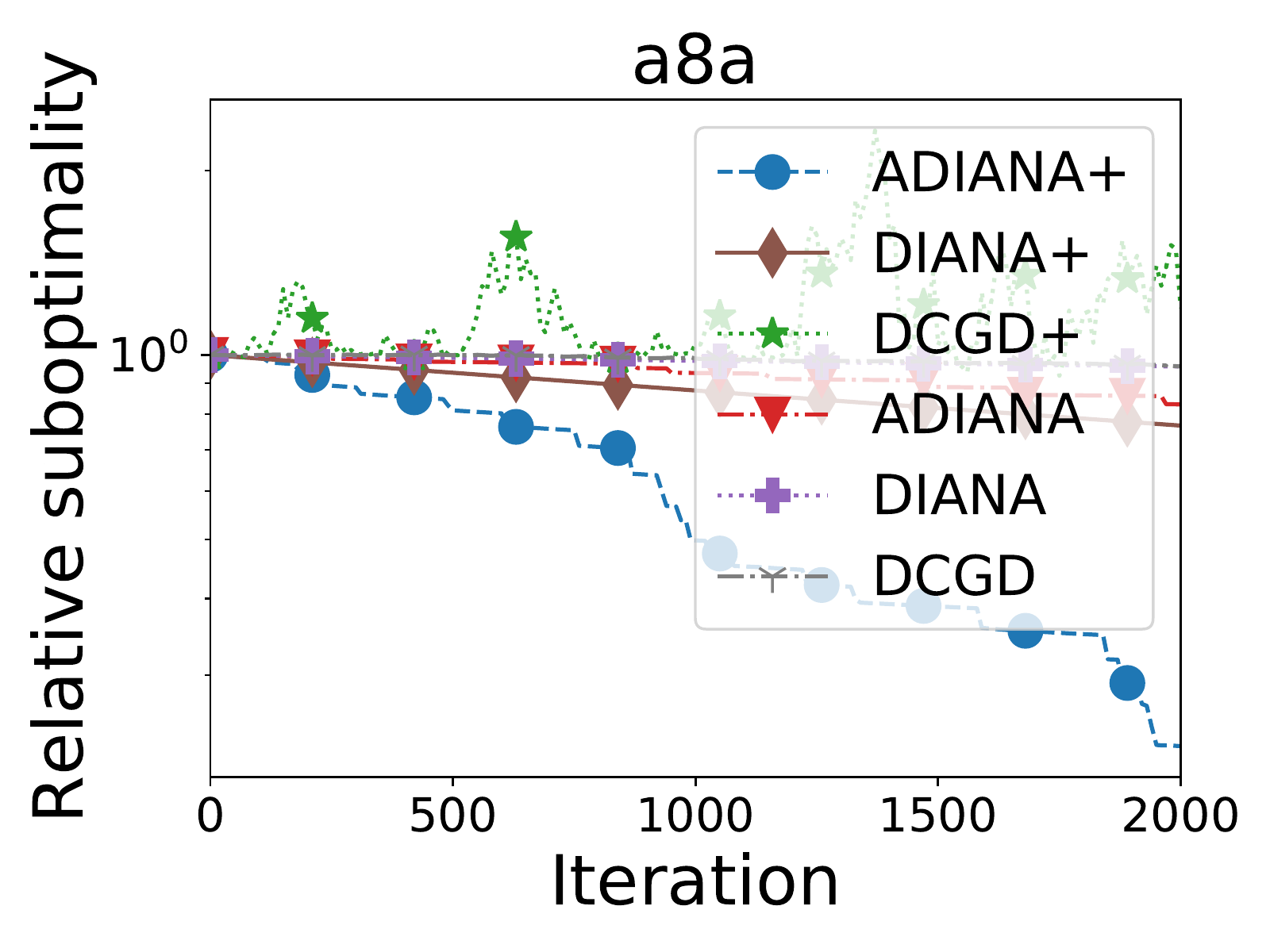}
\end{minipage}%
\caption{Comparison of the three original methods DCGD \citep{KFJ}, DIANA \citep{MGTR} and ADIANA \citep{AccCGD} with the proposed new methods DCGD+ (Alg. \ref{alg:dskgd}), DIANA+ (Alg. \ref{alg:DIANA+}) and ADIANA+ (Alg. \ref{alg:aDIANA+}). All six methods use uniform sampling with single mini-batch size $\tau=1$.}   
\label{fig:acc_vr2}
\end{figure}

Figure~\ref{fig:acc_vr2} demonstrates the following: i) methods with matrix-aware sparsification (i.e.,  DCGD+, DIANA+, ADIANA+) always outperform their baselines (i.e., DCGD, DIANA, ADIANA) ii) acceleration almost always outperforms the non-accelerated variant, often dramatically so and iii) variance reduction never hurts the convergence,  but often stabilizes the oscillation of the non-variance reduced counterpart.


\subsection{The effect of sparsification level $\tau$ on the convergence rate}

In this experiment, we study the effect of sparsification level $\tau$ on the convergence rate.  Informally speaking, our theory suggests that the sparsification does not hurt the convergence rate unless $\tau$ is smaller than some constant. The value of such constant depends on various factors such as the type of sampling and the specific smoothness structure of the objective.  

To contrast this with known results,~\citet{99KFP} show that the sparsification does not hurt ISEGA significantly (a method with sparsification unaware of smoothness matrix) as soon as $\tau n\geq d$. Addmitedly, ~\citet{99KFP} assume identical smoothness constants for both $f$ and $f_i$, so such a conclusion is slightly imprecise.  In our case,  ignoring the $\tilde{\omega}_{\max} $ factor, the rate is dominated by the sparsification factors only if $L = \cO\left( \frac{\tilde{\cL}_{\max}}{n}\right)$.

The results are presented in Fugure~\ref{fig:minibatch} (Iteration vs Residual) and Fugure~\ref{fig:minibatch2} (Communication vs Residual). As expected, we see that the sparsification only hurts the iteration complexity when $\tau$ is below certain treshold which is smaller for the uniform sampling compared to the importance sampling. Consequently, DIANA+ is capable of significantly reducing the worker->server communication at no cost in terms of the total iteration complexity.

\begin{figure}[!h]
\centering
\begin{minipage}{0.33\textwidth}
  \centering
\includegraphics[width =  \textwidth ]{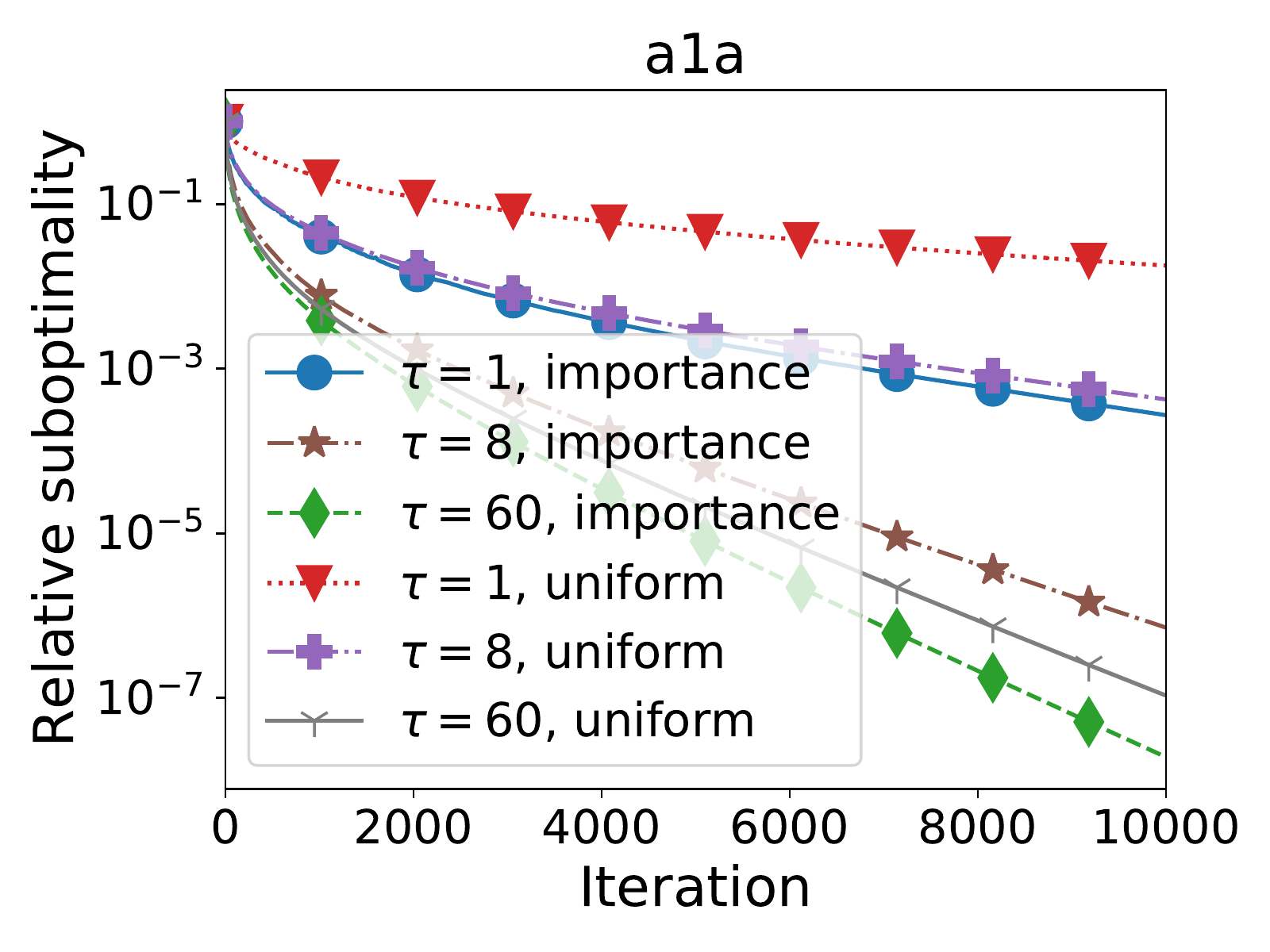}
\end{minipage}%
\begin{minipage}{0.33\textwidth}
  \centering
\includegraphics[width =  \textwidth ]{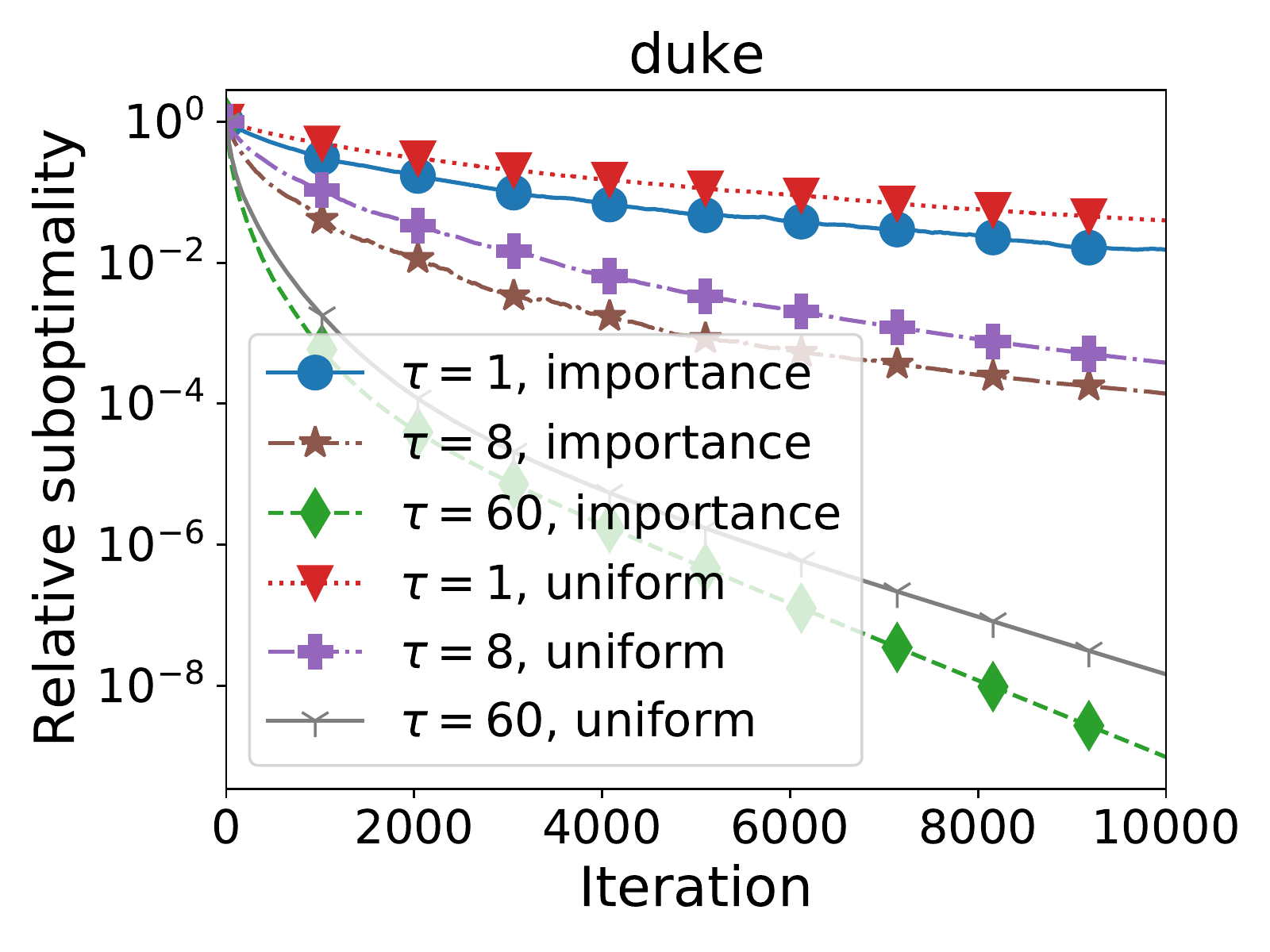}
\end{minipage}%
\begin{minipage}{0.33\textwidth}
  \centering
\includegraphics[width =  \textwidth ]{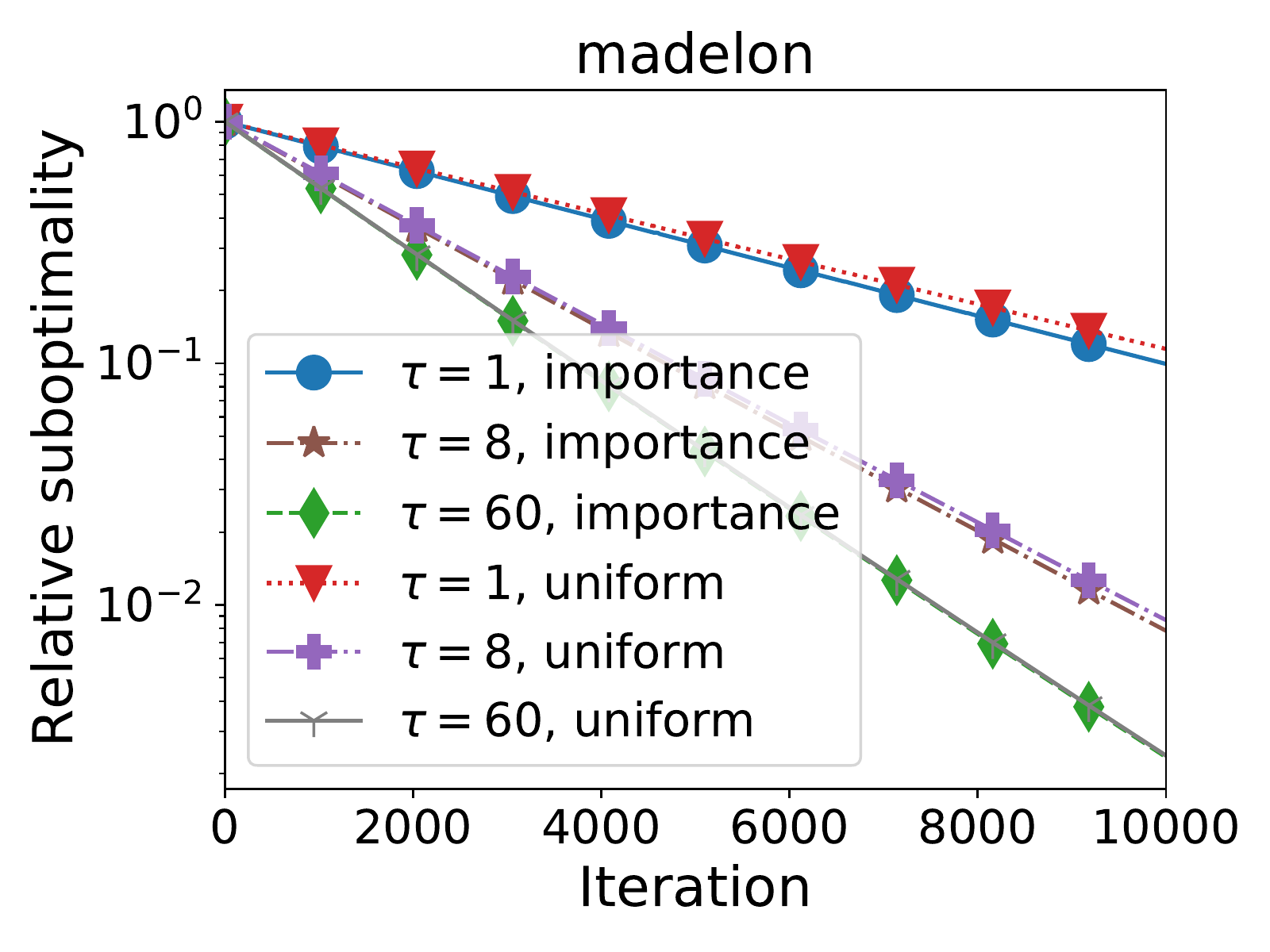}
\end{minipage}%
\\
\begin{minipage}{0.33\textwidth}
  \centering
\includegraphics[width =  \textwidth ]{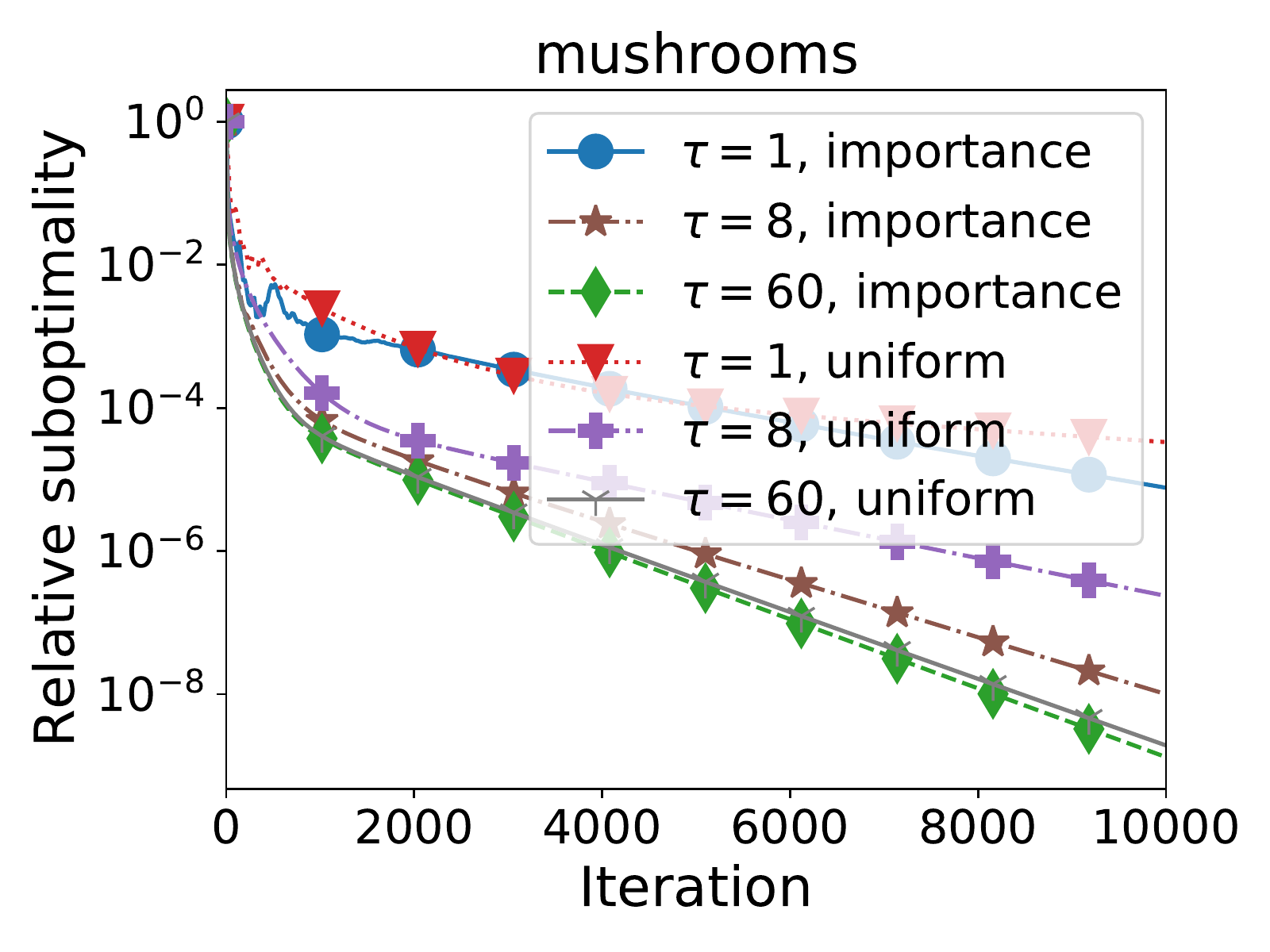}
\end{minipage}%
\begin{minipage}{0.33\textwidth}
  \centering
\includegraphics[width =  \textwidth ]{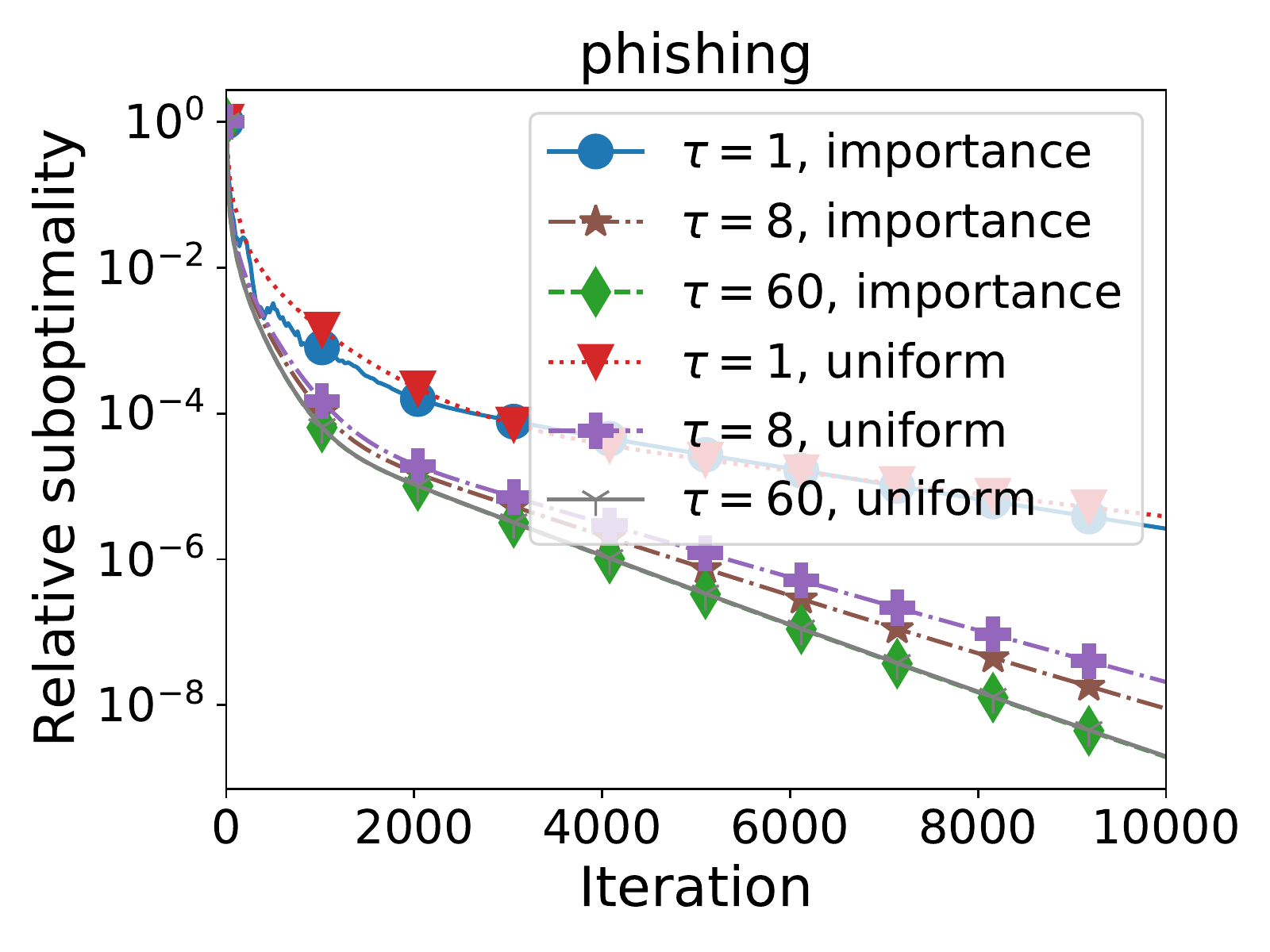}
\end{minipage}%
\begin{minipage}{0.33\textwidth}
  \centering
\includegraphics[width =  \textwidth ]{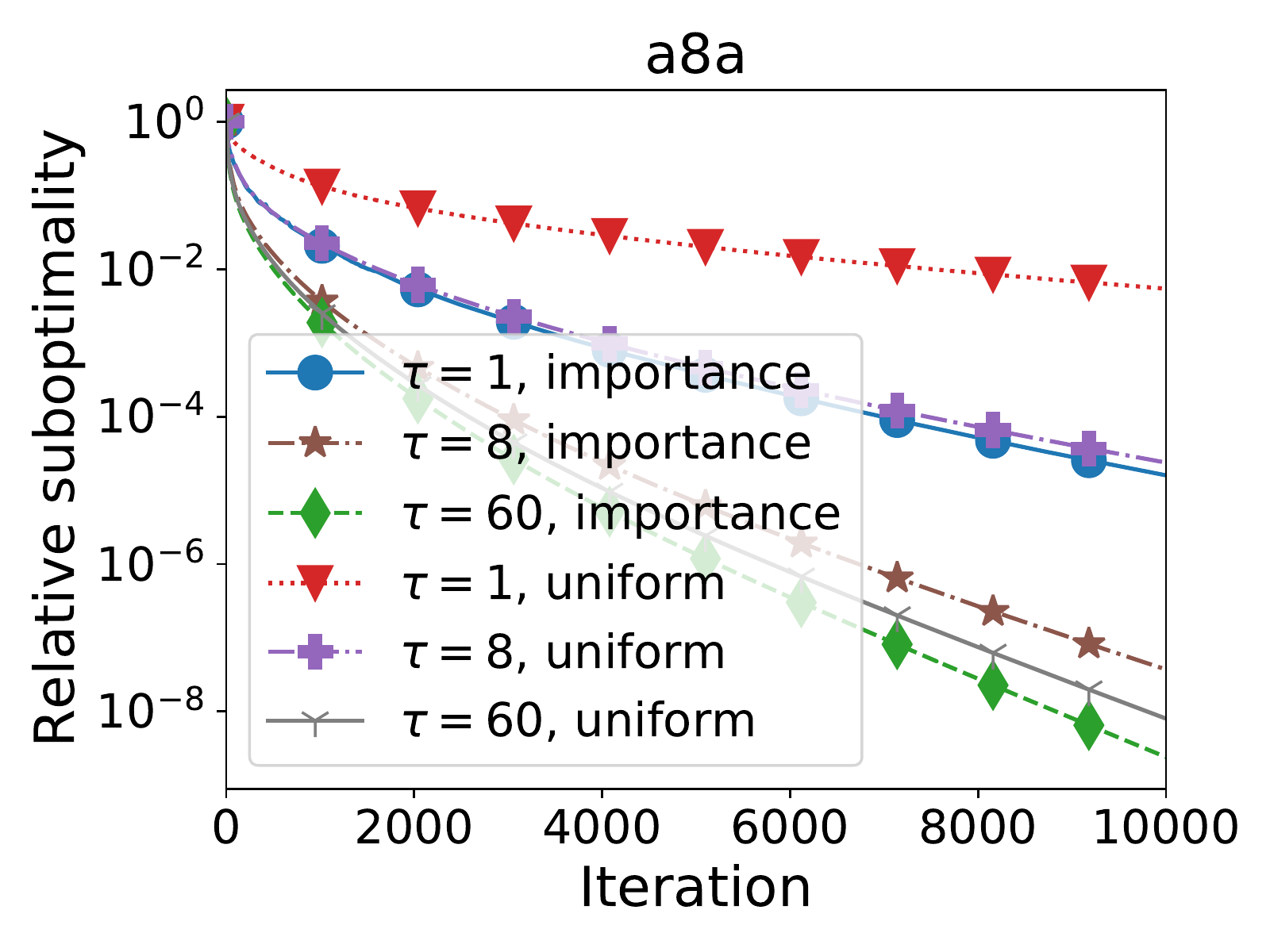}
\end{minipage}%
\caption{Effect of $\tau$ on the convergence speed of DIANA+ (Algorithm~\ref{alg:DIANA+}).}   
\label{fig:minibatch}
\end{figure}

\begin{figure}[!h]
\centering
\begin{minipage}{0.33\textwidth}
  \centering
\includegraphics[width =  \textwidth ]{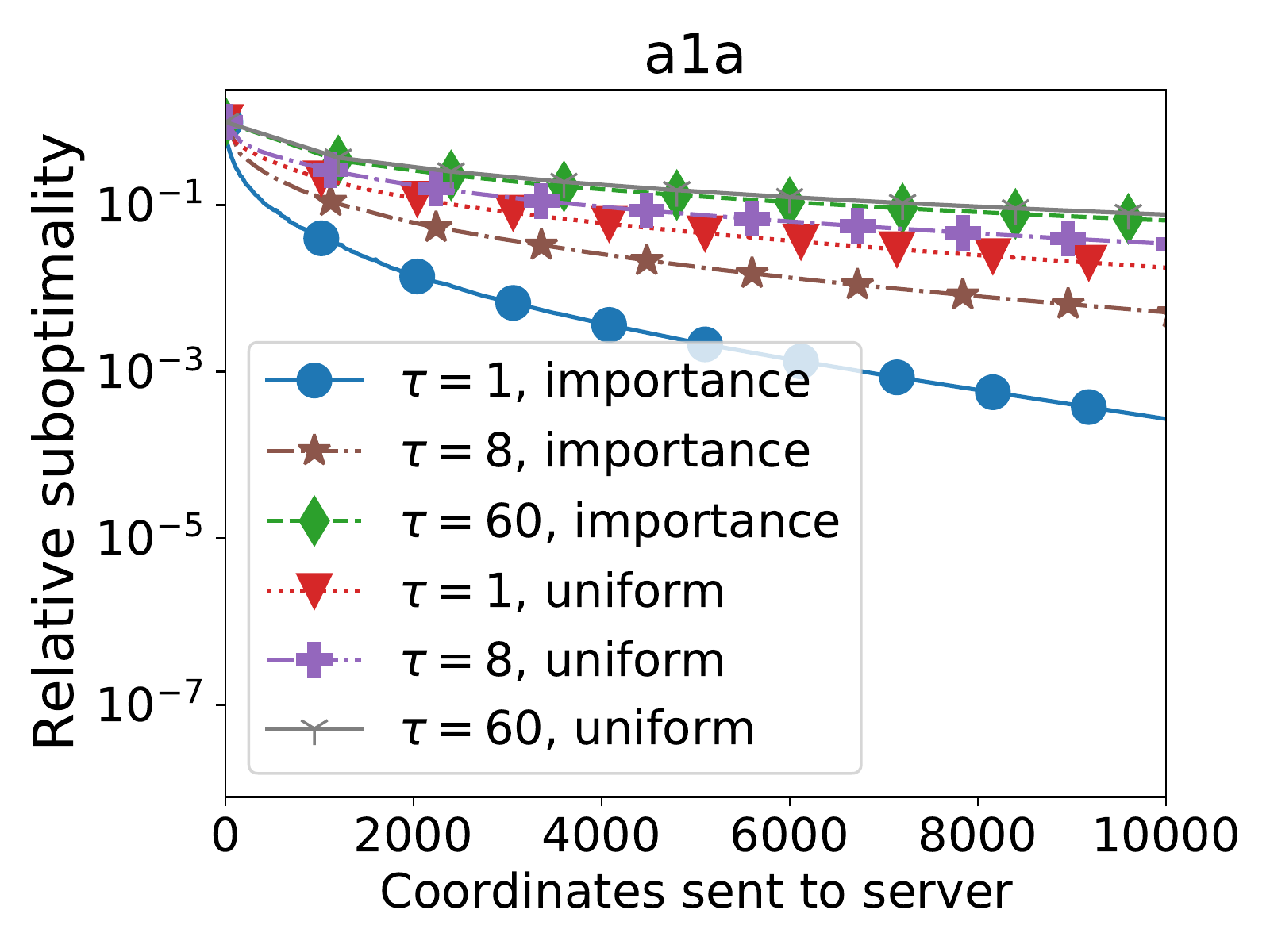}
\end{minipage}%
\begin{minipage}{0.33\textwidth}
  \centering
\includegraphics[width =  \textwidth ]{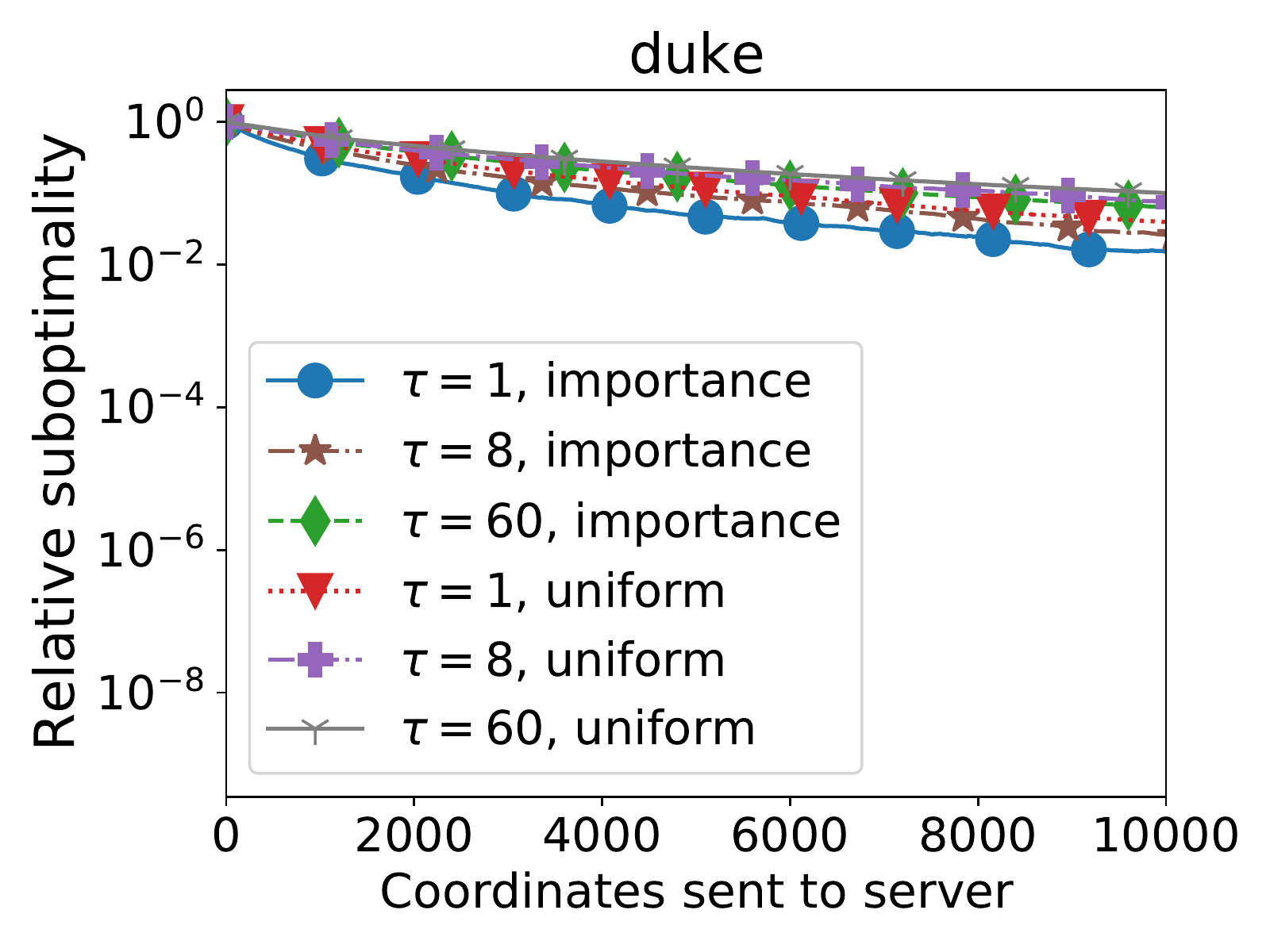}
\end{minipage}%
\begin{minipage}{0.33\textwidth}
  \centering
\includegraphics[width =  \textwidth ]{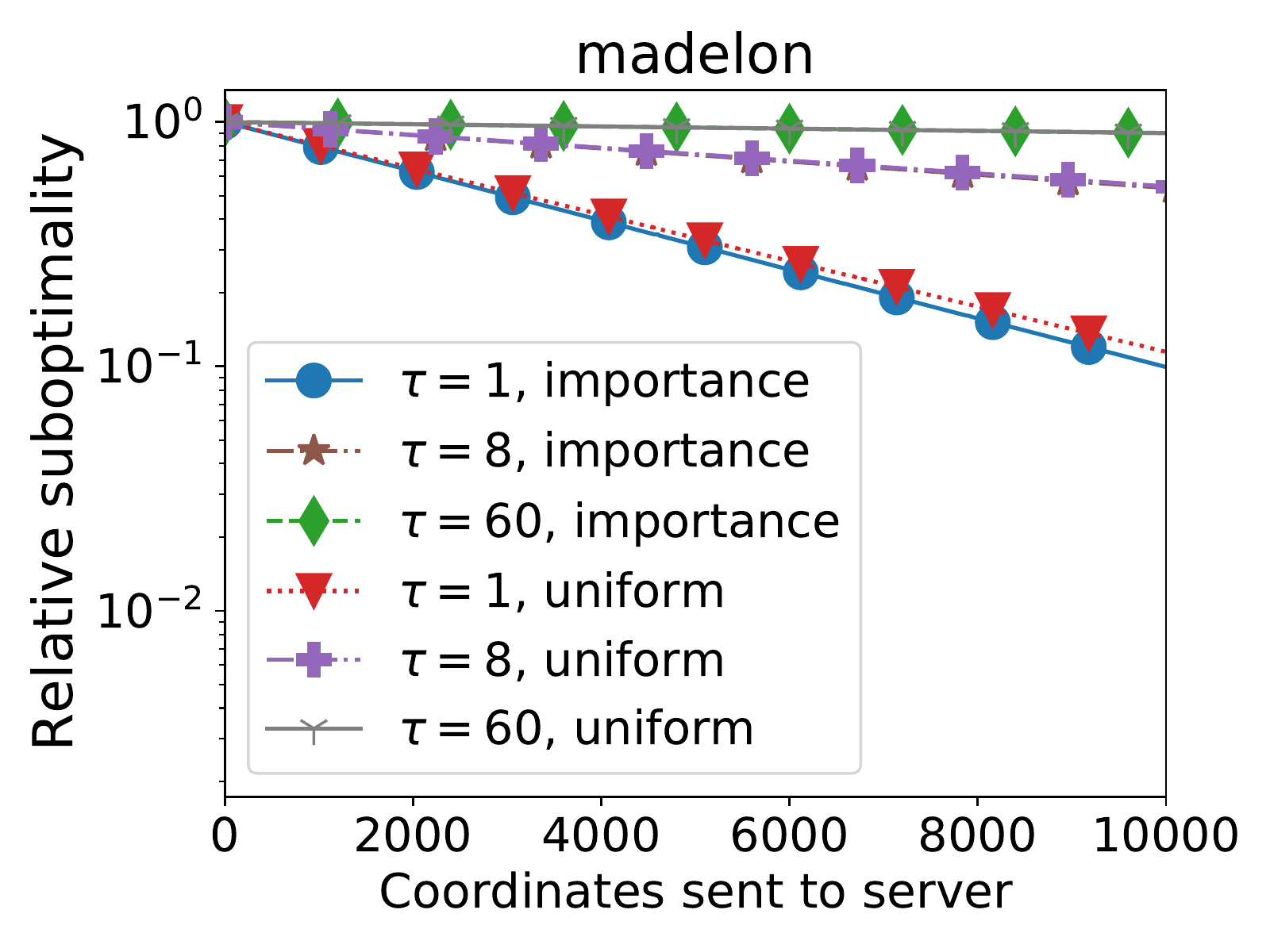}
\end{minipage}%
\\
\begin{minipage}{0.33\textwidth}
  \centering
\includegraphics[width =  \textwidth ]{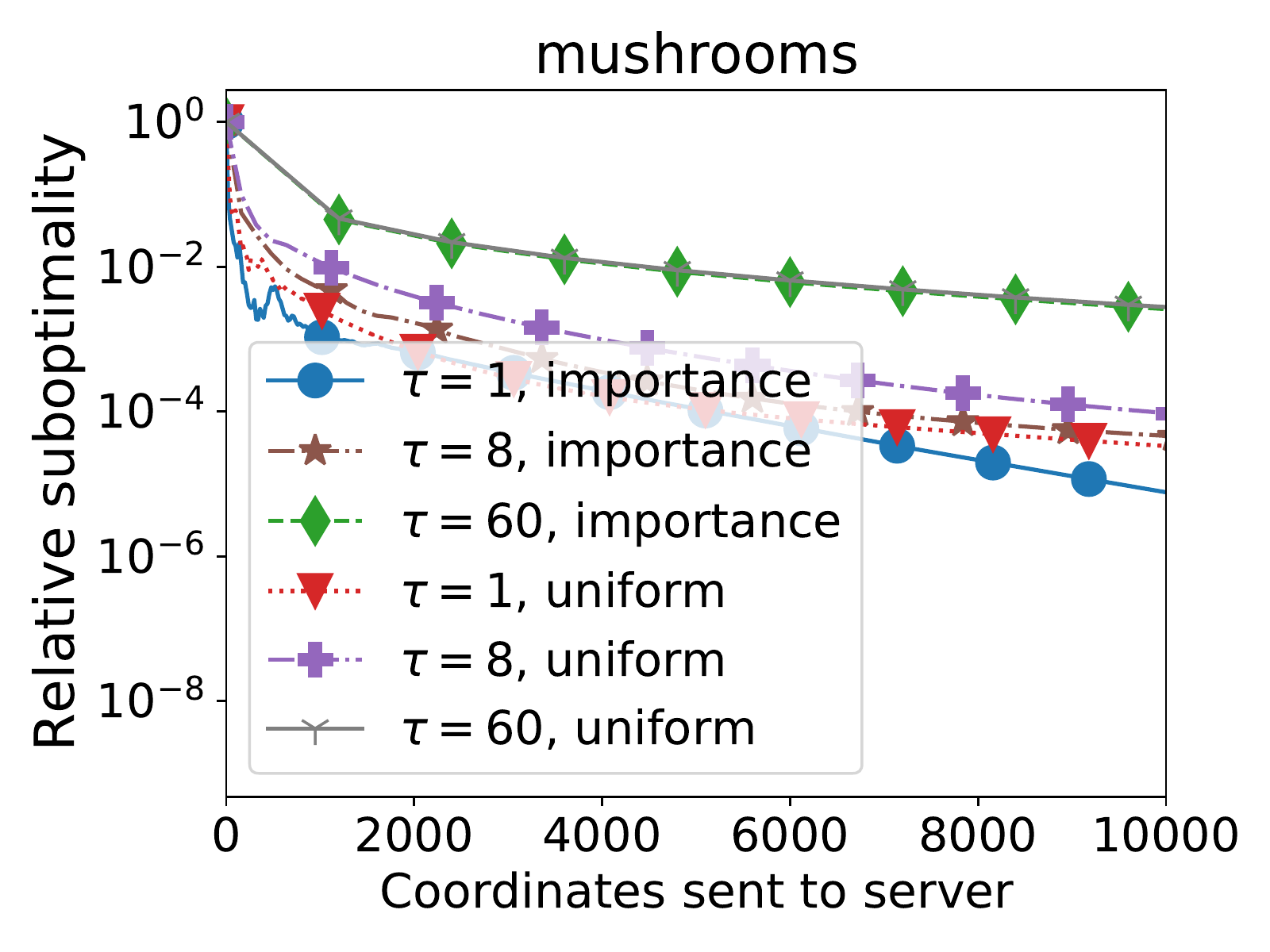}
\end{minipage}%
\begin{minipage}{0.33\textwidth}
  \centering
\includegraphics[width =  \textwidth ]{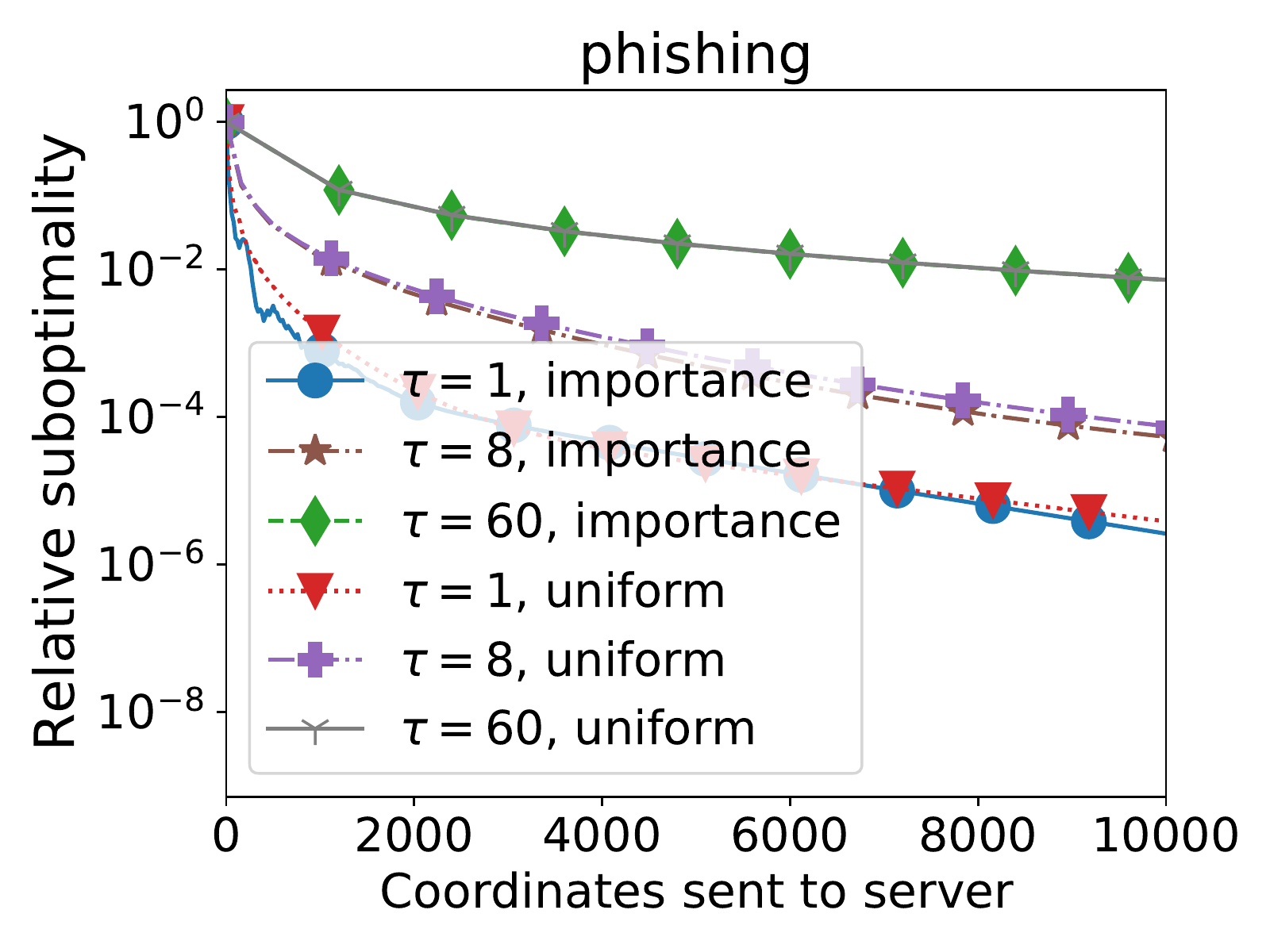}
\end{minipage}%
\begin{minipage}{0.33\textwidth}
  \centering
\includegraphics[width =  \textwidth ]{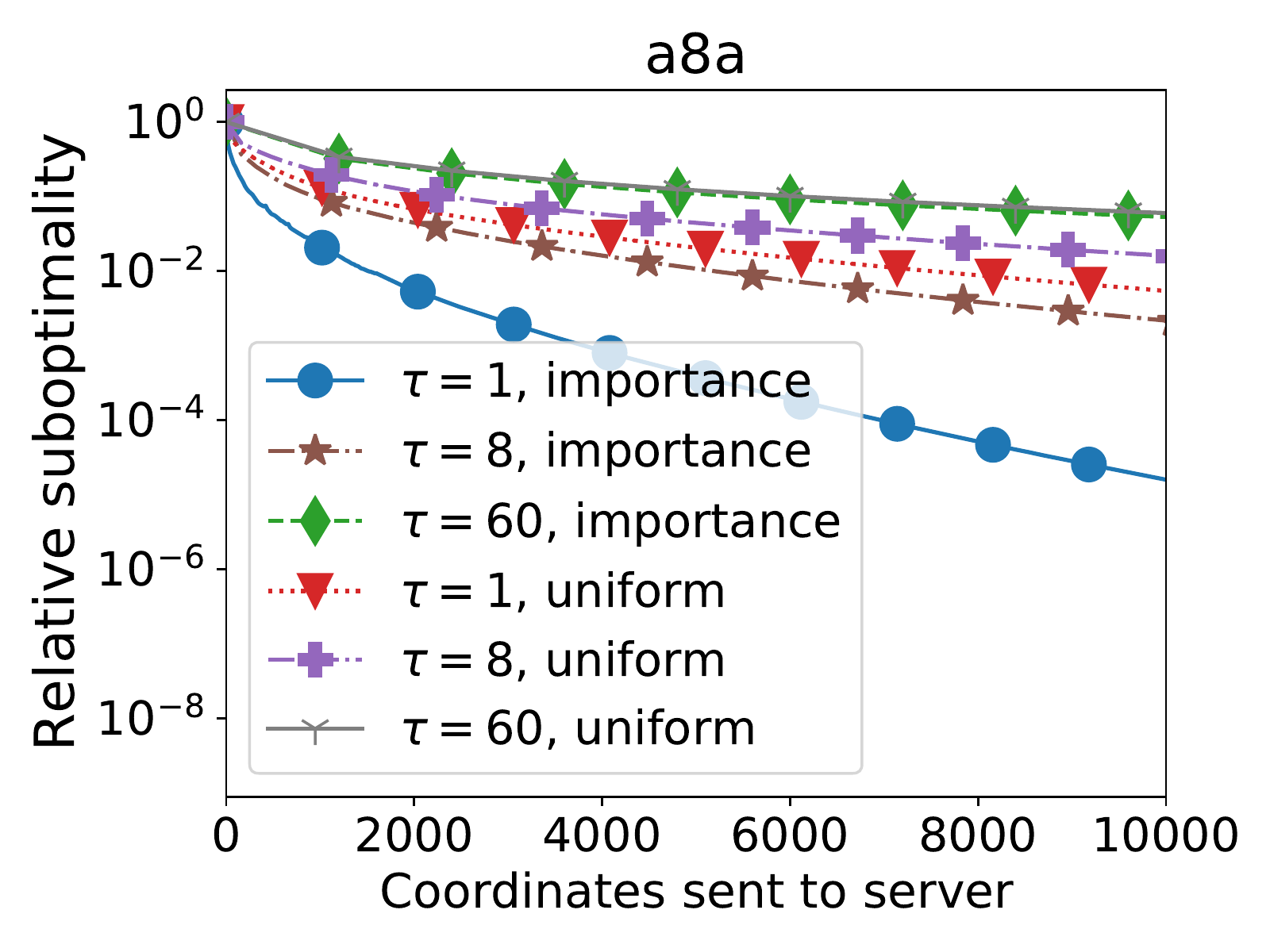}
\end{minipage}%
\caption{Same as Figure~\ref{fig:minibatch}, but $x$-axis corresponds to the coordinates sent to the server instead of the iteration.}   
\label{fig:minibatch2}
\end{figure}

\section{Conclusions,  Extensions and Future Work}

In this paper we have proposed a novel gradient sparsification technique for distributed optimization and demonstrated that it allows one to properly exploit the smoothness structure of the local objective. We have shown that the proposed matrix-smoothness-aware sparsification can be coupled with both the variance reduction and acceleration, providing further speedup in terms of the convergence rate and the total bits transmitted from workers to server.  Next, we list possible extensions of our work that we believe can or should be done in the future: 

\begin{itemize}
\item {\bf Subsampling the local objective.}  While DCGD+, DIANA+ and ADIANA+ all require an access to the full local gradient from each machine at every iteration, we believe this requirement can be easily dropped.  In particular, the local objective can be further subsampled and extra variance reduction can be employed on top of these methods, similarly to as done for ISAEGA~\citep{GJS-HR}.
\item {\bf Greedy sparsification.}  Notice that the sparsified local gradient can be seen as a randomized coordinate descent estimator of a given machine. However, greedy coordinate descent was shown to outperform randomized coordinate descent in certain scenarios~\citep{nutini2017let}. Therefore, one might pose a question whether a greedy sparsification might work for distributed optimization. 

\item {\bf Bi-directional sparsification.} As we also mention in Section~\ref{sec:limit}, one drawback of our approach\footnote{In fact, this is a drawback of the vast majority of compression methods from the literature. A notable exception is DoubleSqueeze~\citep{DoubleSqueeze2019} which compresses the server$\rightarrow$worker communication too.} is that only worker$\rightarrow$server communication is sparse.  It would be very interesting to develop a bi-directional sparsification capable of properly exploiting the smoothness matrices. For this matter, in Section \ref{apx-thm:DIANA++} we develop and analyze DIANA++ method employing bi-directional matrix-smoothness-aware sparsification and twofold variance reduction.

\item {\bf Weakly convex and non-convex cases.} While we state our theory for the strongly convex case (i.e., Assumpiton~\ref{asm:mu-convex}), it can be rather easily extended to weakly convex case (i.e., $\mu=0$). However, obtaining an efficiennt smoothness matrix aware sparsification for non-convex optimization remains an open problem. 
\end{itemize}

\section{Limitations}\label{sec:limit}

Next,  we discuss main limitations of our approach. 

\begin{itemize}
\item The server is required to store matrices $\ML_i^{\half}$ for all machines $i\in[n]$ and multiply them by sparse updates $\MC_i^k\ML_i^\phalf\nabla f_i(x^k)$ in each iteration.  Therefore, our method is not expected to be practical when $d$ is large and matrices $\ML_i$ are not of a special structure so that they are cheap to store and so that $\MC_i^k\ML_i^\phalf\nabla f_i(x^k)$ can be evaluated cheaply.\footnote{For example, if $\ML_i$ is of rank $r$, for all $i$, we require extra $\cO(ndr)$ storage and $\cO(ndr)$ flops at the server at each iteration.}. On the other hand, our strategy is still practical when i) $d$ is small or ii)  $\ML_i$ is of a special structure such as low rank or diagonal. In particular, diagonal $\ML_i$ requires only $\cO(\tau)$ extra computation per each node (which is negligible), while attaining a rate which is never worse compared to the naive sparsification. 

\item Except DIANA++ method presented in Section \ref{apx-thm:DIANA++}, we sparsify only the communication from the workers to server. Sparsifying workers$\rightarrow$server communication only is very common in the area of distributed optimization as the workers$\rightarrow$server communication is significantly more expensive compared to the server$\rightarrow$workers communication. Such a phenomenon can be assigned to the fact that the server is broadcasting the same vector to all workers, and thus the server$\rightarrow$workers communication can be implemented more efficiently. 
\end{itemize}

\begin{remark}
The overhead that comes from the computation of $\ML_i^\phalf\nabla f_i(x^k)$ is not an issue in general. Given that $\ML_i$ is of rank $r$,  one requires $\cO\left(d^2r\right)$ flops to precompute SVD of $\ML_i$. Given that SVD of $\ML_i$ is known, the evaluation of $\ML_i^\phalf\nabla f_i(x^k)$ takes only $\cO\left(d^2r\right)$ flops. While the cost of computing $\nabla f_i(x^k)$ varies depending on the application, we can expect it to takes at least  $\Omega\left(d^2r\right)$ flops for the application of generalized linear models (i.e., logistic regression).  Next, we shall mention that evaluating $\ML_i^\phalf\nabla f_i(x^k)$ comes at $\cO\left(d\right)$ cost when $\ML_i$ is diagonal.
\end{remark}

\clearpage
\bibliography{references}
\bibliographystyle{plainnat}

\clearpage

\appendix
\part*{Appendix}

\section{Table of Frequently Used Notation}

\begin{table*}[ht!]
\scriptsize
    \centering
    \caption{Notation used throughout the paper}\label{table:notations}
    \renewcommand{\arraystretch}{1.7}
    \begin{tabular}{|c|c|c|}
        \hline
        \hline        
       {\bf Symbol }& {\bf Description } & {\bf Reference} \\
        \hline
        $d$ & dimension of the model $x\in\R^d$ &  (\ref{main-opt-problem}) \\
        \hline
        $\mu$ & strong convexity parameter of $f$  &  Asm. \ref{asm:mu-convex} \\
        \hline
        $\ML$ & smoothness matrix of $f$ &  Asm. \ref{asm:Li-smooth-convex} \\
        \hline
        $\ML_{ij}$ & the element at $i$th row and $j$th column of $\ML$ &  - \\
        \hline
        $\ML_i$ & smoothness matrix of $f_i$ &  Asm. \ref{asm:Li-smooth-convex} \\
        \hline        
        $L_i$ & smoothness constant of $f_i(x)$, i.e., $L_i = \lambda_{\max}(\ML_i)$ &  - \\
        \hline
        $L$ & smoothness constant of $f$, i.e., $L=\lambda_{\max}(\ML)$ &  - \\
        \hline
        $S$ & random sampling (subset) of coordinates $[d]\eqdef\{1,2,\dots,d\}$ &  - \\
        \hline
        $p_{jl},\; p_j$ & $p_{jl} \eqdef \Pr\(\{j,l\}\subseteq S\), \; p_j \eqdef p_{jj}$ &  - \\
        \hline
        $\MP$ & the probability matrix $(p_{jl})_{j,l=1}^d$ associated with random sampling $S$ &  (\ref{def:P}) \\
        \hline
        $v_i$ & ESO parameters associated with $f$ and $S$ jointly &  - \\
        \hline
        $\MC$ & diagonal sketch matrix with $i$th random variable $\mathrm{c}_i = \nicefrac{1}{p_i}$ if $i\in S$ and $0$ otherwise &  (\ref{sketch-matrix-C}) \\
        \hline
        $\omega$ & variance of general compression operator $\cC$, i.e. $\E\[\|\cC(x)-x\|^2\] \le \omega\|x\|^2,\; \forall x\in\R^d$ &  - \\
        \hline
        $\overbar{\MC},\; \overbar{\MC}_i^k$ & $\overbar{\MC} \eqdef \ML^{\half}\MC\ML^{\phalf}, \; \overbar{\MC}_i^k = \ML_i^{\half}\MC_i^k\ML_i^{\phalf}$ &  - \\
        \hline
        $\MI,\; \ME$ & the identity matrix and the matrix with all entries equal to $1$ &  - \\
        \hline
        $\overbar{\MP},\; \widetilde{\MP}$ & $\overbar{\MP}=\Mdiag\(\nicefrac{1}{p}\)\MP\Mdiag\(\nicefrac{1}{p}\)$ with entries $\overbar{p}_{ij} = \frac{p_{ij}}{p_ip_j}$, and $\widetilde{\MP} = \overbar{\MP} - \ME$ &  (\ref{def:P}) \\
        \hline
        $\overbar{\cL},\; \widetilde{\cL}$ & expected smoothness constants $\overbar{\cL} = \lambda_{\max}(\overbar{\MP}\circ\ML), \; \widetilde{\cL} = \lambda_{\max}(\widetilde{\MP}\circ\ML)$ &  - \\
        \hline
        $n$ & number of parallel machines in distributed setting &  (\ref{main-opt-problem-dist}) \\
        \hline
        $\MC_i, \MP_i, \overbar{\MP}_i, \widetilde{\MP}_i$ & diagonal sketch matrix and probability matrices for $i$th worker &  (\ref{sketch-matrix-C}), (\ref{def:P}) \\
        \hline
        $p_{i;j}, \; \overbar{p}_{i;j}, \; \widetilde{p}_{i;j}$ & $j$-th diagonal element of $\MP_i,\; \overbar{\MP}_i,\; \widetilde{\MP}_i$ &  - \\
        \hline
        $\omega_i$ & variance of compression operator induced by $\MC_i$, i.e. $\omega_i = \max_{1\le j\le d}\frac{1}{p_{i;j}}-1$ &  - \\
        \hline
        $\omega_{\max}$ & $\max_{1\le i \le n}\omega_i = \max_{1\le i \le n}\max_{1\le j\le d}\frac{1}{p_{i;j}}-1$ &  (\ref{DIANA+-complexity}) \\
        \hline
        $\overbar{\cL}_i,\; \widetilde{\cL}_i$ & expected smoothness constants, $\overbar{\cL}_i = \lambda_{\max}(\overbar{\MP}_i\circ\ML_i), \; \widetilde{\cL}_i = \lambda_{\max}(\widetilde{\MP}_i\circ\ML_i)$ &  - \\
        \hline
        $\overbar{\cL}_{\max},\; \widetilde{\cL}_{\max}$ & $\overbar{\cL}_{\max} = \max_{1\le i\le n}\lambda_{\max}(\overbar{\MP}_i\circ\ML_i), \; \widetilde{\cL}_{\max} = \max_{1\le i\le n}\lambda_{\max}(\widetilde{\MP}_i\circ\ML_i)$ &  (\ref{def:tilde-cL-max}) \\
        \hline
        $\nu,\; \nu_s$ & Parameters describing distribution of $\ML_i$, \;
        $\nu \eqdef \frac{\sum_{i=1}^n L_i}{\max_{i\in[n]} L_i},
         \nu_s \eqdef \max_{i\in[n]}\frac{\sum_{j=1}^d \ML^{\nicefrac{1}{s}}_{i;j}}{\max_{j\in[d]}\ML^{\nicefrac{1}{s}}_{i;j}}$, &  (\ref{def:nu}) \\
        \hline
    \end{tabular}       
\end{table*}

\clearpage
\section{Theory in the Single Node Case: RCD as Sketched Gradient Descent (SkGD)}\label{rcd-as-skgd}



In single node setup, matrix smoothness assumption and arbitrary samplings have been considered mainly in the context of coordinate descent methods. For example, randomized sampling $S=\{j\},j\in[d]$ with arbitrary probabilities $p_j\in(0,1]$ reduces to standard {\em Randomized Coordinate Descent (RCD)} algorithms \citep{nesterov2012efficiency,RichTaka14}. Parallel and mini-batch variants arise when the sampling $S$ contains more than one coordinate \citep{L1-PCD,RichTaka16-2}.
The first coordinate descent method analyzed with arbitrary sampling and under $\ML$-smoothness assumption is the 'NSync algorithm \citep{Nsync,QuRich16-1,QuRich16-2} considered for strongly convex losses. In the same general setup, \citet{ACD-HanzRich} developed and analyzed {\em Accelerated Coordinate Descent}.
Recently, \citet{SEGA-FP} developed a variance reduced coordinate descent algorithm, {\em SEGA (SkEtched GrAdient)}, which uses general sketch matrices and handles non-separable proximal terms in contrast to traditional coordinate descent methods. This idea of gradient sketching then extended to {\em Generalized Jacobian Sketching (GJS)} algorithm providing a unified theory for first-order methods with variance reduced \citep{GJS-HR}. 

Consider the unconstrained optimization problem
\begin{equation}\label{main-opt-problem}
\min_{x\in\R^d} f(x),
\end{equation}
with very large dimension $d$ and assume that function $f$ is $\ML$-smooth. In this setting, the state-of-art methods are {\em Randomized Coordinate Descent (RCD)} type methods where in each iteration only a few coordinates get updated. Here we present new theories for RCD with arbitrary sampling paradigm, which are new and follow the idea of sketches. We will view RCD as a special case of {\em Compressed Gradient Descent (CGD)} with sketches (\ref{sketch-matrix-C}).

\subsection{`NSync}

First, we recall the first coordinate descent type algorithm, `NSync \citep{Nsync}, using arbitrary sampling. Let $S\subseteq[d]$ be an arbitrary (proper) sampling\footnote{only proper samplings are considered in this work} of coordinates such that $p_j \eqdef \Pr(j\in S)>0,\, j=1,2,\dots,d$. For a vector $h\in\R^d$, let $h_S\in\R^d$ be the vector coinciding with $h$ at coordinates $j\in S$ and zeros everywhere else. Denote by $\circ$ the Hadamard (i.e. element-wise) product. Given an arbitrary sampling $S$ and smoothness matrix $\ML$, let $v = (v_1,v_2,\dots,v_d)$ be positive constants satisfying the {\em Expected Separable Overapproximation (ESO)} inequality
\begin{equation}\label{ESO}
\MP \circ \ML \preceq \Mdiag(p\circ v),
\end{equation}
where $\MP$ is the probability matrix associated with sampling $S$ having entries $p_{jl} \eqdef \Pr(\{j,l\}\subseteq S),\; p_j = p_{jj}$. Analogous to (\ref{def:P}), let $\widetilde{\MP} = \overbar{\MP}-\ME$.

\begin{algorithm}[H]
\begin{algorithmic}[1]
\STATE \textbf{Input:} Initial point $x^0\in\R^d$, random sampling $S$, step size parameters $v$, current point $x^k$
\STATE \text{Sample random set of coordinates} $S_k\sim S$
\STATE \text{Update selected coordinates} $x^{k+1} = x^k - \frac{1}{v}\circ\nabla f(x)_{S_k}$
\end{algorithmic}
\caption{\sc `NSync \citep{Nsync}}
\label{alg:nsync}
\end{algorithm}

\begin{theorem}[`NSync, \citep{Nsync}]\label{thm:nsync}
Let Assumptions \ref{asm:Li-smooth-convex}, \ref{asm:mu-convex} hold and $v\sim\textrm{ESO}(f,S)$ be the vector of ESO parameters associated with function $f$ and sampling $S$. Then the iterates $\{x^k\}$ of `NSync converge as follows
$$
\E\[f(x^k)\]-f(x^*) \le \(1-\min_{1\le j\le d}\frac{p_j\mu}{v_j}\)^k \Delta_f,
$$
where $\Delta_f=f(x^0)-f(x^*)$.
\end{theorem}

Thus, `Nsync gives an iteration complexity
\begin{equation}\label{nsync-rate-step}
\max_{1\le j \le d}\frac{v_j}{p_j\mu}\log\frac{\Delta_f}{\varepsilon}.
\end{equation}

In case of serial sampling, namely $|S|=1$ a.s., we have $\MP=\Mdiag(p_1,p_2,\dots,p_d)$. Hence ESO holds with $v_j=\ML_{jj}$ and iteration complexity becomes $\max_j\frac{\ML_{jj}}{p_j\mu}\log\frac{\Delta_f}{\varepsilon}$. This leads to the optimal probabilities $p_j=\frac{\ML_{jj}}{\sum_l \ML_{ll}}$ yielding iteration complexity $\frac{\sum \ML_{jj}}{\mu}\log\frac{\Delta_f}{\varepsilon}$.

\subsection{Sketched Gradient Descent (SkGD)}

Let us view RCD methods as a special case of {\em Compressed Gradient Descent (CGD)} with linear and diagonal sketch $\MC$ defined in (\ref{sketch-matrix-C}) and consider random sparsification operator $\cC$ induced by random diagonal sketch $\MC$, namely $\mathcal{C}(x) = \MC x,\,x\in\R^d$. Clearly, $\cC$ is an unbiased compression (i.e. $\E\[\cC(x)\] = x$) with variance $\omega=\max_{1\le j\le d}\frac{1}{p_j}-1$:
\begin{equation}\label{variance_formula}
\E\left[ \|\MC x - x\|_2^2 \right] = x^{\top} \E\left[\MC^2-\MI\right]x \le \omega \|x\|_2^2.
\end{equation}

\begin{algorithm}[H]
\begin{algorithmic}[1]
\STATE \textbf{Input:} Initial point $x^0\in\R^d$, diagonal sketch $\MC$, step size $\gamma$, current point $x^k$
\STATE $x^{k+1} = x^k - \gamma\MC\nabla f(x^k)$
\end{algorithmic}
\caption{\sc SkGD}
\label{alg:cgd-1}
\end{algorithm}

\begin{theorem}[see \ref{apx-thm:rate-fval}]\label{thm:rate-fval}
Let Assumptions \ref{asm:Li-smooth-convex}, \ref{asm:mu-convex} hold and $S$ be any proper sampling with probability matrix $\MP$. Then, for the step-size $0<\gamma\le\lambda^{-1}_{\max}\(\overbar{\MP} \circ \ML\)$, the iterates $\{x^k\}$ of Algorithm \ref{alg:cgd-1} converge as follows
$$
\E\[f(x^k)\]-f(x^*) \le \(1-\gamma\mu\)^k \Delta_f.
$$
\end{theorem}

The following lemma shows that, both `NSync and SkGD provide the same theoretical guarantees.

\begin{lemma}\label{lem:nsync-skgd}
$$
\min\limits_{v\colon \MP\circ\ML \le \Mdiag(v\circ p)} \max\limits_{1\le j\le d} \frac{v_j}{p_j} = \lambda_{\max}\(\overbar{\MP}\circ\ML\).
$$
\end{lemma}
\begin{proof}
If parameters $v$ satisfy ESO inequality (\ref{ESO}), then parameters defined by $$v'_i:=p_i\max_j\frac{v_j}{p_j}\ge v_i, \quad 1\le i \le d$$ also satisfy ESO inequality and give the same iteration complexity as
\begin{equation*}
\lambda := \max_i \frac{v_i}{p_i} = \max_i \frac{v'_i}{p_i}.
\end{equation*}
In particular, this implies that instead of searching for $d$ parameters $v_1,\dots,v_d$ satisfying ESO inequality $\MP\circ\ML \le \Mdiag(v\circ p)$ it suffices to find one scalar $\lambda>0$ such that $\MP\circ\ML \le \Mdiag(\lambda p\circ p)$ and set $v_i=\lambda p_i$ for all $i\in[d]$. The optimal (smallest) value of the scaling factor is
$$
\lambda
= \lambda_{\max}\( \Mdiag(\nicefrac{1}{p})(\MP\circ\ML)\Mdiag(\nicefrac{1}{p}) \)
= \lambda_{\max}\( \( \Mdiag(\nicefrac{1}{p})\MP\Mdiag(\nicefrac{1}{p}) \) \circ\ML \)
= \lambda_{\max}\(\overbar{\MP}\circ\ML\).
$$
Notice that with the choice of $v=\lambda p$, iteration complexities as well as the update rules of both methods coincide.
\end{proof}

One difference between these two methods is that, the update direction $\frac{1}{v}\circ\nabla f(x)_S$ of `NSync is biased in general as opposed to unbiased direction $\frac{1}{p}\circ\nabla f(x)_S$ of SkGD.

Note that the rate and the analysis of Theorem \ref{thm:rate-fval} is with respect to functional values (i.e. $f(x^k)-f^*$). Natural question is to develop an analysis based on iterates of the algorithm (i.e. $\|x^k-x^*\|^2$). Below, we provide such analysis under slightly different conditions on $f$ and with weighted distances. Formally, let, instead of $\ML$-smoothness and $\mu$-convexity, assume 
\begin{equation}\label{eq:strange-mu-L}
\mu \|x-x^*\|^2_{\ML} + \|\nabla f(x)\|^2 \le 2\ve{\nabla f(x)}{(x-x^*)}_{\ML}.
\end{equation}
Notice that the following is true just by combining $\ML$-smoothness and $\mu$-convexity:
\begin{equation}\label{eq:mu-L}
\mu \|x-x^*\|^2 + \|\nabla f(x)\|^2_{\ML^\dagger} \le 2\ve{\nabla f(x)}{(x-x^*)}.
\end{equation}
However, in general, inequalities (\ref{eq:strange-mu-L}) and (\ref{eq:mu-L}) are not equivalent.

\begin{theorem}\label{thm:rate-wdist}
Let instead of $\ML$-smoothness and $\mu$-convexity assume (\ref{eq:strange-mu-L}) holds. Then, for the step-size $0<\gamma\le\lambda^{-1}_{\max}\(\overbar{\MP} \circ \ML\)$, the iterates $\{x^k\}$ of Algorithm \ref{alg:cgd-1} converge as follows
$$
\E\[\|x^k - x^*\|^2_{\ML}\] \le \(1-\gamma\mu\)^k \|x^0 - x^*\|^2_{\ML}.
$$
\end{theorem}
\begin{proof}
Consider the improvement of the algorithm in a single iteration $x^+ = x - \gamma\MC\nabla f(x)$.
\begin{eqnarray}
\E\[\|x^+ - x^*\|^2_{\ML}\]
&=& \E\[ \|x  - x^* - \gamma\MC\nabla f(x)\|^2_{\ML} \]  \notag \\
&=& \|x  - x^*\|^2_{\ML} - 2\gamma\ve{x-x^*}{\nabla f(x)}_{\ML} + \gamma^2\E\[ \|\MC\nabla f(x)\|^2_{\ML} \] \notag \\
&= & \|x  - x^*\|^2_{\ML} - 2\gamma\ve{x-x^*}{\nabla f(x)}_{\ML} + \gamma^2\|\nabla f(x)\|^2_{\E\[\MC\ML\MC\]} \notag \\
&\overset{(\ref{E[CLC]})}{=}&  \|x  - x^*\|^2_{\ML} - 2\gamma\ve{x-x^*}{\nabla f(x)}_{\ML} + \gamma^2 \|\nabla f(x)\|^2_{\overbar{\MP}\circ\ML}\notag  \\
&\le & \|x  - x^*\|^2_{\ML} - 2\gamma\ve{x-x^*}{\nabla f(x)}_{\ML} + \gamma^2\lambda_{\max}(\overbar{\MP}\circ\ML) \|\nabla f(x)\|^2 \notag \\
&\le& \|x  - x^*\|^2_{\ML} - 2\gamma\ve{x-x^*}{\nabla f(x)}_{\ML} + \gamma \|\nabla f(x)\|^2 \notag \\
&\overset{(\ref{eq:strange-mu-L})}{\le}& \(1-\gamma\mu\)\|x  - x^*\|^2_{\ML}.\notag 
\end{eqnarray}
\end{proof}

\subsection{CGD+}

Here we introduce a new variant of CGD with non-diagonal matrix $\overbar{\MC} \eqdef \ML^{\nicefrac{1}{2}}\MC\ML^{\dagger\nicefrac{1}{2}}$, which works with any proximable regularizer $R(x)$. In this case the method converges to the neighborhood of the solution. Recall that the proximal operator is defined as followsL:
\begin{equation}\label{weighted-prox}
\prox_{R}(x) = \argmin_{u\in\R^d}\(R(u) + \frac{1}{2}\|u-x\|^2\).
\end{equation}

Define expected smoothness constants
$$
\overbar{\cL} = \lambda_{\max}(\overbar{\MP}\circ\ML), \quad \widetilde{\cL} = \lambda_{\max}(\widetilde{\MP}\circ\ML).
$$

The following lemma reveals the relationship between these constants.

\begin{lemma}\label{lem:PoL}
Let $L = \lambda_{\max}(\ML)$. Then $L \le \overbar{\cL} \le L + \widetilde{\cL}$.
\end{lemma}
\begin{proof}
First, positive semi-definiteness of $\MP$ was proved in Theorem 3.1 \citep{QuRich}. As $\Mdiag(\nicefrac{1}{p})$ is positive definite, then $\overbar{\MP}$ is positive semi-definite too. Since Hadamard product $\circ$ preserves positive semi-definiteness, we have that $\overbar{\MP}\circ\ML\succeq0$. It follows from Lemma \ref{lem:transform} that
\begin{equation*}
\E\[\ML^{\half} \(\overbar{\MC} - \MI\)^{\top}\(\overbar{\MC} - \MI\) \ML^{\half}\]
=
\ML^{\half}\ML^{\phalf} (\widetilde{\MP} \circ \ML) \ML^{\phalf}\ML^{\half}.
\end{equation*}

Hence the left hand side as well as $\widetilde{\MP}\circ\ML$ are symmetric and positive semidefinite. In particular, $\overbar{\MP}\circ\ML\succeq\ML$. Hence $L = \lambda_{\max}(\ML) \le \lambda_{\max}(\overbar{\MP}\circ\ML) = \overbar{\cL}$. The upper bound follows from the convexity of $\lambda_{\max}$ as $\overbar{\cL} = \lambda_{\max}(\overbar{\MP}\circ\ML) = \lambda_{\max}(\ML + \widetilde{\MP}\circ\ML) \le \lambda_{\max}(\ML) + \lambda_{\max}(\widetilde{\MP}\circ\ML) = L + \widetilde{\cL}$.
\end{proof}

\begin{algorithm}[H]
\begin{algorithmic}[1]
\STATE \textbf{Input:} Initial point $x^0\in\R^d$, sketch matrix $\overbar{\MC}=\ML^{\nicefrac{1}{2}}\MC\ML^{\dagger\nicefrac{1}{2}}$, step size $\gamma$, current point $x^k$
\STATE $x^{k+1} = \prox_{\gamma R}\(x^k - \gamma\overbar{\MC}\nabla f(x^k)\)$
\end{algorithmic}
\caption{\sc CGD+}
\label{alg:prox-cgd-2}
\end{algorithm}

With the new sketch $\overbar{\MC}$ in Algorithm \ref{alg:prox-cgd-2} we able to perform the analysis with respect to iterates in standard norm, under strong convexity and $\ML$-smoothness, allowing any proximable regularizer.

\begin{theorem}[see \ref{apx-thm:rate-dist-opt2}]\label{thm:rate-dist-opt2}
Let Assumptions \ref{asm:Li-smooth-convex}, \ref{asm:mu-convex} hold and $S$ be a sampling with probability matrix $\MP$. Then, for the step-size $0<\gamma\le\nicefrac{1}{2\overbar{\cL}}$, the iterates $\{x^k\}$ of Algorithm \ref{alg:prox-cgd-2} converge as follows
$$
\E\[\|x^k - x^*\|^2\] \le \(1-\gamma\mu\)^k\|x^0  - x^*\|^2 + \frac{2\gamma\widetilde{\cL}}{\mu} \|\nabla f(x^*)\|^2_{\ML^{\dagger}}.
$$
\end{theorem}

\begin{table}[t]
\caption{Original and proposed new methods for both single node and distributed setups.}
\label{tbl:new-methods2}
\vskip 0.15in
\begin{center}
\begin{small}
\begin{sc}
\begin{tabular}{@{\hskip 0.01in}cccccc@{\hskip 0.01in}}
\toprule
\makecellnew{Original} & `NSync & CGD & DCGD   & DIANA & ADIANA \\
\midrule
\makecellnew{{\bf NEW}}
& \makecellnew{SkGD \\ (Alg.\ref{alg:cgd-1})}
& \makecellnew{CGD+ \\ (Alg.\ref{alg:prox-cgd-2})}
& \makecellnew{DCGD+ \\ (Alg.\ref{alg:dskgd})}
& \makecellnew{DIANA+ \\ (Alg.\ref{alg:DIANA+})}
& \makecellnew{ADIANA+ \\ (Alg.\ref{alg:aDIANA+})}
\\
\midrule
Proximal
& \xmark
& \cmark
& \cmark
& \cmark
& \cmark
\\
Distributed
& \xmark
& \xmark
& \cmark
& \cmark
& \cmark
\\
\makecellnew{Variance  Reduced}
& \xmark
& \xmark
& \xmark
& \cmark
& \cmark
\\
Accelerated
& \xmark
& \xmark
& \xmark
& \xmark
& \cmark
\\
\bottomrule
\end{tabular}
\end{sc}
\end{small}
\end{center}
\vskip -0.1in
\end{table}

{\footnotesize
    
    \begin{table*}[t]
        \centering
        \caption{Complexity  of new methods with hidden log factors and constants.}\label{table:notations}
        \renewcommand{\arraystretch}{1.7}
        \begin{tabular}{|c|l|}
            \hline
            Method & Iteration Complexity \\
            \hline
            SkGD (Algorithm \ref{alg:cgd-1}) & $\frac{\overbar{\cL}}{\mu}$  \\
            \hline
            CGD+ (Algorithm \ref{alg:prox-cgd-2}) & $\frac{\overbar{\cL}}{\mu} + \frac{\widetilde{\cL}}{\mu^2\varepsilon}$  \\
            \hline
            DCGD+ (Algorithm \ref{alg:dskgd}) & $\frac{L}{\mu} + \frac{\widetilde{\cL}_{\max}}{\mu n} + \frac{\widetilde{\cL}_{\max}}{\mu^2n\varepsilon}$  \\
            \hline
            DIANA+ (Algorithm \ref{alg:DIANA+}) & $\omega_{\max} + \frac{L}{\mu} + \frac{\widetilde{\cL}_{\max}}{\mu n}$  \\
            \hline
            ADIANA+ (Algorithm \ref{alg:aDIANA+}) &
            $
            \begin{cases}
            \omega_{\max} + \sqrt{\omega_{\max}\frac{\widetilde{\cL}_{\max}}{\mu n}} & \text{if}\quad nL \le \widetilde{\cL}_{\max}\\
            \omega_{\max} + \sqrt{\frac{L}{\mu}} + \sqrt{\omega_{\max}\sqrt{\frac{\widetilde{\cL}_{\max}}{\mu n}}\sqrt{\frac{L}{\mu}}} & \text{if}\quad  nL > \widetilde{\cL}_{\max}.
            \end{cases}
            $  \\
            \hline
        \end{tabular}       
    \end{table*}
}

\clearpage

\section{Lower Bounds for Sketches as Linear Compression Operators}\label{sec:lower-bounds}

Here we investigate general sketch matrices $\MS$ as a linear compression operators. The motivation of this is to understand the trade-off between communication and variance of linear compressors. The notation, used in this section only, slightly deviates from the paper but otherwise is consistent throughout the section.

Consider compression of vectors $x\in\R^d$ allowing approximation error in exchange for less bits of communication. Let compression operator $\cC\colon\R^d\to\R^d$ be composed of some linear encoder $E(x)=\MS x$ with $s\times d$ sketch matrix $\MS$ and an arbitrary decoder $D\colon\R^s\to\R^d$, so that $\cC(x) = D(\MS x)$. Throughout we consider the space $\R^d$ equipped with an inner product together with its induced norm given by some symmetric and positive definite matrix $\MB$ of size $d\times d$ as follows
$$
\<x, y\>_{\MB} = x^{\top}\MB y, \quad \|x\|_\MB = \sqrt{\<x, x\>_\MB}, \quad x,y\in\R^d.
$$
In general, we let matrix $\MS$, number of rows $s$ and decoder $D$ to be random, while the matrix $\MB$ will be fixed throughout the analysis. Since we consider only linear encoders, we may assume $\|x\|_\MB=1$.

\subsection{Fixed sketches}

We first analyze the case where the sketch matrix $\MS$ is fixed and hence the compression operator $\cC$ is deterministic. The analysis then we will lead us on a more usefull result for random sketches. The decoder $D$ receiving vector $y=\MS x$ should be able to reconstruct $\hat{x}=D(y)$ so to minimize the squared error
$$
\alpha(\MS) \eqdef \sup_{\|x\|_\MB=1} \|\cC(x)-x\|^2_\MB = \sup_{\|x\|_\MB=1} \|D(\MS x)-x\|^2_\MB \le 1.
$$

The following lemma shows the optimal strategy for the decoder and possible values for $\alpha(\MS)$.
\begin{lemma}\label{lem:fixed-sketch}
For a fixed sketch $\MS$ the optimal reconstruction from $y=\MS x$ is
\begin{equation}\label{D*-fixed}
D^*(y) = \MS^{\dagger_{\MB}} y \equiv \MB^{-1}\MS^{\top}\(\MS \MB^{-1}\MS^{\top}\)^{\dagger}y,
\end{equation}
where $\cdot^\dagger$ indicates the Moore–Penrose inverse of a matrix. Furthermore, if $\ker(\MS) = \{0\}$ then $\alpha(\MS)=0$ as in this case $D^*(\MS x)=x$ for any $x\in\R^d$. Otherwise, if $\ker(\MS)\ne\{0\}$, then $\alpha(\MS)=1$.
\end{lemma}
\begin{proof}
Let $\ker(\MS) = \{z\colon \MS z=0\}$ be the kernel of $\MS$ and $x^{\dagger_\MB}=\MS^{\dagger_\MB} y$ be the minimal $\MB$-norm solution to the system $\MS z=y$ so that the set of all solutions is $x^{\dagger_\MB}+\ker(\MS)$:
$$
x^{\dagger_\MB} = \argmin_{x\colon \MS x=y} \|x\|_\MB^2 = \MS^{\dagger_\MB} y = \MB^{-\nicefrac{1}{2}}\(\MS\MB^{-\nicefrac{1}{2}}\)^{\dagger}y,
$$
Denote by
$$
\hat{S}(x) \eqdef \(x^{\dagger_\MB}+\ker(\MS)\)\cap \{z\in\R^d \colon \|z\|_\MB=1\}
$$
the intersection of the affine set of solutions and the unit sphere. Notice that initial vector $x\in\hat{S}(x)$ as it has unit $\MB$-norm and satisfies $\MS x=y$. Now the cost of sending $\MS x$ instead of original $x$, is the uncertainty that the decoder has to deal with by estimating the original vector within the set $\hat{S}$ so to minimize $\alpha$. We first show that $x^S \eqdef 2x^{\dagger_\MB}-x\in\hat{S}(x)$, which is equivalent to
$$
x^{\dagger_\MB} -  x\in \ker(\MS) \quad\text{and}\quad \|2x^{\dagger_\MB} -  x\|^2_\MB = 1.
$$
The first claim follows from the fact that both $x$ and $x^{\dagger_\MB}$ are solutions to $\MS z=y$, namely $\MS x^{\dagger_\MB} = y = \MS x$. Expanding the square in the second claim we get $\<x^{\dagger_\MB},x^{\dagger_\MB}-x\>_\MB=0$ which holds as $x^{\dagger_\MB}$ is the minimal $\MB$-norm solution. Therefore the vector $y$ the decoder receives does not differentiate between $x$ and $x^S$.
This implies that for any choice of $\hat{x}$ of the decoder
\begin{equation*}
\max\(\|\hat{x}-x\|_\MB^2,\|\hat{x}-x^S\|_\MB^2\) \ge \tfrac{1}{4}\(\|\hat{x}-x\|_\MB+\|\hat{x}-x^S\|_\MB\)^2 \ge \tfrac{1}{4} \|x^S-x\|_\MB^2 = \|x^{\dagger_\MB}-x\|_\MB^2
\end{equation*}
squared-error is unavoidable for the couple $x,x^S$ and the optimal choice is $\hat{x}=x^{\dagger_\MB}$. Thus, the optimal decoding strategy to $y=\MS x$ is $D^*(y) = x^{\dagger_\MB}$ given in (\ref{D*-fixed}).
Now, if $\ker(\MS)\ne\{0\}$ then we could pick the initial vector $x$ from the kernel space, i.e. $x\in\ker(\MS)$ and $\|x\|_\MB=1$. Then we would have $x^{\dagger_\MB}=0$ and hence the minimal squared-error $\alpha(\MS)=1$. On the other hand, if $\ker(S)=\{0\}$, then $x^{\dagger_\MB}=x$ as the system $\MS z=y$ has unique solution.
\end{proof}

To conclude for fixed sketches, notice that, $x$ and $x^S$ are in symmetry in this analysis. Indeed, if the initial vector was $x^S$ as opposed to $x$, then $\MS x=\MS x^S$, hence $x^{S\dagger_\MB}=x^{\dagger_\MB}$ and $x^{SS}=x$. Therefore, the analysis of Lemma \ref{lem:fixed-sketch} leads to the following lower bound for any decoder $D$ and initial vector $x\in\R^d$
\begin{equation}\label{fixed-sketch-lb}
\max_{z = x,x^S}\|\cC(z)-z\|^2_\MB \ge \|x^{\dagger_\MB}-x\|^2_\MB = 1 - \|x^{\dagger_\MB}\|^2_\MB = 1 - \|\MZ x\|^2_\MB,
\end{equation}
where we used orthogonality $\<x^{\dagger_\MB},x^{\dagger_\MB}-x\>_\MB=0$ and defined the random matrix $\MZ=\MZ(\MS)$ via
$$
\MZ \eqdef \MS^{\dagger_\MB}\MS = \MB^{-\nicefrac{1}{2}}\(\MS\MB^{-\nicefrac{1}{2}}\)^{\dagger}\MS = \MB^{-1}\MS^{\top}\(\MS\MB^{-1}\MS^{\top}\)^{\dagger}\MS.
$$

\subsection{Random sketches}

Now we turn to the general case when sketch matrix $\MS$ is random and drawn from some distribution $\cD$, to which both encoder and decoder have access. The number of rows $s$ of $\MS$ can also be random. In this case, the decoder $D$ upon receiving random vector $y=\MS x$ should estimate possibly randomized $\hat{x}=D(y)$ so to minimize the expected square error
\begin{equation}\label{def:alpha-D}
\alpha(\cD) \eqdef \sup_{\|x\|_\MB=1} \E\[\|\cC(x)-x\|^2_\MB\] \le 1,
\end{equation}
where $\cC(x) = D(\MS x)$ is a random mapping with a source of randomness coming from the distribution $\cD$ and decoder $D$. Below we prove a lower bound for $\alpha(\cD)$.

\begin{theorem}
Let $\cD$ be some distribution over $s\times d$ matrices $\MS$ allowing variable number of rows $s\in[d]$. Then for any (possibly randomized) compression operator $\cC(x) = D(\MS x)$ with i.i.d. samples $\MS\sim\cD$ and $x\in\R^d$ the following lower bound holds
\begin{equation}\label{up-sketches-simple}
\alpha(\cD) + \E_{\cD}\[\nicefrac{r}{d}\] \ge 1,
\end{equation}
where $r = \rank(\MS)$ is the number of independent rows in $\MS$.
\end{theorem}
\begin{proof}
Based on the lower bound (\ref{fixed-sketch-lb}) obtained from the deterministic case, decoder cannot avoid the error $1 - \|\MZ x\|_\MB^2$ even in the case of knowing what sketch the encoder used. Therefore minimal expected error $1 - \E_{\MS\sim\cD}\|\MZ x\|_\MB^2$ is unavoidable for any initial $x$. This leads to the following bound
\begin{eqnarray*}
1-\alpha(\cD)
&\le & \inf_{\|x\|_\MB=1} \E_{\cD}\[\|\MZ x\|_\MB^2\] \\
&= &  \inf_{\|x\|_\MB=1} \E_{\cD}\[x^{\top} \MZ^{\top}\MB\MZ x\] \\
&\overset{z = \MB^{\nicefrac{1}{2}}x}{=}& \inf_{\|z\|=1} \E_{\cD}\[z^{\top} \MB^{-\nicefrac{1}{2}} \MZ^{\top}\MB\MZ \MB^{-\nicefrac{1}{2}} z\] \\
&=& \inf_{\|z\|=1} z^{\top}\E_{\cD}\[ \MB^{-\nicefrac{1}{2}} \MZ^{\top}\MB\MZ \MB^{-\nicefrac{1}{2}} \]z \\
&=& \lambda_{\min}\(\E_{\cD}\[ \MB^{-\nicefrac{1}{2}} \MZ^{\top}\MB\MZ \MB^{-\nicefrac{1}{2}} \]\) \\
&=& \lambda_{\min}\(\E_{\cD}\[ \MB^{-1} \MZ^{\top}\MB\MZ \]\) \\
&= & \lambda_{\min}\(\E_{\cD}\[ \MB^{-1}\MS^{\top}\(\MS\MB^{-1}\MS^{\top}\)^{\dagger}\MS \]\) \\
&=& \lambda_{\min}\(\E_{\cD}\[ \MZ \]\),
\end{eqnarray*}
where the expectation is with respect to $\MS\sim\cD$. Thus, we obtained the following lower bound:
\begin{equation}\label{up-sketches}
\alpha(\cD) + \lambda_{\min}\(\E_{\cD}\[\MS^{\dagger_\MB}\MS\]\) \ge 1.
\end{equation}
To prove the inequality (\ref{up-sketches-simple}), it is enough to establish the following upper bound for the minimal eigenvalue
$$
\lambda_{\min}\(\E_{\cD}\[\MZ\]\) \le \E_{\cD}\[\nicefrac{r}{d}\].
$$

We follow the proof of Lemma 4.2 of \citet{GowRic} to prove this inequality. It can be easily checked that, using the properties of pseudo-inverse, $\MZ=\MS^{\dagger_\MB}\MS$ is an idempotent matrix
for any $\MS$, namely $\MZ^2=\MZ$. This implies that all eigenvalues of $\MZ$ are either $0$ or $1$ as they must satisfy the same relation $\lambda^2=\lambda$. Trace $\tr(\MZ)$ of such matrices coincides with the number of non-zero eigenvalues, which also shows the rank:
\begin{equation}\label{trace=rank}
\tr(\MZ) = \sum_{i=1}^d \lambda_i\(\MZ\) = \#\{i\in[d]\colon \lambda_i\(\MZ\)\ne 0\} = \rank(\MZ).
\end{equation}
From the properties of pseudo-inverse it follows that $\rank(\MA^\dagger \MA)=\rank(\MA^\dagger)=\rank(\MA)$ for any matrix $\MA$. Hence
\begin{align*}
\rank(\MZ) = \rank(\MS^{\dagger_\MB}\MS) &= \rank\(\MB^{-\nicefrac{1}{2}}\(\MS\MB^{-\nicefrac{1}{2}}\)^{\dagger}\MS\) \\
&= \rank\(\(\MS\MB^{-\nicefrac{1}{2}}\)^{\dagger}\MS\MB^{-\nicefrac{1}{2}}\) = \rank\(\MS\MB^{-\nicefrac{1}{2}}\) = \rank\(\MS\) = r.
\end{align*}
Combining with (\ref{trace=rank}) we get $\tr(\MZ) = r$. The purpose of expressing the rank as a trace is that in contrast to rank, trace and expectation operators are commutative, which basically follows from the linearity of the expectation:
\begin{equation}\label{trace-expectation}
\tr\(\E_{\cD}[\MZ]\) = \E_{\cD}\[\tr(\MZ)\].
\end{equation}
Using (\ref{trace=rank}), (\ref{trace-expectation}) and $\tr(\MZ)=r$, we conclude
\begin{equation*}
\lambda_{\min}\(\E_{\cD}\[\MZ\]\) \le \frac{1}{d}\sum_{i=1}^d \lambda_i\(\E_{\cD}\[\MZ\]\) = \frac{\tr\(\E_{\cD}\[\MZ\]\)}{d} = \frac{\E_{\cD}\[\tr\(\MZ\)\]}{d} = \frac{\E_{\cD}[r]}{d},
\end{equation*}
which completes the proof.
\end{proof}


\subsection{Optimal sketches}

With the knowledge of this new lower bound, here we construct a distribution $\cD$ of sketches that will achieve equality in (\ref{up-sketches-simple}). Let $\MB=\MQ\bm{\Lambda}\MQ^{\top}$ be the eigendecomposition of the symmetric matrix $\MB$, where $\bm{\Lambda}$ is diagonal with eigenvalues and $\MQ$ is orthogonal with eigenvectors as columns. Let $\MC$ be the diagonal sketch of size $d\times d$ corresponding to random sparsification with probabilities $p=(p_i)_{i=1}^d$, namely
\begin{equation*}
\MC = \diag(c), \quad 
c_i =
\begin{cases}
1 & \text{with prob.}\quad p_i,\\
0 & \text{with prob.}\quad 1-p_i.
\end{cases}
\end{equation*}
Define a distribution $\cD=\cD_p$ of sketches as $\MS = \MC\MQ^{\top}$ and notice that
$$
\E_{\cD}\[\rank(\MS)\] = \E_{\cD}\[\rank(\MC)\] = \E_{\cD}\[\#\{i\in[d] \colon c_i=1\}\] = \E_{\cD}\[\sum_{i=1}^d c_i\] = \sum_{i=1}^d \E_{\cD}\[c_i\] = \sum_{i=1}^d p_i.
$$
Therefore, $\E_{\cD}\[\nicefrac{r}{d}\] = \frac{1}{d}\sum p_i$. With decoder $D(x)=\MQ x$ we get a compression operator $\cC(x) = \MQ\MS x$. Next,  we compute $\alpha(\cD)$ as follows
\begin{eqnarray*}
\alpha(\cD)
&=& \sup_{\|x\|_\MB=1} \E\[\|\cC(x)-x\|^2_\MB\] \\
&=& \sup_{\|x\|_\MB=1} \E\[\|\MQ\MS x-x\|^2_\MB\] \\
&= &\sup_{x^{\top} \MB x = 1} \E\[x^{\top}(\MI-\MQ\MS)^{\top}\MB(\MI-\MQ\MS)x\] \\
&=& \sup_{x^{\top} \MQ\MC\MQ^{\top} x = 1} x^{\top}\E\[(\MI-\MQ\MC\MQ^{\top})\MB(\MI-\MQ\MC\MQ^{\top})\]x \\
&=& \sup_{(\MQ^{\top}x)^{\top}\bm{\Lambda}(\MQ^{\top}x)} (\MQ^{\top}x)^{\top}\E\[(\MI-\MC)\MQ^{\top}\MB\MQ(\MI-\MC)\](\MQ^{\top}x) \\
&\overset{y=\MQ^{\top}x}{=}  & \sup_{y^{\top}\bm{\Lambda} y=1} y^{\top}\E\[(\MI-\MC)\bm{\Lambda}(\MI-\MC)\]y\\
&=& \sup_{y^{\top}\bm{\Lambda} y=1} (\bm{\Lambda}^{\nicefrac{1}{2}} y)^{\top}\E\[(\MI-\MC)^2\](\bm{\Lambda}^{\nicefrac{1}{2}} y) \\
&\overset{z=\bm{\Lambda}^{\nicefrac{1}{2}} y}{=} & \sup_{\|z\|=1} z^{\top} \cdot \Mdiag(1-p) \cdot z \\
&= & \max_{1\le i\le d} (1-p_i) = 1 - \min_{1\le i\le d} p_i.
\end{eqnarray*}
Hence $$1 \le \alpha(\cD) + \E_{\cD}\[\nicefrac{r}{d}\] = 1 - \min_{1\le i\le d} p_i + \frac{1}{d}\sum_{i=1}^d p_i,$$ and equality occurs if and only if all probabilities $p_i$ are equal to some $q\in[0,1]$. Thus, the optimal sketches are obtained by rotating the coordinate basis to the basis of eigenvectors of $\MQ$ (i.e. $x\to \MQ^{\top}x$), and then randomly sparsify coordinates with diagonal sketch matrix $\MC$ (i.e. $\MQ^{\top}x\to \MC\MQ^{\top}x=\MS x$). We summarize this result in the following theorem.

\begin{theorem}\label{thm:opt-sketches}
Let $\MB=\MQ\bm{\Lambda} \MQ^{\top}$ be the eigendecomposition of $\MB$ of induced norm, $q\in[0,1]$ and $\MC$ be random diagonal sketch corresponding to the random $q$-sparsifer. Then sketches $\MS=\MC\MQ^{\top}$ are optimal with respect to variance against rank trade-off (\ref{up-sketches-simple}) with squared error $\alpha=1-q$ and expected rank $\E[r]=q d$.
\end{theorem}


\subsection{Random sketches with linear constraints}

In this part we extend the theory of compressing vectors $x\in\R^d$ with an additional linear constraint $x\in\range(\MA)$ for some $d\times d'$ matrix $\MA$. Such scenarios occur when to-be-compressed vectors are the gradients of $f(w)=\phi(\MA^{\top}w)$, for which $\nabla f(w) = \MA\nabla\phi(\MA^{\top}w)\in\range(\MA)$. Without loss of generality, we may assume that $\MA$ is of full column rank and consequently $d'=\dim\range(\MA)=\rank(\MA)$. The constraint $x\in\range(\MA)$ then can be equivalently written as $x=\MA x'$ for some $x'\in\R^{d'}$. The induced inner product and norm on $\range(\MA)$ is then given by the matrix $\MA^{\top}\MB\MA$ as
$$
\<x, y\>_\MB = \<\MA x', \MA y'\>_\MB = \<x', y'\>_{\MA^{\top}\MB\MA}, \quad x=\MA x',\,y=\MA y'.
$$
Notice that, since $\MS x = \MS\MA x'$, communication of $x\in\R^d$ with sketches $\MS$ reduces to communication of $x'\in\R^{d'}$ with sketches $\MS\MA$. Thus, the additional constraint $x\in\range(\MA)\subset\R^d$ reduces the problem to lower $d'$-dimension with sketches $\MS\MA,\MS\sim\cD$ and norm induced by $\MA^{\top}\MB\MA$.

\subsection{Variance against communication trade-off}

The obtained lower bound (\ref{up-sketches-simple}) can be easily translated in terms of the number of bits. Assuming each float takes $32$ bits to encode and there is no redundant row in $\MS$ (i.e. $s=r$), then $\MS x\in\R^r$ can be communicated with up to $b = 32r$ bits.
Therefore, the lower bound (\ref{up-sketches-simple}) can be written as
\begin{equation}\label{var-bits-lb}
\alpha + \frac{\E\[b\]}{32d} \ge 1,
\end{equation}
which (ignoring the expectation) is exponentially stronger than the lower bound $\alpha\cdot 4^{\nicefrac{b}{d}}\ge 1$ obtained for general compressors in \citep{up_kashin_2020}. We visualize the comparison of these two lower bounds in Figure \ref{fig:ups}. Furthermore, denote by $\beta\eqdef \E\[b\]/32d$ the expected communication reduction factor and recall that $\alpha$ is the portion of the expected lost of information. With this notation the above lower bound (\ref{var-bits-lb}) turns to the following simple inequality
$$
\alpha + \beta \ge 1,
$$
showing the trade-off between information lost and communication reduction for linear compressors; namely more reduction in communication leads to bigger information loss and vice versa. In one extreme, when all $32d$ bits are sent, no reduction in communication is made ($\beta=1$) and no information is lost ($\alpha=0$). In other extreme, when no bits gets transferred ($\beta=0$) we loose all information ($\alpha=1$).

To conclude this section, let us investigate the optimality of random $q$-sparsifier with respect to the lower bound (\ref{var-bits-lb}). Recall that random $q$-sparsifier is optimal with respect to (\ref{up-sketches-simple}). Let $q\in(0,1)$, and $k$ be the (random) number of non-zero entries of sparsified vector. Clearly, $\E\[k\]=qd$ and to encode any $k$-sparse vector one needs $b=32k + \log_2\binom{d}{k}$ bits. As we know from Theorem \ref{thm:opt-sketches}, the squared error $\alpha=1-q$. Therefore
$$
\alpha + \beta = 1 - q +  \frac{1}{32d}\E\[32k + \log_2\binom{d}{k}\] = 1 + \frac{1}{32d}\E\[\log_2\binom{d}{k}\] \le 1 + \frac{1}{32}\E\[H_2\(\frac{k}{d}\)\] \le 1 + \frac{H_2(q)}{32}.
$$
The first inequality follows from the following estimate (only upper bound) for binomial coefficients
$$
\frac{2^{d H_2(\tau)}}{\sqrt{8d\tau(1-\tau)}} \le \binom{d}{\tau d} \le \frac{2^{d H_2(\tau)}}{\sqrt{2\pi d\tau(1-\tau)}}, \quad 0<\tau<1,
$$
where $H_2(\tau) = -\tau\log_2\tau - (1-\tau)\log_2(1-\tau)$ is the binary entropy function in bits. The second inequality follows from concavity $H_2$ function and the Jensen's inequality. Because of the symmetry around $\tau=\nicefrac{1}{2}$ (namely $H_2(1-\tau)=H_2(\tau)$) and concavity of the function $H_2$, one can show that the maximum is achieved at $\tau=\nicefrac{1}{2}$ and $H_2(\nicefrac{1}{2})=1$. Thus, in the worst case we have $\alpha+\beta\le\nicefrac{33}{32}$ upper bound, when roughly half of the entries are chosen uniformly at random. For other values of $q$, it is even closer to the optimum; numerically $H_2\(\tau\) \approx \(4\tau\(1-\tau\)\)^{\nicefrac{3}{4}}, \; 0\le\tau\le 1$.

\begin{figure}[ht!]
\begin{center}
    \includegraphics[scale=0.6]{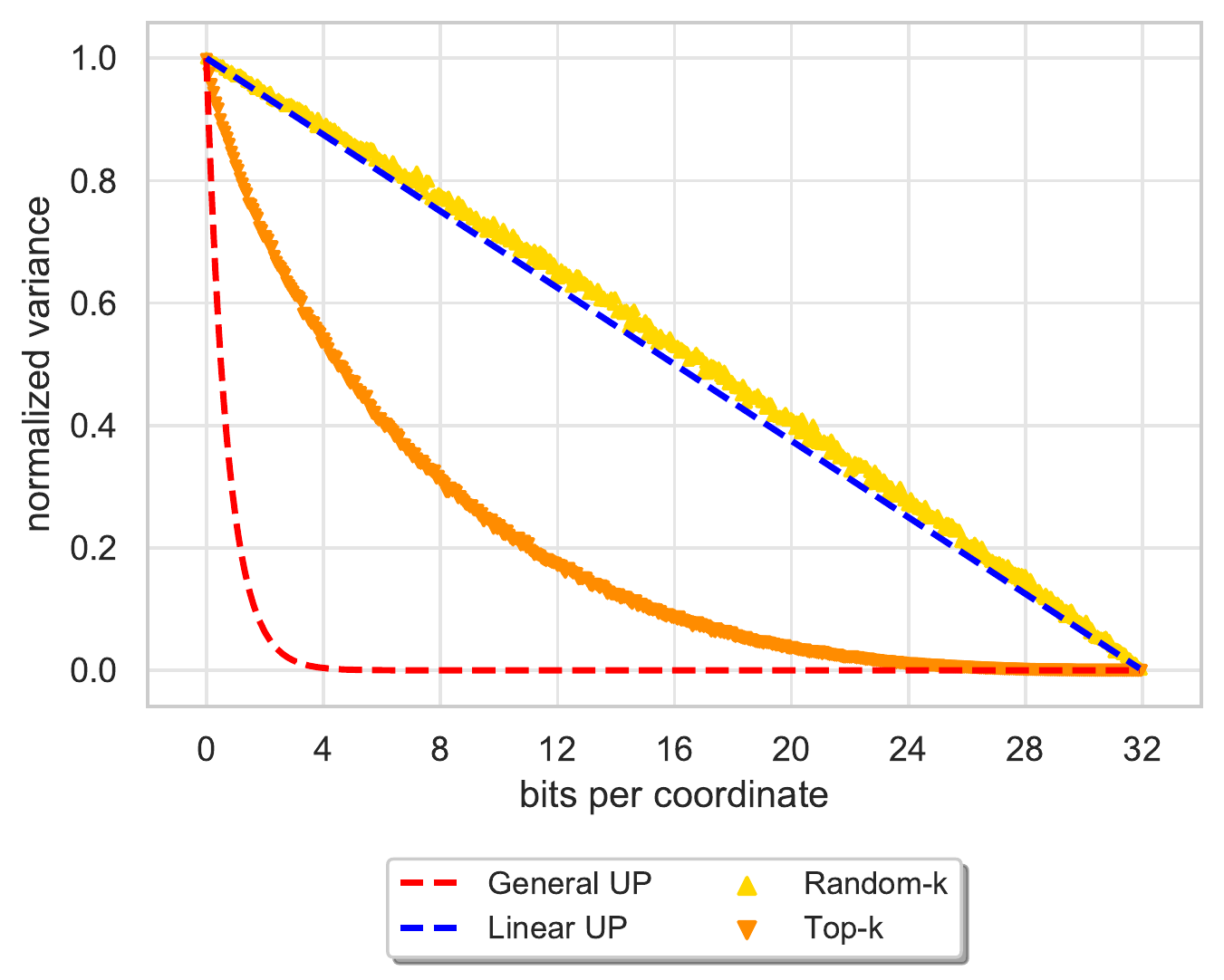}
\end{center}
\vspace{-5pt}
\caption{Comparison of general uncertainty principle $\alpha\cdot 4^{\nicefrac{b}{d}}\ge 1$ (dashed red line) of \citet{up_kashin_2020} against the new linear version (\ref{var-bits-lb}) (dashed blue line). Each color represents one compression method: yellow for usual random sparsification with uniform probabilities and orange for greedy sparsification (a.k.a Top-$k$ sparsification). Each triangle marker indicates one particular $d = 10^3$ dimensional vector randomly generated from Gaussian distribution, which subsequently gets compressed by the compression operator mentioned in the legend.}
\label{fig:ups}
\end{figure}

\clearpage

\section{Proofs}

\subsection{Proof of Theorem \ref{thm:rate-fval}}\label{apx-thm:rate-fval}

Using smoothness of $f$, we have
\begin{align}
\begin{split}\label{f-descent-muL}
\E f(x^{k+1})
&=  \E f(x^k - \gamma\MC\nabla f(x^k)) \\
&\le f(x^k) - \gamma\<\nabla f(x^k), \E\[\MC\nabla f(x^k)\]\> + \frac{\gamma^2}{2}\E\[\|\MC\nabla f(x^k)\|_{\ML}^2\] \\
&=   f(x^k) - \gamma\|\nabla f(x^k)\|^2 + \frac{\gamma^2}{2}\|\nabla f(x^k)\|_{\E\[\MC\ML\MC\]}^2 \\
&\le f(x^k) - \gamma\(2 - \gamma\lambda_{\max}\(\E\[\MC\ML\MC\]\) \) \cdot \frac{1}{2}\|\nabla f(x^k)\|^2.
\end{split}
\end{align}

Computing the expectation inside, we get
\begin{equation}\label{E[CLC]}
\E\[\MC\ML\MC\] = \E\[\(\mathrm{c}_i\mathrm{c}_j\ML_{ij}\)_{i,j=1}^d\] = \(\frac{p_{ij}\ML_{ij}}{p_ip_j}\)_{i,j=1}^d = \( \Mdiag(\nicefrac{1}{p}) \MP \Mdiag(\nicefrac{1}{p}) \) \circ \ML = \overbar{\MP} \circ \ML.
\end{equation}

Therefore, using the bound for the step size $\gamma$ and strong convexity of $f$, we get
\begin{align}
\begin{split}
\E\[f(x^{k+1}) - f(x^*)\]
&\le \(f(x^k) - f(x^*)\) - \gamma\(2 - \gamma\lambda_{\max}\(\overbar{\MP} \circ \ML\) \) \cdot \frac{1}{2}\|\nabla f(x^k)\|^2 \\
&\le \(f(x^k) - f(x^*)\) - \frac{\gamma}{2}\|\nabla f(x^k)\|^2 \\
&\le \(1-\gamma\mu\)\(f(x^k)-f(x^*)\),
\end{split}
\end{align}
repeated application of which completes the proof.

\subsection{Proof of Theorem \ref{thm:rate-dist-opt2}}\label{apx-thm:rate-dist-opt2}

The following lemmas will be useful to handle the computation with pseudo-inverses.
\begin{lemma}[Lemma E.2 and E.3 \citep{GJS-HR}]\label{lem:grad-space}
If $f$ is convex and $\ML$-smooth, then for any $x,y\in\R^d$
\begin{equation}\label{bregman-smooth}
f(y) \ge f(x) + \<\nabla f(x), y-x\> + \frac{1}{2}\|\nabla f(x) - \nabla f(y)\|^2_{\ML^\dagger}.
\end{equation}
If, in addition, $f$ is bounded below, then $\nabla f(x)\in\range(\ML^\dagger) = \range(\ML)$ for all $x\in\R^d$.
\end{lemma}

\begin{lemma}\label{lem:transform}
With $\overbar{\MC} = \ML^{\nicefrac{1}{2}}\MC\ML^{\phalf}$, the following holds
\begin{equation}\label{eq:transform}
\E\[\ML^{\nicefrac{1}{2}} \(\overbar{\MC} - \MI\)^{\top}\(\overbar{\MC} - \MI\) \ML^{\nicefrac{1}{2}}\]
=
\ML^{\nicefrac{1}{2}}\ML^{\phalf}\(\widetilde{\MP} \circ \ML\)\ML^{\phalf}\ML^{\nicefrac{1}{2}}.
\end{equation}
\end{lemma}
\begin{proof}
Using the property $\ML^{\half}\ML^{\phalf}\ML^{\half} = \ML^{\half}$ of pseudoinverse, we have
\begin{eqnarray*}
\E\[\ML^{\half} \(\overbar{\MC} - \MI\)^{\top}\(\overbar{\MC} - \MI\) \ML^{\half}\]
&= &
    \E\[ \ML^{\nicefrac{1}{2}} \(\ML^{\phalf}\MC\ML^{\nicefrac{1}{2}} - \MI\)\(\ML^{\nicefrac{1}{2}}\MC\ML^{\phalf} - \MI\) \ML^{\half} \] \\
&= &
    \E\[ \ML^{\nicefrac{1}{2}} \(\ML^{\phalf}\MC\ML\MC\ML^{\phalf}
    - \ML^{\phalf}\MC\ML^{\nicefrac{1}{2}}
    - \ML^{\nicefrac{1}{2}}\MC\ML^{\phalf}
    + \MI\) \ML^{\nicefrac{1}{2}} \] \\
&
\overset{(\ref{E[CLC]})}{=} &
    \ML^{\nicefrac{1}{2}} \(\ML^{\phalf}\(\overbar{\MP} \circ \ML\)\ML^{\phalf}
    - \ML^{\phalf}\ML^{\half}
    - \ML^{\nicefrac{1}{2}}\ML^{\phalf}
    + \MI\) \ML^{\nicefrac{1}{2}} \\
&= &
    \ML^{\nicefrac{1}{2}} \(\ML^{\phalf}\(\overbar{\MP} \circ \ML\)\ML^{\phalf}
    - \ML^{\phalf}\ML\ML^{\phalf}\) \ML^{\nicefrac{1}{2}} \\
    &&\quad
    +
    \ML^{\nicefrac{1}{2}} \(\ML^{\phalf}\ML\ML^{\phalf}
    - \ML^{\dagger^{\nicefrac{1}{2}}}\ML^{\nicefrac{1}{2}}
    - \ML^{\nicefrac{1}{2}}\ML^{\dagger^{\nicefrac{1}{2}}}
    + \MI\) \ML^{\nicefrac{1}{2}} \\
&= &
    \ML^{\nicefrac{1}{2}} \ML^{\phalf}\(\overbar{\MP} \circ \ML - \ML\)\ML^{\phalf}
    \ML^{\nicefrac{1}{2}}
    +
    \ML^{\nicefrac{1}{2}} \( \MI - \ML^{\phalf}\ML^{\nicefrac{1}{2}} \)
    \( \MI - \ML^{\nicefrac{1}{2}}\ML^{\phalf}\)
    \ML^{\nicefrac{1}{2}} \\
&= &
    \ML^{\nicefrac{1}{2}} \ML^{\phalf}\(\widetilde{\MP} \circ \ML\)\ML^{\phalf}
    \ML^{\nicefrac{1}{2}}
    +
    \( \ML^{\nicefrac{1}{2}} - \ML^{\nicefrac{1}{2}}\ML^{\phalf}\ML^{\nicefrac{1}{2}} \)
    \( \ML^{\nicefrac{1}{2}} - \ML^{\nicefrac{1}{2}}\ML^{\phalf}\ML^{\nicefrac{1}{2}}\) \\
&=&  \ML^{\nicefrac{1}{2}} \ML^{\phalf}\(\widetilde{\MP} \circ \ML\)\ML^{\phalf}
    \ML^{\nicefrac{1}{2}}.
\end{eqnarray*}
\end{proof}

For convenience we skip iteration count $k$, and write $x,x^+$ instead of $x^k,x^{k+1}$. Using non-expansiveness of the $\prox$ operator we get
\begin{align*}
\begin{split}
\E&\[\|x^+ - x^*\|^2\]
\le \E\[ \|x  - x^* - \gamma\(\ML^\half\MC\ML^\phalf\)\nabla f(x) + \gamma\nabla f(x^*)\|^2 \] \\
&=
     \|x  - x^*\|^2
     - 2\gamma\ve{x-x^*}{\nabla f(x) - \nabla f(x^*)}
     + \gamma^2\E\[ \|\(\ML^\half\MC\ML^\phalf\)\nabla f(x) - \nabla f(x^*)\|^2 \] \\
&\le
     \|x  - x^*\|^2
     - 2\gamma\ve{x-x^*}{\nabla f(x) - \nabla f(x^*)} \\
     &\phantom{=}
     + 2\gamma^2\E\[ \|\ML^\half\MC\ML^\phalf\(\nabla f(x) - \nabla f(x^*)\)\|^2 \]
     + 2\gamma^2\E\[ \|\(\ML^\half\MC\ML^\phalf - \MI\)\nabla f(x^*)\|^2 \] \\
&\le
     \|x  - x^*\|^2
     - 2\gamma\ve{x-x^*}{\nabla f(x)
     - \nabla f(x^*)} \\
     &\phantom{=}
     + 2\gamma^2\lambda_{\max}(\E\[\MC\ML\MC\])\|\ML^\phalf\(\nabla f(x) - \nabla f(x^*)\)\|^2
     + 2\gamma^2\E\[ \|\(\ML^{\nicefrac{1}{2}}\MC\ML^\phalf - \MI\)\nabla f(x^*)\|^2 \] \\
&\overset{(\ref{E[CLC]}), (\ref{E[CLC-L]})}{\le}
     \|x  - x^*\|^2
     - 2\gamma\ve{x-x^*}{\nabla f(x) - \nabla f(x^*)} \\
     &\phantom{=}
     + 2\gamma^2\lambda_{\max}(\overbar{\MP}\circ\ML) \|\nabla f(x) - \nabla f(x^*)\|^2_{\ML^{\dagger}}
     + 2\gamma^2\lambda_{\max}(\widetilde{\MP}\circ\ML) \|\nabla f(x^*)\|^2_{\ML^{\dagger}} \\
&=
     \|x  - x^*\|^2
     - 2\gamma\ve{x-x^*}{\nabla f(x) - \nabla f(x^*)}
     + 2\gamma^2\overbar{\cL} \|\nabla f(x) - \nabla f(x^*)\|^2_{\ML^{\dagger}}
     + 2\gamma^2\widetilde{\cL} \|\nabla f(x^*)\|^2_{\ML^{\dagger}},
\end{split}
\end{align*}
where we used $\E\[\MC\ML\MC\] = \overbar{\MP} \circ \ML$ based on (\ref{E[CLC]}) and for the last term we used Lemma \ref{lem:grad-space} to represent $\nabla f(x^*) = \ML^{\nicefrac{1}{2}}g_*$ and then applied Lemma \ref{lem:transform}
\begin{align}\label{E[CLC-L]}
\begin{split}
\E\[ \left\|\(\ML^{\nicefrac{1}{2}}\MC\ML^{\phalf} - \MI\)\nabla f(x^*)\right\|^2 \]
&= 
    \E\[ g_*^{\top} \ML^{\nicefrac{1}{2}} \(\ML^{\phalf}\MC\ML^{\nicefrac{1}{2}} - \MI\)\(\ML^{\nicefrac{1}{2}}\MC\ML^{\phalf} - \MI\) \ML^{\nicefrac{1}{2}} g_* \] \\
&=  \nabla f(x^*)^{\top} \(\ML^{\phalf}\(\widetilde{\MP} \circ \ML\)\ML^{\phalf}
    \) \nabla f(x^*) \\
&\le
    \lambda_{\max}(\widetilde{\MP} \circ \ML)\|\nabla f(x^*)\|^2_{\ML^\dagger}.
\end{split}
\end{align}

Using the bound on step size $\gamma\le \nicefrac{1}{2\widetilde{\cL}}$, strong convexity of $f$ and (\ref{bregman-smooth}), we continue as follows
\begin{eqnarray*}
\E\[\|x^+ - x^*\|^2\]
&\le &
     \|x  - x^*\|^2 - \gamma\ve{x-x^*}{\nabla f(x) - \nabla f(x^*)} \\
     && \quad - \gamma \(\ve{x-x^*}{\nabla f(x) - \nabla f(x^*)} - \|\nabla f(x) - \nabla f(x^*)\|^2_{\ML^{\dagger}} \) \\
     && \quad + 2\gamma^2\widetilde{\cL} \|\nabla f(x^*)\|^2_{\ML^{\dagger}} \\
&\overset{(\ref{bregman-smooth})}{\le} &
     \(1-\gamma\mu\)\|x  - x^*\|^2 + 2\gamma^2\widetilde{\cL} \|\nabla f(x^*)\|^2_{\ML^{\dagger}}.
\end{eqnarray*}

Telescoping the above inequality, we complete the proof.

\subsection{Proof of Theorem \ref{thm:dist-prox-skgd-better}}\label{apx-thm:dist-prox-skgd-better}

In this proof we skip the iteration count $k$ to simplify the notation. Define
\begin{eqnarray}
\MM_i
&\eqdef &
    \ML_i^{\nicefrac{1}{2}} \E\[(\overbar{\MC}_i-\MI)^{\top}(\overbar{\MC}_i-\MI)\] \ML_i^{\nicefrac{1}{2}} \notag \\    
    & \overset{(\ref{eq:transform})}{=}&
    \ML_i^{\nicefrac{1}{2}}\ML_i^{\dagger\nicefrac{1}{2}} (\widetilde{\MP}_i\circ\ML_i) \ML_i^{\dagger\nicefrac{1}{2}}\ML_i^{\nicefrac{1}{2}}\notag \\
&=&
    \ML_i^{\nicefrac{1}{2}}\ML_i^{\dagger\nicefrac{1}{2}} (\overbar{\MP}_i\circ\ML_i - \ML_i) \ML_i^{\dagger\nicefrac{1}{2}}\ML_i^{\nicefrac{1}{2}} \notag\\
&=&
    \ML_i^{\nicefrac{1}{2}}\ML_i^{\dagger\nicefrac{1}{2}} (\overbar{\MP}_i\circ\ML_i) \ML_i^{\dagger\nicefrac{1}{2}}\ML_i^{\nicefrac{1}{2}}
    - \ML_i^{\nicefrac{1}{2}}\ML_i^{\dagger\nicefrac{1}{2}} \ML_i \ML_i^{\dagger\nicefrac{1}{2}}\ML_i^{\nicefrac{1}{2}} \notag\\
&=&  \ML_i^{\nicefrac{1}{2}}\ML_i^{\dagger\nicefrac{1}{2}} (\overbar{\MP}_i\circ\ML_i) \ML_i^{\dagger\nicefrac{1}{2}}\ML_i^{\nicefrac{1}{2}} - \ML_i\notag \\
&=&
    \ML_i^{\nicefrac{1}{2}} \( \E\[\overbar{\MC}_i^{\top}\overbar{\MC}_i\] - \MI\) \ML_i^{\nicefrac{1}{2}}.\label{def:M-1}
\end{eqnarray}

We are going to estimate the moment $\E\[\|g(x)-\nabla f(x^*)\|^2\]$ and show the following bound for the gradient estimator $g(x) = \frac{1}{n}\sum_{i=1}^n \overbar{\MC}_i \nabla f_i(x)$ (see line 5 of Algorithm \ref{alg:dskgd}):
\begin{equation*}
\E\[\|g(x) - \nabla f(x^*)\|^2\] \le 2\(L+\frac{2\widetilde{\cL}}{n}\) D_{f}(x,x^*) + \frac{2\sigma^*}{n}.
\end{equation*}

Due to Lemma \ref{lem:grad-space}, we have $\nabla f_i(x) = \ML_i^{\nicefrac{1}{2}}r_i$ for some $r_i$. Therefore
\begin{equation}\label{unbiased-est}
\E\[ \overbar{\MC}_i \nabla f_i(x) \]
= \E\[ \ML_i^{\nicefrac{1}{2}}\MC_i\ML_i^{\dagger\nicefrac{1}{2}} \ML_i^{\nicefrac{1}{2}}r_i \]
= \ML_i^{\nicefrac{1}{2}}\ML_i^{\dagger\nicefrac{1}{2}} \ML_i^{\nicefrac{1}{2}}r_i
= \ML_i^{\nicefrac{1}{2}}r_i
= \nabla f_i(x),
\end{equation}
which implies unbiasedness of the estimator $g(x)$, namely $\E\[g(x)\] = \nabla f(x)$. Next, note that
\begin{align*}\label{moment-decomposition-proof}
\begin{split}
\E&\[\|g(x)-\nabla f(x^*)\|^2\]
=
   \E\[ \left\|\frac{1}{n}\sum_{i=1}^n \overbar{\MC}_i \nabla f_i(x) - \nabla f(x^*) \right\|^2 \] \\
&=
   \frac{1}{n^2}\sum_{i=1}^n\E\[ \left\|\overbar{\MC}_i \nabla f_i(x) - \nabla f(x^*) \right\|^2 \]
   + \frac{1}{n^2}\sum_{i\ne j}\E\< \overbar{\MC}_i \nabla f_i(x) - \nabla f(x^*), \overbar{\MC}_j \nabla f_j(x) - \nabla f(x^*) \> \\
&=
   \frac{1}{n^2}\sum_{i=1}^n\E\[ \left\|\overbar{\MC}_i \nabla f_i(x) \right\|^2 \]
   + \left\| \nabla f(x^*) \right\|^2 - 2\E\<\overbar{\MC}_i \nabla f_i(x), \nabla f(x^*)\>
   + \frac{1}{n^2}\sum_{i\ne j} \<\nabla f_i(x) - \nabla f(x^*), \nabla f_j(x) - \nabla f(x^*) \> \\
&=
   \frac{1}{n^2}\sum_{i=1}^n \left\|\nabla f_i(x) \right\|^2_{\E\[\overbar{\MC}_i^{\top}\overbar{\MC}_i\]}
   + \| \nabla f(x^*) \|^2 - 2\<\nabla f_i(x), \nabla f(x^*)\>
   + \|\nabla f(x) - \nabla f(x^*)\|^2 - \frac{1}{n^2}\sum_{i=1}^n \|\nabla f_i(x) - \nabla f(x^*)\|^2 \\
&=
   \frac{1}{n^2}\sum_{i=1}^n \left\|\ML_i^{\nicefrac{1}{2}}r_i \right\|^2_{\E\[\overbar{\MC}_i^{\top}\overbar{\MC}_i\] - \MI}
   + \frac{1}{n^2}\sum_{i=1}^n \left\|\nabla f_i(x) \right\|^2
   + \| \nabla f(x^*) \|^2 - 2\<\nabla f_i(x), \nabla f(x^*)\> \\
   &\phantom{=}
   + \|\nabla f(x) - \nabla f(x^*)\|^2 - \frac{1}{n^2}\sum_{i=1}^n \|\nabla f_i(x) - \nabla f(x^*)\|^2 \\
&=
   \|\nabla f(x) - \nabla f(x^*)\|^2
   + \frac{1}{n^2}\sum_{i=1}^n \left\|r_i \right\|^2_{\MM_i} \\
&=
   \|\nabla f(x) - \nabla f(x^*)\|^2
   + \frac{1}{n^2}\sum_{i=1}^n \left\|r_i \right\|^2_{ \ML_i^{\nicefrac{1}{2}}\ML_i^{\phalf} (\widetilde{\MP}_i\circ\ML_i) \ML_i^{\phalf}\ML_i^{\nicefrac{1}{2}} } \\
&=
   \|\nabla f(x) - \nabla f(x^*)\|^2
   + \frac{1}{n^2}\sum_{i=1}^n \left\|\ML_i^{\phalf} \nabla f_i(x) \right\|^2_{ \widetilde{\MP}_i\circ\ML_i },
\end{split}
\end{align*}
which gives as the following decomposition
\begin{equation}\label{moment-decomposition}
\E\[\|g(x) - \nabla f(x^*)\|^2\] = \|\nabla f(x) - \nabla f(x^*)\|^2 + \frac{1}{n^2}\sum_{i=1}^n \left\|\ML_i^{\dagger\nicefrac{1}{2}} \nabla f_i(x) \right\|^2_{ \widetilde{\MP}_i\circ\ML_i }.
\end{equation}
For the first term it can be bounded using convexity and smoothness of $f$, namely $\|\nabla f(x) - \nabla f(x^*)\|^2 \le 2L D_f(x,x^*)$. For the second term we proceed as follows
\begin{align}
\begin{split}\label{bound-01}
\frac{1}{n^2}\sum_{i=1}^n \left\|\ML_i^{\dagger\nicefrac{1}{2}} \nabla f_i(x) \right\|^2_{ \widetilde{\MP}_i\circ\ML_i }
&\le
     \frac{1}{n^2}\sum_{i=1}^n \lambda_{\max}(\widetilde{\MP}_i\circ\ML_i)\|\ML_i^{\dagger\nicefrac{1}{2}} \nabla f_i(x)\|^2 \\
&=
     \frac{1}{n^2}\sum_{i=1}^n \widetilde{\cL}_i \|\nabla f_i(x)\|^2_{\ML_i^{\dagger}} \\
&\le
     \frac{2}{n^2}\sum_{i=1}^n \widetilde{\cL}_i \|\nabla f_i(x) - \nabla f_i(x^*)\|^2_{\ML_i^{\dagger}}
     + \frac{2}{n^2}\sum_{i=1}^n \widetilde{\cL}_i \|\nabla f_i(x^*)\|^2_{\ML_i^{\dagger}} \\
&\le
     \frac{2}{n^2}\sum_{i=1}^n 2\widetilde{\cL}_i D_{f_i}(x,x^*)
     + \frac{2\sigma^*}{n} \\
&=  
     \frac{4\widetilde{\cL}_{\max}}{n} D_{f}(x,x^*) + \frac{2\sigma^*}{n}.
\end{split}
\end{align}
Combining these two estimates, we get
$$
\E\[\|g(x) - \nabla f(x^*)\|^2\] \le 2\(L + \frac{2\widetilde{\cL}_{\max}}{n}\) D_{f}(x,x^*) + \frac{2\sigma^*}{n}.
$$

It remains to apply the result of \citet{sigma_k}.

\subsection{Proof of Theorem \ref{thm:DIANA+}}\label{apx-thm:DIANA+}

First, we show the unbiasedness of the estimator $g(x^k)$. In (\ref{unbiased-est}), we showed unbiasedness of $\overbar{\MC}_i^k \nabla f_i(x^k)$ using inclusion $\nabla f_i(x^k)\in\range(\ML_i)$. Assume for a moment that we also have $h_i^k\in\range(\ML_i)$. Hence, in the same way we can show $\E_k\[\overbar{\MC}_i^k h_i^k\] = h_i^k$, which implies the unbiasedness of $g^k$ as
$$
\E_k\[g^k\]
= \frac{1}{n}\sum_{i=1}^n \E_k\[\overbar{\MC}_i^k \nabla f_i(x^k)\] - \E_k\[\overbar{\MC}_i^k h_i^k\] + h_i^k
= \frac{1}{n}\sum_{i=1}^n \nabla f_i(x^k)
= \nabla f(x^k).
$$
The inclusion $h_i^k\in\range(\ML_i)$ follows from the initialization $h_i^0\in\range(\ML_i)$ (see line 1 of Algorithm \ref{alg:DIANA+}) and linear update rule of $h_i^{k+1} = h_i^k + \alpha\ML_i^{\nicefrac{1}{2}}\Delta_i^k$ (see line 5 of Algorithm \ref{alg:DIANA+}). As both $\nabla f_i(x^k)$ and $h_i^k$ belong to $\range(\ML_i)$, denote $\nabla f_i(x^k)-h_i^k = \ML_i^{\nicefrac{1}{2}}r_i^k$. Next we bound

\begin{align}\label{bound-diana+}
\begin{split}
\E&\[\|g^k - \nabla f(x^*)\|^2\]
=
     \|\nabla f(x^k) - \nabla f(x^*)\|^2
     + \E\[\|g^k - \nabla f(x^k)\|^2\] \\
&\le
     2 L D_f(x^k,x^*)
     + \E\[\left\| \frac{1}{n}\sum_{i=1}^n \overbar{\MC}_i^k(\nabla f_i(x^k)-h_i^k) + h_i^k - \nabla f_i(x^k) \right\|^2\] \\
&=
     2 L D_f(x^k,x^*)
     + \frac{1}{n^2}\sum_{i=1}^n \E\[\left\| (\overbar{\MC}_i^k-\MI)\ML_i^{\nicefrac{1}{2}} r_i^k \right\|^2\] \\
&=
     2 L D_f(x^k,x^*)
     + \frac{1}{n^2}\sum_{i=1}^n \left\| r_i^k \right\|^2_{\E\[\ML_i^{\nicefrac{1}{2}}(\overbar{\MC}_i^k-\MI)^{\top}(\overbar{\MC}_i^k-\MI)\ML_i^{\nicefrac{1}{2}}\]} \\
&\overset{(\ref{def:M-1})}{=}
     2 L D_f(x^k,x^*)
     + \frac{1}{n^2}\sum_{i=1}^n \left\| r_i^k \right\|^2_{\ML_i^{\nicefrac{1}{2}}\ML_i^{\dagger\nicefrac{1}{2}} (\widetilde{\MP}_i\circ\ML_i) \ML_i^{\dagger\nicefrac{1}{2}}\ML_i^{\nicefrac{1}{2}}} \\
&=
     2 L D_f(x^k,x^*)
     + \frac{1}{n^2}\sum_{i=1}^n \left\| \ML_i^{\dagger\nicefrac{1}{2}}(\nabla f_i(x^k)-h_i^k) \right\|^2_{\widetilde{\MP}_i\circ\ML_i} \\
&\le
     2 L D_f(x^k,x^*)
     + \frac{\widetilde{\cL}_{\max}}{n^2}\sum_{i=1}^n \left\| \nabla f_i(x^k)-h_i^k \right\|^2_{\ML_i^{\dagger}} \\
&\le
     2 L D_f(x^k,x^*)
     + \frac{2\widetilde{\cL}_{\max}}{n^2}\sum_{i=1}^n \left\| \nabla f_i(x^k)-f_i(x^*) \right\|^2_{\ML_i^{\dagger}}
     + \frac{2\widetilde{\cL}_{\max}}{n^2}\sum_{i=1}^n \left\| h_i^k - \nabla f_i(x^*) \right\|^2_{\ML_i^{\dagger}} \\
&\le
     2 L D_f(x^k,x^*)
     + \frac{4\widetilde{\cL}_{\max}}{n} D_f(x^k,x^*)
     + \frac{2\widetilde{\cL}_{\max}}{n^2}\sum_{i=1}^n \left\| h_i^k - \nabla f_i(x^*) \right\|^2_{\ML_i^{\dagger}} \\
&=
     2\(L + \frac{2\widetilde{\cL}_{\max}}{n}\)D_f(x^k,x^*)
     + \frac{2\widetilde{\cL}_{\max}}{n^2}\sum_{i=1}^n \left\| h_i^k - \nabla f_i(x^*) \right\|^2_{\ML_i^{\dagger}} \\
\end{split}
\end{align}

Then we deduce a recurrence relation for the last term $\sigma^k \eqdef \frac{1}{n}\sum_{i=1}^n \left\| h_i^k - \nabla f_i(x^*) \right\|^2_{\ML_i^{\dagger}}$. For that we will need the following bounds
\begin{equation}\label{pseudo-eigen-bound}
0\preceq \ML_i^{\nicefrac{1}{2}}\ML_i^{\dagger}\ML_i^{\nicefrac{1}{2}} \preceq \MI,
\end{equation}
which can be proved via SVD and eigenvalue decompositions. Since $\ML_i$ is square, symmetric and positive semidefinite, we know that singular value decomposition and eigenvalue decompositions are the same. Let $\ML_i^{\nicefrac{1}{2}} = \MU_i\MD_i\MU_i^{\top}$, where $\MD_i$ is diagonal and $\MU_i$ is orthogonal so that $\MU_i^{\top}=\MU_i^{-1}$. Then
$$
\ML_i^{\nicefrac{1}{2}}\ML_i^{\dagger}\ML_i^{\nicefrac{1}{2}}
= \MU_i\MD_i\MU_i^{\top} \MU_i\MD_i^{\dagger 2}\MU_i^{\top} \MU_i\MD_i\MU_i^{\top}
= \MU_i \(\MD_i\MD_i^{\dagger 2}\MD_i\) \MU_i^{\top}
= \MU_i \(\MD_i\MD_i^{\dagger}\) \MU_i^{\top},
$$
which can admit eigenvalues only in $[0,1]$ since the matrix $\MD_i\MD_i^{\dagger}$ is diagonal with entries either $0$ or $1$. Denote
\begin{equation}\label{def:omega_i}
\omega_i = \lambda_{\max}\(\E\[(\MC_i^k)^2\]\) - 1 = \max_{1\le j\le d}\frac{1}{p_{i;j}}-1.
\end{equation}
and bound each summand of $\sigma^{k+1}$ as follows
\begin{align*}
\begin{split}
\E_k&\[\left\| h_i^{k+1} - \nabla f_i(x^*) \right\|^2_{\ML_i^{\dagger}}\] \\
&=
    \E_k\[\left\| h_i^k - \nabla f_i(x^*) + \alpha\overbar{\Delta}_i^k \right\|^2_{\ML_i^{\dagger}}\] \\
&=
    \left\| h_i^k - \nabla f_i(x^*) \right\|^2_{\ML_i^{\dagger}}
    + 2\alpha\<h_i^k - \nabla f_i(x^*), \nabla f_i(x^k)-h_i^k\>_{\ML_i^{\dagger}}
    + \alpha^2 \E\[\left\| \overbar{\MC}_i^k(\nabla f_i(x^k)-h_i^k) \right\|^2_{\ML_i^{\dagger}}\] \\
&=
    \left\| h_i^k - \nabla f_i(x^*) \right\|^2_{\ML_i^{\dagger}}
    + 2\alpha\<h_i^k - \nabla f_i(x^*), \nabla f_i(x^k)-h_i^k\>_{\ML_i^{\dagger}}
    + \alpha^2 \left\| \nabla f_i(x^k)-h_i^k \right\|^2_{\E\[(\overbar{\MC}_i^k)^{\top}\ML_i^{\dagger}\overbar{\MC}_i^k\]} \\
&\le
    \left\| h_i^k - \nabla f_i(x^*) \right\|^2_{\ML_i^{\dagger}}
    + 2\alpha\<h_i^k - \nabla f_i(x^*), \nabla f_i(x^k)-h_i^k\>_{\ML_i^{\dagger}}
    + \alpha^2 \left\| \nabla f_i(x^k)-h_i^k \right\|^2_{\ML_i^{\dagger\nicefrac{1}{2}}\E\[(\MC_i^k)^2\]\ML_i^{\dagger\nicefrac{1}{2}}} \\
&\le
    \left\| h_i^k - \nabla f_i(x^*) \right\|^2_{\ML_i^{\dagger}}
    + 2\alpha\<h_i^k - \nabla f_i(x^*), \nabla f_i(x^k)-h_i^k\>_{\ML_i^{\dagger}}
    + \alpha^2(1+\omega_i) \left\| \nabla f_i(x^k)-h_i^k \right\|^2_{\ML_i^{\dagger}} \\
&\le
    \left\| h_i^k - \nabla f_i(x^*) \right\|^2_{\ML_i^{\dagger}}
    + 2\alpha\<h_i^k - \nabla f_i(x^*), \nabla f_i(x^k)-h_i^k\>_{\ML_i^{\dagger}}
    + \alpha \left\| \nabla f_i(x^k)-h_i^k \right\|^2_{\ML_i^{\dagger}} \\
&\le
    (1-\alpha)\left\| h_i^k - \nabla f_i(x^*) \right\|^2_{\ML_i^{\dagger}}
    + \alpha\left\| \nabla f_i(x^k)- \nabla f_i(x^*) \right\|^2_{\ML_i^{\dagger}},
\end{split}
\end{align*}
where we used bounds $\alpha \le \frac{1}{1+\omega_i}$ and
$$
\E\[(\overbar{\MC}_i^k)^{\top}\ML_i^{\dagger}\overbar{\MC}_i^k\]
= \ML_i^{\dagger\nicefrac{1}{2}} \E\[\MC_i^k \ML_i^{\nicefrac{1}{2}}\ML_i^{\dagger}\ML_i^{\nicefrac{1}{2}} \MC_i^k\] \ML_i^{\dagger\nicefrac{1}{2}}
\preceq \ML_i^{\dagger\nicefrac{1}{2}}\E\[(\MC_i^k)^2\]\ML_i^{\dagger\nicefrac{1}{2}}.
$$

Therefore

\begin{align*}
\begin{split}
\E_k\[\sigma^{k+1}\]
&=
     \frac{1}{n} \sum_{i=1}^n \E_k\[\left\| h_i^{k+1} - \nabla f_i(x^*) \right\|^2_{\ML_i^{\dagger}}\] \\
&\le \frac{1-\alpha}{n}\sum_{i=1}^n \left\| h_i^k - \nabla f_i(x^*) \right\|^2_{\ML_i^{\dagger}}
      + \frac{\alpha}{n}\sum_{i=1}^n \left\| \nabla f_i(x^k)- \nabla f_i(x^*) \right\|^2_{\ML_i^{\dagger}} \\
&\le
     (1-\alpha)\sigma^k + \frac{2\alpha}{n}\sum_{i=1}^n D_{f_i}(x^k,x^*) \\
&=
     (1-\alpha)\sigma^k + 2\alpha D_f(x^k,x^*).
\end{split}
\end{align*}

Thus, with $\alpha \le \frac{1}{1+\omega_{\max}}$, the estimator $g^k$ of Algorithm \ref{alg:DIANA+} satisfies
\begin{align*}
\E_k\[g^k\] &= \nabla f(x^k) \\
\E_k\[\|g^k - \nabla f(x^*)\|^2\] &\le 2\(L + \frac{2\widetilde{\cL}_{\max}}{n}\)D_f(x^k,x^*) + \frac{2\widetilde{\cL}_{\max}}{n}\sigma^k \\
\E_k\[\sigma^{k+1}\] &\le (1-\alpha)\sigma^k + 2\alpha D_f(x^k,x^*).
\end{align*}

It remains to apply Theorem 4.1 \citep{sigma_k} with parameters $A = L + \frac{2}{n}\widetilde{\cL}_{\max},\; B = \frac{2}{n}\widetilde{\cL}_{\max},\; \rho=\alpha,\; C = \alpha$ and $M = \frac{4}{\alpha n}\widetilde{\cL}_{\max},\; A+CM = L + \frac{6}{n}\widetilde{\cL}_{\max},\; 1+\frac{B}{M}-\rho = 1 - \frac{\alpha}{2}$.

\subsection{Proof of Theorem \ref{thm:aDIANA+}}\label{apx-thm:aDIANA+}

Following the analysis of \citet{AccCGD}, define

\begin{align*}
Z^k \eqdef \|z^k-x^*\|^2, \qquad
&Y^k \eqdef F(y^k) - F(x^*), \qquad
W^k \eqdef F(w^k) - F(x^*), \\
H^k &\eqdef \frac{1}{n}\sum_{i=1}^n \|\nabla f_i(w^k) - h_i^k\|^2_{\ML_i^{\dagger}}.
\end{align*}

\begin{lemma}[Lemma 2, \citep{AccCGD}]\label{lem:2-adiana}
Let $\eta\le\frac{1}{2L},\; \theta_1\le\frac{1}{4},\; \theta_2=\frac{1}{2},\; \gamma = \frac{\eta}{2(\theta_1+\eta\mu)}$ and $\beta = 1-\gamma\mu$. Then
\begin{align*}
\begin{split}
\E\[Z^{k+1}\] + \frac{2\gamma\beta}{\theta_1}\E\[Y^{k+1}\]
\le & \beta Z^k + (1-\theta_1-\theta_2)\frac{2\gamma\beta}{\theta_1}Y^k + 2\gamma\beta\frac{\theta_2}{\theta_1}W^k + \frac{\gamma\eta}{\theta_1}\E\[\|g^k - \nabla f(x^k)\|^2\] \\
&  - \frac{\gamma}{4n\theta_1}\sum_{i=1}^n \|\nabla f_i(w^k) - \nabla f_i(x^k)\|^2_{\ML_i^{\dagger}}
   - \frac{\gamma}{8n\theta_1}\sum_{i=1}^n \|\nabla f_i(y^k) - \nabla f_i(x^k)\|^2_{\ML_i^{\dagger}}.
\end{split}
\end{align*}
\end{lemma}
\begin{proof}
Proof is the same as for the original lemma except we use $\ML_i$-smoothness of $f_i$ via (\ref{bregman-smooth}).
$$
f_i(u) \ge f_i(x^k) + \<\nabla f_i(x^k), u-x^k\> + \frac{1}{2}\|\nabla f_i(u) - \nabla f_i(x^k)\|^2_{\ML_i^{\dagger}}.
$$
\end{proof}

\begin{lemma}[Lemma 3, \citep{AccCGD}]\label{lem:3-adiana}
$$
\E\[W^{k+1}\] = (1-q)W^k + q Y^k.
$$
\end{lemma}

\begin{lemma}[Lemma 4, \citep{AccCGD}]\label{lem:4-adiana}
$$
\E\[\|g^k - \nabla f(x^k)\|^2\] \le \frac{2\widetilde{\cL}_{\max}}{n^2}\sum_{i=1}^n \|\nabla f_i(w^k) - \nabla f_i(x^k)\|^2_{\ML_i^{\dagger}} + \frac{2\widetilde{\cL}_{\max}}{n} H^k.
$$
\end{lemma}
\begin{proof}
Let $\nabla f_i(x^k)-h_i^k = \ML_i^{\nicefrac{1}{2}}r_i^k$. Then
\begin{align*}
\begin{split}
\E\[\|g^k - \nabla f(x^k)\|^2\]
&=
     \E\[ \left\|\frac{1}{n}\sum_{i=1}^n \overbar{\MC}_i^k(\nabla f_i(x^k)-h_i^k) - (\nabla f_i(x^k)-h_i^k)\right\|^2 \] \\
&=
     \frac{1}{n^2}\E\[ \left\|\sum_{i=1}^n (\overbar{\MC}_i^k - \MI)(\nabla f_i(x^k)-h_i^k)\right\|^2 \]
=    
     \frac{1}{n^2}\sum_{i=1}^n \E\[ \left\|(\overbar{\MC}_i^k - \MI)\ML_i^{\nicefrac{1}{2}}r_i^k\right\|^2 \] \\
&=   
     \frac{1}{n^2}\sum_{i=1}^n \left\|r_i^k\right\|^2_{\ML_i^{\nicefrac{1}{2}}\E\[(\overbar{\MC}_i^k - \MI)^{\top}(\overbar{\MC}_i^k - \MI)\]\ML_i^{\nicefrac{1}{2}}}
\overset{(\ref{def:M-1})}{=}    
     \frac{1}{n^2}\sum_{i=1}^n \left\|r_i^k\right\|^2_{\ML_i^{\nicefrac{1}{2}}\ML_i^{\dagger\nicefrac{1}{2}} (\widetilde{\MP}_i\circ\ML_i) \ML_i^{\dagger\nicefrac{1}{2}}\ML_i^{\nicefrac{1}{2}}} \\
&=   
     \frac{1}{n^2}\sum_{i=1}^n \left\|\ML_i^{\dagger\nicefrac{1}{2}}(\nabla f_i(x^k)-h_i^k)\right\|^2_{\widetilde{\MP}_i\circ\ML_i}
\le
     \frac{\widetilde{\cL}_{\max}}{n^2}\sum_{i=1}^n \left\|\nabla f_i(x^k)-h_i^k\right\|^2_{\ML_i^{\dagger}} \\
&\le 
     \frac{2\widetilde{\cL}_{\max}}{n^2}\sum_{i=1}^n \left\|\nabla f_i(x^k)-\nabla f_i(w^k)\right\|^2_{\ML_i^{\dagger}} + \frac{2\widetilde{\cL}_{\max}}{n^2}\sum_{i=1}^n \left\|\nabla f_i(w^k)-h_i^k\right\|^2_{\ML_i^{\dagger}}.
\end{split}
\end{align*}
\end{proof}

\begin{lemma}[Lemma 5, \citep{AccCGD}]\label{lem:5-adiana}
If $\alpha \le \frac{1}{1+\omega_{\max}}$, where $\omega_{\max}=\max_{1\le i \le n}\omega_i$ and $\omega_i = \max_{1\le j\le d}\frac{1}{p_{i;j}}-1$, then
$$
\E\[H^{k+1}\] \le \(1-\frac{\alpha}{2}\)H^k + \(1+\frac{2q}{\alpha}\)\frac{2q}{n}\( \sum_{i=1}^n \|\nabla f_i(w^k) - \nabla f_i(x^k)\|^2_{\ML_i^{\dagger}} + \sum_{i=1}^n \|\nabla f_i(w^k) - \nabla f_i(x^k)\|^2_{\ML_i^{\dagger}} \).
$$
\end{lemma}
\begin{proof}
We start bounding the summands of $H^{k+1}$. Let $\nabla f_i(w^k)-h_i^k = \ML_i^{\nicefrac{1}{2}}r_i^k$.
\begin{align*}
\begin{split}
\E_k&\[\left\|\nabla f_i(w^{k+1}) - h_i^{k+1}\right\|^2_{\ML_i^{\dagger}}\]
=    
     q\E_k\[\left\|\nabla f_i(y^k)-h_i^{k+1}\right\|^2_{\ML_i^{\dagger}}\] + (1-q)\E_k\[\left\|\nabla f_i(w^k)-h_i^{k+1}\right\|^2_{\ML_i^{\dagger}}\] \\
&\le 
     q\(1+\frac{2q}{\alpha}\) \left\|\nabla f_i(w^k)-\nabla f_i(y^k)\right\|^2_{\ML_i^{\dagger}}
+ \(1-q + \(1+\frac{\alpha}{2q}\)q\)\E\[\left\|\nabla f_i(w^k)-h_i^{k+1}\right\|^2_{\ML_i^{\dagger}}\] \\
&=   q\(1+\frac{2q}{\alpha}\) \left\|\nabla f_i(w^k)-\nabla f_i(y^k)\right\|^2_{\ML_i^{\dagger}}
+ \(1+\frac{\alpha}{2}\)\E\[\left\|\nabla f_i(w^k)-h_i^{k+1}\right\|^2_{\ML_i^{\dagger}}\] \\
&=   q\(1+\frac{2q}{\alpha}\) \left\|\nabla f_i(w^k)-\nabla f_i(y^k)\right\|^2_{\ML_i^{\dagger}}
+ \(1+\frac{\alpha}{2}\)\E\[\left\|(\mI-\alpha\overbar{\MC}_i^k)(\nabla f_i(w^k)-h_i^k)\right\|^2_{\ML_i^{\dagger}}\] \\
&=   q\(1+\frac{2q}{\alpha}\) \left\|\nabla f_i(w^k)-\nabla f_i(y^k)\right\|^2_{\ML_i^{\dagger}}
+ \(1+\frac{\alpha}{2}\) \left\|r_i^k\right\|^2_{ \ML_i^{\nicefrac{1}{2}}\E\[(\MI-\alpha\overbar{\MC}_i^k)^{\top}\ML_i^{\dagger}(\MI-\alpha\overbar{\MC}_i^k)\]\ML_i^{\nicefrac{1}{2}} }.
\end{split}
\end{align*}

Next, we simplify the matrix of the second term.
\begin{align*}
\begin{split}
&\ML_i^{\nicefrac{1}{2}}\E\[(\MI-\alpha\overbar{\MC}_i^k)^{\top}\ML_i^{\dagger}(\MI-\alpha\overbar{\MC}_i^k)\]\ML_i^{\nicefrac{1}{2}} \\
&=
    \E\[\ML_i^{\nicefrac{1}{2}}(\MI-\alpha\ML_i^{\dagger\nicefrac{1}{2}}\MC_i^k\ML_i^{\nicefrac{1}{2}})\ML_i^{\dagger}(\MI-\alpha\ML_i^{\nicefrac{1}{2}}\MC_i^k\ML_i^{\dagger\nicefrac{1}{2}})\ML_i^{\nicefrac{1}{2}}\] \\
&=  \E\[
    (\ML_i^{\nicefrac{1}{2}}-\alpha\ML_i^{\nicefrac{1}{2}}\ML_i^{\dagger\nicefrac{1}{2}}\MC_i^k\ML_i^{\nicefrac{1}{2}})
    \ML_i^{\dagger}
    (\ML_i^{\nicefrac{1}{2}}-\alpha\ML_i^{\nicefrac{1}{2}}\MC_i^k\ML_i^{\dagger\nicefrac{1}{2}}\ML_i^{\nicefrac{1}{2}})
    \] \\
&=
    \E\[
    \ML_i^\half\ML_i^{\dagger}\ML_i^\half
    - \alpha\ML_i^\half\ML_i^{\dagger}\ML_i^\half\MC_i^k\ML_i^\phalf\ML_i^\half \right. \\
    &\qquad\qquad\qquad\quad
    - \left. \alpha\ML_i^\half\ML_i^\phalf\MC_i^k\ML_i^\half\ML_i^{\dagger}\ML_i^\half
    + \alpha^2 \ML_i^\half\ML_i^\phalf\MC_i^k\ML_i^\half\ML_i^{\dagger}\ML_i^\half\MC_i^k\ML_i^\phalf \ML_i^\half
    \] \\
&\overset{(\ref{pseudo-eigen-bound})}{\preceq}
    \E\[
    \ML_i^\half\ML_i^{\dagger}\ML_i^\half
    - \alpha\ML_i^\half\ML_i^{\dagger}\ML_i^\half\MC_i^k\ML_i^\phalf\ML_i^\half \right. \\
    &\qquad\qquad\qquad\quad
    -  \left. \alpha\ML_i^\half\ML_i^\phalf\MC_i^k\ML_i^\half\ML_i^{\dagger}\ML_i^\half
    + \alpha^2 \ML_i^\half\ML_i^\phalf(\MC_i^k)^2\ML_i^\phalf \ML_i^\half
    \] \\
&=
    \ML_i^\half\ML_i^{\dagger}\ML_i^\half
    - \alpha\ML_i^\half\ML_i^{\dagger}\ML_i^\half\ML_i^\phalf\ML_i^\half
    - \alpha\ML_i^\half\ML_i^\phalf\ML_i^\half\ML_i^{\dagger}\ML_i^\half
    + \alpha^2 \ML_i^\half\ML_i^\phalf\E\[(\MC_i^k)^2\]\ML_i^\phalf \ML_i^\half \\
&\overset{(\ref{def:omega_i})}{\preceq}
    \ML_i^\half\ML_i^{\dagger}\ML_i^\half
    - 2\alpha\ML_i^\half\ML_i^{\dagger}\ML_i^\half
    + \alpha^2(\omega_i + 1) \ML_i^\half\ML_i^\phalf\ML_i^\phalf \ML_i^\half \\
&=
    (1 - 2\alpha + \alpha^2(\omega_i + 1))\ML_i^\half\ML_i^{\dagger}\ML_i^\half
    \\
&\preceq
    (1-\alpha)\ML_i^\half\ML_i^{\dagger}\ML_i^\half,
\end{split}
\end{align*}
where in the last step we make use of the bound $\alpha \le \frac{1}{1+\omega_{\max}} = \min_{1\le i \le n}\frac{1}{1+\omega_{i}}$. Then we finish the recurrence as follows

\begin{align*}
\begin{split}
\E_k&\[\left\|\nabla f_i(w^{k+1}) - h_i^{k+1}\right\|^2_{\ML_i^{\dagger}}\] \\
&\le
     q\(1+\frac{2q}{\alpha}\) \left\|\nabla f_i(w^k)-\nabla f_i(y^k)\right\|^2_{\ML_i^{\dagger}}
     + \(1+\frac{\alpha}{2}\) \left\|r_i^k\right\|^2_{ \ML_i^{\nicefrac{1}{2}}\E\[(\MI-\alpha\overbar{\MC}_i^k)^{\top}\ML_i^{\dagger}(\MI-\alpha\overbar{\MC}_i^k)\]\ML_i^{\nicefrac{1}{2}} } \\
&\le 
     q\(1+\frac{2q}{\alpha}\) \left\|\nabla f_i(w^k)-\nabla f_i(y^k)\right\|^2_{\ML_i^{\dagger}}
     + \(1+\frac{\alpha}{2}\)(1-\alpha) \left\|r_i^k\right\|^2_{\ML_i^\half\ML_i^{\dagger}\ML_i^\half} \\
&=
     q\(1+\frac{2q}{\alpha}\) \left\|\nabla f_i(w^k)-\nabla f_i(y^k)\right\|^2_{\ML_i^{\dagger}}
     + \(1+\frac{\alpha}{2}\)(1-\alpha) \left\|\nabla f_i(w^k)-h_i^k\right\|^2_{\ML_i^{\dagger}} \\
&\le 
     2q\(1+\frac{2q}{\alpha}\) \(\left\|\nabla f_i(w^k)-\nabla f_i(x^k)\right\|^2_{\ML_i^{\dagger}}
     + \left\|\nabla f_i(y^k)-\nabla f_i(x^k)\right\|^2_{\ML_i^{\dagger}}\) 
     + \(1-\frac{\alpha}{2}\) \left\|\nabla f_i(w^k)-h_i^k\right\|^2_{\ML_i^{\dagger}}.
\end{split}
\end{align*}

Averaging over $i\in[n]$ completes the proof.
\end{proof}

\begin{proof}[Proof of Theorem \ref{thm:aDIANA+}]
Using the 4 lemmas above and $\theta_1\le\frac{1}{4},\;\theta_2=\frac{1}{2}$, the Lyapunov function $\Psi^{k+1}$ admits the following recurrence
\begin{align*}
\E&\[\Psi^{k+1}\]
\eqdef
    \E\[Z^{k+1}
    + \frac{2\gamma\beta}{\theta_1}Y^{k+1}
    + 2\gamma\beta \frac{\theta_2(1+\theta_1)}{\theta_1 q}W^{k+1}
    + \frac{8\gamma\eta\widetilde{\cL}_{\max}}{\alpha\theta_1 n} H^{k+1}\]    \notag \\
&\overset{\text{Lemma } \ref{lem:2-adiana}}{\leq}  
    \beta Z^k
    +
    (1 - \theta_1 - \theta_2)\frac{2\gamma\beta}{\theta_1} Y^k
    +
    2\gamma\beta\frac{\theta_2}{\theta_1} W^k
    +
    \frac{\gamma\eta}{\theta_1}\E\[\|g^k - \nabla f(x^k)\|^2\] \notag\\ 
    &\qquad
    - \frac{\gamma}{4n\theta_1}\sum_{i=1}^n\|\nabla f_i(w^k) - \nabla f_i(x^k)\|^2_{\ML_i^{\dagger}}
    - \frac{\gamma}{8n\theta_1}\sum_{i=1}^n\|\nabla f_i(y^k) - \nabla f_i(x^k)\|^2_{\ML_i^{\dagger}} \notag\\
    &\qquad
    + \E\[ 2\gamma\beta \frac{\theta_2(1+\theta_1)}{\theta_1 q} W^{k+1}
    +      \frac{8\gamma\eta\widetilde{\cL}_{\max}}{\alpha\theta_1 n} H^{k+1} \] \notag\\
&\overset{\text{Lemma } \ref{lem:3-adiana}}{=}  
    \beta Z^k  +
    (1 - \theta_1 - \theta_2)\frac{2\gamma\beta}{\theta_1} Y^k
    +
    2\gamma\beta\frac{\theta_2}{\theta_1} W^k
    +
    \frac{\gamma\eta}{\theta_1}\E\[\|g^k - \nabla f(x^k)\|^2\] \notag\\ 
    &\qquad
    - \frac{\gamma}{4n\theta_1}\sum_{i=1}^n\|\nabla f_i(w^k) - \nabla f_i(x^k)\|^2_{\ML_i^{\dagger}}
    - \frac{\gamma}{8n\theta_1}\sum_{i=1}^n\|\nabla f_i(y^k) - \nabla f_i(x^k)\|^2_{\ML_i^{\dagger}} \notag\\
    &\qquad
    + 2\gamma\beta \frac{\theta_2(1+\theta_1)}{\theta_1 q}(1-q) W^k 
    + 2\gamma\beta \frac{\theta_2(1+\theta_1)}{\theta_1} Y^k
    + \E\[\frac{8\gamma\eta\widetilde{\cL}_{\max}}{\alpha\theta_1 n} H^{k+1}\] \notag\\
&\leq
    \beta Z^k
    +
    \left(1 - \frac{\theta_1}{2}\right)\frac{2\gamma\beta}{\theta_1} Y^k
    +
    \left(1 - \frac{\theta_1 q}{2}\right)2\gamma\beta \frac{\theta_2(1+\theta_1)}{\theta_1 q} W^k \notag\\
    &\qquad
    - \frac{\gamma}{4n\theta_1}\sum_{i=1}^n\|\nabla f_i(w^k) - \nabla f_i(x^k)\|^2_{\ML_i^{\dagger}}
    - \frac{\gamma}{8n\theta_1}\sum_{i=1}^n\|\nabla f_i(y^k) - \nabla f_i(x^k)\|^2_{\ML_i^{\dagger}} \notag\\
    &\qquad
    + \frac{\gamma\eta}{\theta_1}\E\[\|g^k - \nabla f(x^k)\|^2\]
    + \E\[\frac{8\gamma\eta\widetilde{\cL}_{\max}}{\alpha\theta_1 n} H^{k+1}\] \notag\\
&\overset{\text{Lemma } \ref{lem:4-adiana}}{\le}  
    \beta Z^k
    +
    \left(1 - \frac{\theta_1}{2}\right)\frac{2\gamma\beta}{\theta_1} Y^k
    +
    \left(1 - \frac{\theta_1q}{2}\right)2\gamma\beta \frac{\theta_2(1+\theta_1)}{\theta_1 q} W^k \notag\\
    &\qquad
    - \frac{\gamma}{4n\theta_1}\sum_{i=1}^n\|\nabla f_i(w^k) - \nabla f_i(x^k)\|^2_{\ML_i^{\dagger}}
    - \frac{\gamma}{8n\theta_1}\sum_{i=1}^n\|\nabla f_i(y^k) - \nabla f_i(x^k)\|^2_{\ML_i^{\dagger}} \notag\\
    &\qquad
    + \frac{2\gamma\eta\widetilde{\cL}_{\max}}{\theta_1 n^2}\|\nabla f_i(w^k) - \nabla f_i(x^k)\|^2_{\ML_i^{\dagger}}
    + \frac{2\gamma\eta\widetilde{\cL}_{\max}}{\theta_1 n} H^k
    + \E\[\frac{8\gamma\eta\widetilde{\cL}_{\max}}{\alpha\theta_1 n} H^{k+1}\] \notag\\
&\overset{\text{Lemma } \ref{lem:5-adiana}}{\le}  
    \beta Z^k
    + \left(1 - \frac{\theta_1}{2}\right)\frac{2\gamma\beta}{\theta_1} Y^k
    + \left(1 - \frac{\theta_1 q}{2}\right)2\gamma\beta \frac{\theta_2(1+\theta_1)}{\theta_1 q} W^k \notag\\
    &\qquad
    - \frac{\gamma}{4n\theta_1}\sum_{i=1}^n\|\nabla f_i(w^k) - \nabla f_i(x^k)\|^2_{\ML_i^{\dagger}}
    - \frac{\gamma}{8n\theta_1}\sum_{i=1}^n\|\nabla f_i(y^k) - \nabla f_i(x^k)\|^2_{\ML_i^{\dagger}} \notag\\
    &\qquad
    + \frac{2\gamma\eta\widetilde{\cL}_{\max}}{\theta_1 n^2}\|\nabla f_i(w^k) - \nabla f_i(x^k)\|^2_{\ML_i^{\dagger}}
    + \frac{2\gamma\eta\widetilde{\cL}_{\max}}{\theta_1 n} H^k
    + \frac{8\gamma\eta\widetilde{\cL}_{\max}}{\alpha\theta_1 n}\left(1-\frac{\alpha}{2}\right) H^{k} \notag\\
    &\qquad
    + \left(1 + \frac{2q}{\alpha}\right)\frac{16\gamma\eta\widetilde{\cL}_{\max} q}{\alpha\theta_1 n^2}
      \left(\sum_{i=1}^n\|\nabla f_i(w^k) - \nabla f_i(x^k)\|^2_{\ML_i^{\dagger}}  
    + \sum_{i=1}^n\|\nabla f_i(y^k) - \nabla f_i(x^k)\|^2_{\ML_i^{\dagger}}\right)  \notag\\
&=
    \beta Z^k
    + \left(1 - \frac{\theta_1}{2}\right)\frac{2\gamma\beta}{\theta_1} Y^k
    + \left(1 - \frac{\theta_1 q}{2}\right)2\gamma\beta \frac{\theta_2(1+\theta_1)}{\theta_1q} W^k
    + \left(1 - \frac{\alpha}{4}\right)\frac{8\gamma\eta\widetilde{\cL}_{\max}}{\alpha\theta_1 n} H^{k} \notag\\
    &\phantom{=}
    - \frac{\gamma}{n\theta_1}\left(\frac{1}{8}-\frac{2\eta\widetilde{\cL}_{\max}}{n}\right)\sum_{i=1}^n\|\nabla f_i(w^k) - \nabla f_i(x^k)\|^2_{\ML_i^{\dagger}} \notag\\
    &\phantom{=}
    - \frac{\gamma}{n\theta_1}\left(\frac{1}{8} - \(1 + \frac{2q}{\alpha}\)\frac{16\eta\widetilde{\cL}_{\max} q}{\alpha n}\right)
      \left(\sum_{i=1}^n\|\nabla f_i(w^k) - \nabla f_i(x^k)\|^2_{\ML_i^{\dagger}}  
    + \sum_{i=1}^n\|\nabla f_i(y^k) - \nabla f_i(x^k)\|^2_{\ML_i^{\dagger}}\right) \notag.
\end{align*}

To make the last two lines disappear from the recurrence, we need to make sure
$$
\frac{1}{8}-\frac{2\eta \widetilde{\cL}_{\max}}{n} \ge 0 \qquad\text{and}\qquad \frac{1}{8} - \(1 + \frac{2q}{\alpha}\)\frac{16\eta\widetilde{\cL}_{\max} q}{\alpha n} \ge 0,
$$
or equivalently
$$
\eta \le \frac{n}{16\widetilde{\cL}_{\max}} \qquad\text{and}\qquad \eta \le \frac{n}{64\widetilde{\cL}_{\max}}\cdot\frac{1}{\frac{2q}{\alpha}\(\frac{2q}{\alpha}+1\)}.
$$

Since $\alpha\le \frac{1}{\omega_{\max}+1}$ (see Lemma \ref{lem:5-adiana}) and we also need to have $\eta\le\frac{1}{2L}$ (see Lemma \ref{lem:2-adiana}), we can set
$$
\eta = \min\(\frac{1}{2L}, \frac{n}{64\widetilde{\cL}_{\max}\(2q(\omega_{\max}+1)+1\)^2}\).
$$

Therefore

\begin{align*}
\E\[\Psi^{k+1}\]
&\leq
    \beta Z^k
    + \left(1 - \frac{\theta_1}{2}\right)\frac{2\gamma\beta}{\theta_1} Y^k
    + \left(1 - \frac{\theta_1 q}{2}\right)2\gamma\beta \frac{\theta_2(1+\theta_1)}{\theta_1 q} W^k
    + \left(1 - \frac{\alpha}{4}\right)\frac{8\gamma\eta\widetilde{\cL}_{\max}}{\alpha\theta_1 n} H^{k} \\
&\leq 
    \(1-\frac{\eta \mu}{4 \theta_1}\) Z^k
    + \left(1 - \frac{\theta_1}{2}\right)\frac{2\gamma\beta}{\theta_1} Y^k
    + \left(1 - \frac{\theta_1 q}{2}\right)2\gamma\beta \frac{\theta_2(1+\theta_1)}{\theta_1q} W^k
    + \left(1 - \frac{\alpha}{4}\right)\frac{8\gamma\eta\widetilde{\cL}_{\max}}{\alpha\theta_1 n} H^{k} \\
&\leq 
    \left(1-\min\left\{\frac{\alpha}{4},\frac{q}{8},\frac{\sqrt{\eta \mu q}}{4}\right\}\right)\Psi^k,
\end{align*}

where we set $\gamma=\frac{\eta}{2(\theta_1 + \eta\mu)}$, $\beta = 1-\gamma \mu\leq 1-\frac{\eta \mu}{4 \theta_1}$ due to $\eta\mu \leq \theta_1$, and $\theta_1 = \min\left\{\frac{1}{4}, \sqrt{\frac{\eta\mu}{q}} \right\}$. After telescoping we get an $\varepsilon$-solution $\E\[\|z^k-x^*\|^2\]\le\varepsilon$ after
$$
\max\( 4(1+\omega_{\max}), \frac{8}{q}, 4\sqrt{\frac{2}{\mu q}\max\( L, \frac{32\widetilde{\cL}_{\max}\(2q(\omega_{\max}+1)+1\)^2}{n} \)} \) \log\frac{\Psi^0}{\varepsilon}
$$
iterations. Choosing $q = \min\left\{1, \frac{\max\(1,\sqrt{\frac{nL}{32\widetilde{\cL}_{\max}}}-1\)}{2(1+\omega_{\max})}\right\}$ we can simplify the above iteration complexity into
\begin{equation*}
k =
\begin{cases}
\widetilde{\cO}\(\omega_{\max} + \sqrt{\frac{\widetilde{\cL}_{\max}(1+\omega_{\max})}{\mu n}}\) & \text{if}\quad nL \le 128\widetilde{\cL}_{\max}\\
\widetilde{\cO}\(1+\omega_{\max} + \sqrt{\frac{1+\omega_{\max}}{\sqrt{n}}\frac{\sqrt{\widetilde{\cL}_{\max}L}}{\mu}}\) & \text{if}\quad  128\widetilde{\cL}_{\max} < nL \le 32\widetilde{\cL}_{\max}(2\omega_{\max}+3)^2\\
\widetilde{\cO}\(\omega_{\max} + \sqrt{\frac{L}{\mu}}\) & \text{if}\quad   32\widetilde{\cL}_{\max}(2\omega_{\max}+3)^2 < nL.
\end{cases}
\end{equation*}

Combining last two cases concludes the proof.
\end{proof}

\clearpage
\section{Improvements Over The Original Methods}

In this part we provide detailed derivations skipped in Section \ref{sec:compare}.
Recall parameters $\nu,\nu_s$ describing the distribution of matrices $\ML_i$:
\begin{equation}\label{apx-def:nu}
\nu \eqdef \frac{\sum_{i=1}^n L_i}{\max\limits_{1\le i\le n} L_i}, \; 
\nu_s \eqdef \max\limits_{1\le i\le n}\frac{\sum_{j=1}^d \ML^{\nicefrac{1}{s}}_{i;j}}{\max\limits_{1\le j\le d}\ML^{\nicefrac{1}{s}}_{i;j}},
\end{equation}
where $L_i = \lambda_{\max}(\ML_i)$ and we will choose $s=1$ or $s=2$. Let $L_{\max} = \max_{1\le i\le n} L_i$.

\subsection{Importance sampling for DCGD+}\label{apx-rem:ibcd}

Let $\tau=\E\[|S_i|\] = \sum_{j=1}^d p_{i;j}$ be the expected mini-batch size for the samplings $S_i$. Notice that convergence rate of DCGD+ depends on $\widetilde{\cL}_{\max} = \max_{1\le i\le n} \widetilde{\cL}_i$. Since each node $i\in[n]$ generates its own diagonal sketch $\MC_i$ independently from others, each node can optimize $\widetilde{\cL}_i = \lambda_{\max}(\widetilde{\MP}_i\circ\ML_i)$ independently based on local smoothness matrix $\ML_i$. In general, minimizing $\lambda_{\max}(\widetilde{\MP}_i\circ\ML_i)$ with respect to probability matrix $\widetilde{\MP}_i$ is hard. However, we can find the optimal probabilities when each node generates via an independent sampling, namely $p_{i;jl}=p_{i;j}p_{i;l}$ if $j\ne l$. Then
\begin{equation}\label{apx:cL-ind-sampling}
\lambda_{\max}(\widetilde{\MP}_i\circ\ML_i) = \max_{1\le j\le d}\(\frac{1}{p_{i;j}}-1\)\ML_{i;j},
\end{equation}
for which we can find the optimal probabilities $p_{i;j}$. To minimize the maximum term in (\ref{apx:cL-ind-sampling}), we should have $\(\nicefrac{1}{p_{i;j}}-1\)\ML_{i;j} = \rho_i$ for some $\rho_i\ge0$. Then the solution is
\begin{equation}\label{apx:ind-sampling-probs}
p_{i;j} = \frac{\ML_{i;j}}{\ML_{i;j}+\rho_i},
\end{equation}
where $\rho_i\ge0$ is the unique solution to $\sum_{j=1}^d\frac{\ML_{i;j}}{\ML_{i;j}+\rho_i}=\tau$.
The latter does not allow closed form solution for $\rho_i$, but it can be computed numerically using one dimensional solvers. Hence, we can efficiently compute the optimal probabilities (\ref{apx:ind-sampling-probs}). Moreover, we can deduce a simple upper bound for $\rho_i$
\begin{equation}\label{apx:rho-bound}
\tau = \sum_{j=1}^d\frac{\ML_{i;j}}{\ML_{i;j}+\rho_i} \le \sum_{j=1}^d\frac{\ML_{i;j}}{\rho_i} = \frac{1}{\rho_i}\sum_{j=1}^d \ML_{i;j},
\end{equation}
which gives us an upper bound for $\widetilde{\cL}_i$ as follows
\begin{equation}\label{apx:cL_i-bound}
\widetilde{\cL}_i
= \lambda_{\max}(\widetilde{\MP}_i\circ\ML_i)
= \rho_i \le \frac{1}{\tau}\sum_{j=1}^d \ML_{i;j}
\overset{(\ref{apx-def:nu})}{\le} \frac{\nu_1}{\tau}\ML_{\max}.
\end{equation}

\begin{proof}[Proof of Remark \ref{rem:ibcd}]
Using the following inequalities with respect to matrix order
\begin{equation}\label{apx:L-bound-06}
\ML \preceq \frac{1}{n}\sum_{i=1}^n\ML_i, \quad \ML_i \preceq n\ML,
\end{equation}
we bound $L$ as follows
\begin{equation}\label{apx:L-bound}
L = \lambda_{\max}\(\ML\)
\overset{(\ref{apx:L-bound-06})}{\le}
    \lambda_{\max}\(\frac{1}{n}\sum_{i=1}^n \ML_i\)
\le
    \frac{1}{n}\sum_{i=1}^n \lambda_{\max}\(\ML_i\)
=
    \frac{1}{n} \sum_{i=1}^n L_i
\overset{(\ref{apx-def:nu})}{\le}
    \frac{\nu}{n} L_{\max}.
\end{equation}

Fix $\tau = \sum_{j=1}^d p_{i;j}\in[0,d]$ expected mini-batch of coordinates for all nodes $i\in[n]$. Then, with probabilities (\ref{apx:ind-sampling-probs}) we have

\begin{equation*}
\frac{\widetilde{\cL}_{\max}}{n}
=
    \frac{1}{n}\max_{1\le i\le n} \widetilde{\cL}_i
=
    \frac{1}{n}\max_{1\le i\le n} \rho_i
\overset{(\ref{apx:cL_i-bound})}{\le}
    \frac{\nu_1}{\tau n} \ML_{\max}
\le
    \frac{\nu_1}{\tau n} L_{\max},
\end{equation*}
To get it upper bounded by $L_{\max}$, notice that $\max_{1\le j\le d}\ML_{i;j} \le \lambda_{\max}(\ML_i) = L_i$, which implies
\begin{equation}\label{2Lmax}
\ML_{\max} = \max_{1\le i\le n} \max_{1\le j\le d} \ML_{i;j} \le \max_{1\le i\le n} L_i = L_{\max}.
\end{equation}
Therefore
\begin{equation*}
L + \frac{\widetilde{\cL}_{\max}}{n} \le \( \frac{\nu}{n} + \frac{\nu_1}{\tau n} \) L_{\max}.
\end{equation*}
\end{proof}

\subsection{Importance sampling for DIANA+}\label{apx-rem:isega}

To find optimal probabilities for DIANA+, we minimize $\omega_{\max}+\frac{\widetilde{\cL}_{\max}}{\mu n}$ part of the complexity (\ref{DIANA+-complexity}) when each node uses an independent sampling as for DCGD+. Definitions of $\widetilde{\cL}_{\max}$ and $\omega_{\max}$ imply
\begin{equation}\label{eq-01}
\omega_{\max}+\frac{\widetilde{\cL}_{\max}}{\mu n}
=
    \max_{ij}\(\frac{1}{p_{i;j}}-1\)+\max_{ij}\(\frac{1}{p_{i;j}}-1\)\frac{\ML_{i;j}}{\mu n}
=
    \Theta\( \max_{ij}\(\frac{1}{p_{i;j}}-1\)\(\frac{\ML_{i;j}}{\mu n}+1\) \).
\end{equation}
Therefore it is equivalent to minimize the following for each node $i\in[n]$ independently:
\begin{equation}\label{apx:opt-prob-diana}
\max_{1\le j\le d}\(\frac{1}{p_{i;j}}-1\)\ML'_{i;j}, \quad \ML'_{i;j} \eqdef \frac{\ML_{i;j}}{\mu n} + 1 \ge 1,
\end{equation}
This can be solved in the same way as (\ref{apx:cL-ind-sampling}). The optimal probabilities are
\begin{equation}\label{apx:ind-sampling-probs-diana}
p_{i;j} = \frac{\ML'_{i;j}}{\ML'_{i;j}+\rho'_i} = \frac{\frac{\ML_{i;j}}{\mu n} + 1}{\frac{\ML_{i;j}}{\mu n}+ 1 +\rho'_i}
\end{equation}
and an upper bound for $\rho'_i$ is analogous to (\ref{apx:cL_i-bound})
\begin{equation}\label{apx:rho-bound-1}
\rho'_i \le \frac{1}{\tau}\sum_{j=1}^d \ML'_{i;j}
= \frac{1}{\tau}\sum_{j=1}^d \(\frac{\ML_{i;j}}{\mu n} + 1\)
= \frac{d}{\tau} + \frac{1}{n\tau}\sum_{j=1}^d \frac{\ML_{i;j}}{\mu}
\overset{(\ref{apx-def:nu})}{\le}
  \frac{d}{\tau} + \frac{\nu_1}{n\tau} \frac{\ML_{\max}}{\mu}
\overset{(\ref{2Lmax})}{\le}
  \frac{d}{\tau} + \frac{\nu_1}{n\tau} \frac{L_{\max}}{\mu}.
\end{equation}

\begin{proof}[Proof of Remark \ref{rem:isega}]
With probabilities (\ref{apx:ind-sampling-probs-diana}) we can upper bound the complexity (\ref{DIANA+-complexity}) as follows
\begin{align}
\begin{split}
\omega_{\max} + \frac{\widetilde{\cL}_{\max}}{\mu n}
&\overset{(\ref{eq-01})}{\le}
    2\max_{1\le i\le n}\max_{1\le j\le d}\(\frac{1}{p_{i;j}}-1\)\ML'_{i;j} \\
&\overset{(\ref{apx:ind-sampling-probs-diana})}{=}
    \frac{2}{\tau}\max_{1\le i\le n} \rho'_i \\
&\overset{(\ref{apx:rho-bound-1})}{\le}
    \frac{2d}{\tau}
    + \frac{2\nu_1}{\tau n} \frac{L_{\max}}{\mu}.
\end{split}
\end{align}

Combined with (\ref{apx:L-bound}), we have
\begin{equation*}
\omega_{\max} + \frac{L}{\mu} + \frac{\widetilde{\cL}_{\max}}{\mu n}
\le
\frac{2d}{\tau}
+ \(\frac{\nu}{n} + \frac{2\nu_1}{\tau n}\) \frac{L_{\max}}{\mu}.
\end{equation*}
\end{proof}

\begin{remark}[Improvement over standard DGD]
Let us estimate how much improvement do we get with respect to standard Distributed Gradient Descent (DGD), where each node computes full gradients $\nabla f_i(x^k)$ and sends dense updates to the server in each iteration. The iteration complexity of DGD is $\widetilde{\cO}(\frac{L}{\mu})$. To compare it against the complexity (\ref{DIANA+-complexity}) of DIANA+ we use the same setup as in previous remarks (namely, independent samplings with probabilities (\ref{ind-sampling-probs-diana}) and $\tau=\nicefrac{d}{n}$). Since $\ML_i \preceq n\ML$, we have $L_{\max} = \max_{i\in[n]} \lambda_{\max}(\ML_i)\le nL$. Hence, (\ref{L-bound-05})  implies
$$
\omega_{\max} + \frac{L}{\mu} + \frac{\widetilde{\cL}_{\max}}{\mu n}
\le 
2n + \frac{3nL}{\mu},
$$
which is $\cO(n)$ times bigger than the iteration complexity of DGD. However, in case of DGD, each node sends $n$ times more bits to the server. In total, DIANA+ and DGD have the same communication complexity in the worst case. To illustrate the best complexity DIANA+ can provide, consider the special case when $\ML_i=\ML$ for all $i\in[n]$ and $\nu_1=\cO(1)$. Then, clearly $L_{\max} = L$ and we get $\widetilde{\cO}(n+\frac{L}{\mu})$ complexity for DIANA+, yielding up to $n$ times speedup against DGD. Moreover, in case of diagonal matrices $\ML_i$, DIANA+ spends $n$ times less local computation on partial derivatives and guarantees additional $n$ times speedup.
\end{remark}

\subsection{Independent sampling for ADIANA+}

For the accelerated method ADIANA+, we construct probabilities $p_{i;j}$ similar to (\ref{apx:ind-sampling-probs}) and (\ref{apx:ind-sampling-probs-diana}) as follows
\begin{equation}\label{apx:ind-sampling-probs-adiana}
p_{i;j}
\eqdef \(\frac{ \ML'_{i;j} }{\ML'_{i;j} + \rho''_i}\)^{\half}
= \( \frac{ \frac{\ML_{i;j}}{\mu n} + 1 }{ \frac{\ML_{i;j}}{\mu n} + 1 + \rho''_i}\)^{\half},
\quad
\ML'_{i;j} = \frac{\ML_{i;j}}{\mu n} + 1 \ge 1,
\end{equation}
where $\rho''_i$ is determined uniquely from $\sum_{j=1}^d \(\frac{ \ML'_{i;j} }{\ML'_{i;j} + \rho''_i}\)^{\half} = \tau$. Notice that
\begin{equation*}
\tau
= \sum_{j=1}^d \(\frac{ \ML'_{i;j} }{\ML'_{i;j} + \rho''_i}\)^{\half}
\le \sum_{j=1}^d \(\frac{ \ML'_{i;j} }{\rho''_i}\)^{\half}
= \frac{1}{\sqrt{\rho''_i}} \sum_{j=1}^d \sqrt{\ML'_{i;j}}.
\end{equation*}

Therefore
\begin{align}
\begin{split}\label{apx:rho-bound-2}
\sqrt{\rho''_i}
&\le \frac{1}{\tau}\sum_{j=1}^d \sqrt{ \frac{\ML_{i;j}}{\mu n} + 1 }
\le \frac{1}{\tau}\sum_{j=1}^d \(\sqrt{ \frac{\ML_{i;j}}{\mu n} } + 1\)
\le \frac{d}{\tau} + \frac{1}{\tau}\sum_{j=1}^d \sqrt{ \frac{\ML_{i;j}}{\mu n} } \\
&\overset{(\ref{apx-def:nu})}{\le}
    \frac{d}{\tau} + \frac{\nu_2}{\tau}\sqrt{ \frac{\ML_{\max}}{\mu n} }
\overset{(\ref{2Lmax})}{\le}
    \frac{d}{\tau} + \frac{\nu_2}{\tau}\sqrt{ \frac{L_{\max}}{\mu n} }
\end{split}
\end{align}

\begin{proof}[Proof of Remark \ref{rem:adiana}]\label{apx-rem:adiana}
We bound terms $\omega_{\max}$ and $\frac{\cL_{\max}}{\mu n}$ using probabilities (\ref{apx:ind-sampling-probs-adiana}) as follows:
\begin{equation}
\omega_{\max}
= \max_{i,j}\(\frac{1}{p_{i;j}}-1\)
= \max_{i,j}\( \sqrt{\frac{\rho''_i}{\ML'_{i;j}} + 1} -1\)
\le \max_{i,j} \sqrt{ \frac{\rho''_i}{\ML'_{i;j}}} 
\overset{(\ref{apx:ind-sampling-probs-adiana})}{\le} \max_{i,j} \sqrt{\rho''_i}
\overset{(\ref{apx:rho-bound-2})}{\le} \frac{d}{\tau} + \frac{\nu_2}{\tau}\sqrt{ \frac{L_{\max}}{\mu n} }.
\end{equation}

\begin{equation}
\frac{\cL_{\max}}{\mu n}
\overset{(\ref{apx:cL-ind-sampling})}{=}
    \max_{i,j}\(\frac{1}{p_{i;j}}-1\)\frac{\ML_{i;j}}{\mu n}
\overset{(\ref{apx:ind-sampling-probs-adiana})}{\le}
    \max_{i,j} \frac{\sqrt{\rho''_i}\frac{\ML_{i;j}}{\mu n}}{\sqrt{\frac{\ML_{i;j}}{\mu n} + 1}}
\le \max_{i,j} \sqrt{\rho''_i} \sqrt{\frac{\ML_{i;j}}{\mu n}}
\overset{(\ref{apx:rho-bound-2})}{\le} \( \frac{d}{\tau} + \frac{\nu_2}{\tau}\sqrt{ \frac{L_{\max}}{\mu n} } \)\sqrt{\frac{L_{\max}}{\mu n}}.
\end{equation}

Let $\nu$ and $\nu_2$ are $\cO(1)$. Denote $\omega = \frac{d}{\tau},\; \kappa_i=\frac{L_i}{\mu}$ and $\kappa_{\max} = \max_{i\in[n]} \kappa_i$. Then with this notation we have
\begin{align}
\begin{split}\label{bound-adiana}
\frac{L}{\mu}
& \le \frac{\nu}{n}\kappa_{\max}
  = \cO\( \frac{\kappa_{\max}}{n} \) \\
\omega_{\max}
&\le \omega + \frac{\nu_2}{\tau}\sqrt{ \frac{\kappa_{\max}}{ n} }
     = \omega\(1 + \frac{\nu_2}{d}\sqrt{ \frac{\kappa_{\max}}{n} } \)
     = \cO\( \omega\(1 + \frac{\sqrt{\kappa_{\max}}}{d\sqrt{n}}\) \) \\
\frac{\cL_{\max}}{\mu n}
&\le \( \omega + \frac{\nu_2}{\tau}\sqrt{ \frac{\kappa_{\max}}{n} } \)\sqrt{\frac{\kappa_{\max}}{ n}}
     = \cO\( \omega\(1 + \frac{\sqrt{\kappa_{\max}}}{d\sqrt{n}}\) \frac{\sqrt{\kappa_{\max}}}{\sqrt{n}} \)
\end{split}
\end{align}

Then, in case of $nL \le \widetilde{\cL}_{\max}$, we have
\begin{align*}
\omega_{\max} + \sqrt{\omega_{\max}\frac{\widetilde{\cL}_{\max}}{\mu n}}
&=
   \cO\( \omega\(1 + \frac{\sqrt{\kappa_{\max}}}{d\sqrt{n}}\) \( 1 + \(\frac{\kappa_{\max}}{n}\)^{\nicefrac{1}{4}} \) \),
\end{align*}
which should be compared with $\cO\( \omega\(1 + \sqrt{\frac{\kappa_{\max}}{n}}\) \)$ \citep{AccCGD}. If $\kappa_{\max} = \cO(n d^2)$, then we get $\cO(\sqrt{d})$ speedup factor.
If $nL > \widetilde{\cL}_{\max}$, then
\begin{align*}
\omega_{\max} + \sqrt{\frac{L}{\mu}} &+ \sqrt{\omega_{\max}\sqrt{\frac{\widetilde{\cL}_{\max}}{\mu n}}\sqrt{\frac{L}{\mu}}} \\
&=
    \cO\(
    \omega\( 1 + \frac{\sqrt{\kappa_{\max}}}{d\sqrt{n}} \)
    + \sqrt{\frac{\kappa_{\max}}{n}}
    + \sqrt{ \omega\( 1 + \frac{\sqrt{\kappa_{\max}}}{d\sqrt{n}} \) \sqrt{\frac{\kappa_{\max}}{n}}
             \sqrt{ \omega\( 1 + \frac{\sqrt{\kappa_{\max}}}{d\sqrt{n}} \) \sqrt{\frac{\kappa_{\max}}{n}} }
             }
    \) \\
&=
    \cO\(
    \omega\( 1 + \frac{\sqrt{\kappa_{\max}}}{d\sqrt{n}} \)
    + \sqrt{\frac{\kappa_{\max}}{n}}
    + \[ \omega\( 1 + \frac{\sqrt{\kappa_{\max}}}{d\sqrt{n}} \) \sqrt{\frac{\kappa_{\max}}{n}} \]^{\nicefrac{3}{4}}
    \),
\end{align*}
which should be compared with $\omega + \kappa_{\max} + \omega^{\nicefrac{3}{4}} n^{\nicefrac{1}{4}}\sqrt{\frac{\kappa_{\max}}{n}}$ \citep{AccCGD}. If $\kappa_{\max} = \cO(n d^2)$, then we get $\cO(\sqrt{n})$ times smaller second term and $\cO\( (nd)^{\nicefrac{1}{4}} \)$ times smaller third term.

\end{proof}

\clearpage

\section{Variance Reduction: ISEGA+}\label{apx-sec:ISEGA}

In this part we apply our redesign to another variance reduced method called ISEGA \citep{99KFP,GJS-HR}. At the core of ISEGA, the mechanism for variance reduction is based on SEGA method \citep{SEGA-FP}. The key difference between ISEGA and DIANA is that ISEGA updates the control variates $h$ more aggressively using projection instead of the mere $\alpha$-step towards the projection used in DIANA.  Adapting our matrix-smoothness-aware sparsification to ISEGA, we define the update rule of control vectors $h_i^k$ as follows ({\bf for now assume} $\ML_i$ is invertible)
\begin{align*}
h_i^{k+1}
&= \argmin_{\substack{h\in\range(\ML_i) \\ \MC_i^k\ML_i^\phalf\nabla f_i(x^k) = \MC_i^k\ML_i^\phalf h}} \|h-h_i^k\|^2_{\ML_i^\dagger} \\
&= h_i^k + \ML_i\ML_i^\phalf\MC_i^k \( \MC_i^k\ML_i^\phalf\ML_i\ML_i^\phalf\MC_i^k \)^\dagger\MC_i^k\ML_i^\phalf (\nabla f_i(x^k) - h_i^k) \\
&= h_i^k + \ML_i^\half\MC_i^k \( \MC_i^k\MC_i^k \)^\dagger\MC_i^k\ML_i^\phalf (\nabla f_i(x^k) - h_i^k) \\
&= h_i^k + \ML_i^\half \Mdiag(\MP_i) \MC_i^k\ML_i^\phalf (\nabla f_i(x^k) - h_i^k).
\end{align*}

Note that the update rule in DIANA+ has the form
$$
h_i^{k+1} = h_i^k + \alpha \ML_i^\half \MC_i^k\ML_i^\phalf (\nabla f_i(x^k) - h_i^k)
$$
for some fixed scalar $\alpha>0$, and thus is more conservative.  Note that we choose the gradient estimator to be the same $g_i^k = h_i^k + \ML_i^\half \MC_i^k\ML_i^\phalf (\nabla f_i(x^k) - h_i^k)$.  The method is presented as Algorithm~\ref{alg:ISEGA+}.

\begin{algorithm}[H]
\begin{algorithmic}[1]
\STATE \textbf{Input:} Initial point $x^0\in\R^d$, initial shifts $h_i^0\in\R^d$, current point $x^k$, step size parameter $\gamma$ and $\alpha$, sketch $\MC_i^k$ and $\overbar{\MC}_i^k \eqdef \ML_i^\half\MC_i^k\ML_i^\phalf$, current shifts $h_1^k,\dots,h_n^k$ and $h^k \eqdef \frac{1}{n}\sum_{i=1}^n h_i^k$.
\STATE \textbf{on} each node
\STATE \quad get $x^k$ from the server
\STATE \quad send sparse update $\Delta_i^k = \MC_i^k\ML_i^\phalf (\nabla f_i(x^k) - h_i^k)$
\STATE \quad $g_i^k = h_i^k + \ML_i^\half\Delta_i^k$
\STATE \quad $h_i^{k+1} = h_i^k + \ML_i^\half \Mdiag(\MP_i) \Delta_i^k$
\STATE \textbf{on} server
\STATE \quad get sparse updates $\Delta_i^k$ from each node
\STATE \quad $g^k = \frac{1}{n}\sum_{i=1}^n g_i^k =  h^k + \frac{1}{n}\sum_{i=1}^n \ML_i^\half\Delta_i^k$
\STATE \quad $x^{k+1} = \prox_{\gamma R}(x^k - \gamma g^k)$
\STATE \quad $h^{k+1} = \frac{1}{n}\sum_{i=1}^n h_i^{k+1} = h^k + \frac{1}{n}\sum_{i=1}^n \ML_i^\half \Mdiag(\MP_i) \Delta_i^k$
\end{algorithmic}
\caption{\sc ISEGA+}
\label{alg:ISEGA+}
\end{algorithm}

Note that we can not obtain the convergence rate of ISEGA+ directly from the framework of~\citet{sigma_k}. Instead,  to get the tight convergence rate, we shall cast it as an instance of GJS method~\citep{GJS-HR}.  Theorem~\ref{thm:ISEGA} provides the result -- we can see that the worst case complexity is identical to DIANA+. In terms of the practical performance, we expect ISEGA+ to outperform DIANA+ due to the more aggressive update rule of control variates.

\begin{theorem}\label{thm:ISEGA}
Suppose that $\gamma \leq \frac{1}{\frac{4\widetilde{\cL}_{\max}}{n} +2L + \mu (\omega_{\max}+1)}$. Then, we have
$$
\E[\Psi^k] \leq (1-\gamma \mu) \Psi^0,
$$
where 
$$
\Psi^k \eqdef  \| x^{k} - x^*\|^2 + \frac{\gamma}{2n} \sum_{i=1}^n\|\phi_i^{k} - \ML_i^\phalf\nabla f_i(x^*)\|^2_{\Mdiag(\MP_i)^{-1}}
$$
and $\phi_i^k \eqdef \ML_i^\phalf h_i^k$.  Consequently, the overall complexity of ISEGA+ is
$$
\tilde{\cO}\left(\omega_{\max} +\frac{L}{\mu} +  \frac{\widetilde{\cL}_{\max}}{n\mu}  \right).
$$
\end{theorem}

\begin{proof}

The proof can be seen as a special case of the generalized Jacobian sketching theory of~\citet{GJS-HR}. For the sake of clarity, we provide a specialized proof here.

Note first that by~\eqref{bound-diana+},we have
$$
\E\[\|g^k - \nabla f(x^*)\|^2\]
\le
2\(L + \frac{2\widetilde{\cL}_{\max}}{n}\)D_f(x^k,x^*)
     + \frac{2\widetilde{\cL}_{\max}}{n^2}\sum_{i=1}^n \left\| \phi_i^k - \ML_i^\phalf\nabla f_i(x^*) \right\|^2.
$$
Similarly, we have

\begin{eqnarray*}
&& \E\[\|\phi_i^{k+1} - \ML_i^\phalf\nabla f_i(x^*)\|^2_{\Mdiag(\MP_i)^{-1}}\]
\\
&=&
\E\[\|\phi_i^{k}  +\Mdiag(\MP_i) \MC_i^k(\ML_i^\phalf \nabla f_i(x^k)-\phi_i^{k})  - \ML_i^\phalf\nabla f_i(x^*)\|^2_{\Mdiag(\MP_i)^{-1}}\]
\\
&=&
\E\[\|(\MI  -\Mdiag(\MP_i) \MC_i^k)(\phi_i^{k}-  \ML_i^\phalf\nabla f_i(x^*))+ \Mdiag(\MP_i) \MC_i^k \ML_i^\phalf (\nabla f_i(x^k)-\nabla f_i(x^*))  \|^2_{\Mdiag(\MP_i)^{-1}}\]
\\
&=&
\E\[\|(\MI  -\Mdiag(\MP_i) \MC_i^k)\Mdiag(\MP_i)^{-\frac12}(\phi_i^{k}-  \ML_i^\phalf\nabla f_i(x^*))+ \Mdiag(\MP_i)^\half \MC_i^k \ML_i^\phalf (\nabla f_i(x^k)-\nabla f_i(x^*))  \|^2\]
\\
&=&
\E\[\|(\MI  -\Mdiag(\MP_i) \MC_i^k)\Mdiag(\MP_i)^{-\frac12}(\phi_i^{k}-  \ML_i^\phalf\nabla f_i(x^*))\|^2\]
+ \E\[\|\Mdiag(\MP_i)^\half\MC_i^k \ML_i^\phalf (\nabla f_i(x^k)-\nabla f_i(x^*))  \|^2\]
\\
&=&
\| \phi_i^{k} - \ML_i^\phalf\nabla f_i(x^*)\|^2_{\Mdiag(\MP_i)^{-1}-\MI}
+\|\ML_i^\phalf (\nabla f_i(x^k)-\nabla f_i(x^*))  \|^2
\\
&\leq&
\| \phi_i^{k} - \ML_i^\phalf\nabla f_i(x^*)\|^2_{\Mdiag(\MP_i)^{-1}-\MI}
+2D_{f_i}(x^k, x^*)
\end{eqnarray*}
 and therefore 

\begin{equation} \label{eq:bhajxbhjs}
 \E\[\frac1n \sum_{i=1}^n\|\phi_i^{k+1} - \ML_i^\phalf\nabla f_i(x^*)\|^2_{\Mdiag(\MP_i)^{-1}}\]
 \leq 
\frac1n \sum_{i=1}^n \| \phi_i^{k} - \ML_i^\phalf\nabla f_i(x^*)\|^2_{\Mdiag(\MP_i)^{-1}-\MI}
+2D_{f}(x^k, x^*)
\end{equation}

Following the classical analysis of SGD (i.e.,  proof of Lemma C.1 of~\citet{sigma_k}), we get

\begin{eqnarray*}
\E\[\| x^{k+1} - x^*\|^2 \] &=& (1-\gamma \mu) | x^{k} - x^*\|^2 -2\gamma D_f(x^k, x^*) + \gamma^2 \E\[\|g^k - \nabla f(x^*)\|^2\]
\\
&\leq &
 (1-\gamma \mu) | x^{k} - x^*\|^2 -2\gamma\left(1  -\gamma \(L + \frac{2\widetilde{\cL}_{\max}}{n}\) \right) D_f(x^k, x^*) 
     \\
     && \qquad + \frac{2\widetilde{\cL}_{\max}\gamma^2}{n^2}\sum_{i=1}^n \left\| \phi_i^k - \ML_i^\phalf\nabla f_i(x^*) \right\|^2. 
\end{eqnarray*}
Adding $\frac{\gamma}2$-multiple of~\eqref{eq:bhajxbhjs} to the above, we get
\begin{eqnarray}
&&
\nonumber
 \E\[\| x^{k+1} - x^*\|^2 \] + \frac{\gamma}{2}  \E\[\frac1n \sum_{i=1}^n\|\phi_i^{k+1} - \ML_i^\phalf\nabla f_i(x^*)\|^2_{\Mdiag(\MP_i)^{-1}}\]
\\ \nonumber
&& \qquad \leq 
 (1-\gamma \mu) \| x^{k} - x^*\|^2 -2\gamma\left(\frac12  -\gamma \(L + \frac{2\widetilde{\cL}_{\max}}{n}\) \right) D_f(x^k, x^*) 
     \\
     && \qquad \qquad + \frac{2\widetilde{\cL}_{\max}\gamma^2}{n^2}\sum_{i=1}^n \left\| \phi_i^k - \ML_i^\phalf\nabla f_i(x^*) \right\|^2 
     + 
     \frac{\gamma}{2n} \sum_{i=1}^n \| \phi_i^{k} - \ML_i^\phalf\nabla f_i(x^*)\|^2_{\Mdiag(\MP_i)^{-1}-\MI} \label{eq:fohuanj}
\end{eqnarray}

Next, note that we have 
\begin{multline}\label{eq:dnjajnd}
\frac{2\widetilde{\cL}_{\max}\gamma^2}{n^2}\sum_{i=1}^n \left\| \phi_i^k - \ML_i^\phalf\nabla f_i(x^*) \right\|^2  + 
     \frac{\gamma}{2n} \sum_{i=1}^n \| \phi_i^{k} - \ML_i^\phalf\nabla f_i(x^*)\|^2_{\Mdiag(\MP_i)^{-1}-\MI} \\
     \leq 
          \frac{(1-\gamma\mu)\gamma}{2n} \sum_{i=1}^n \| \phi_i^{k} - \ML_i^\phalf\nabla f_i(x^*)\|^2_{\Mdiag(\MP_i)^{-1}}
\end{multline}
since it is equivalent to
\begin{equation*}
\frac{4\widetilde{\cL}_{\max}\gamma}{n}\sum_{i=1}^n \left\| \phi_i^k - \ML_i^\phalf\nabla f_i(x^*) \right\|^2  + 
     \gamma\mu \sum_{i=1}^n \| \phi_i^{k} - \ML_i^\phalf\nabla f_i(x^*)\|^2_{\Mdiag(\MP_i)^{-1}}
     \leq 
        \sum_{i=1}^n \| \phi_i^{k} - \ML_i^\phalf\nabla f_i(x^*)\|^2,
\end{equation*}
which holds since $\gamma \leq \frac{1}{\frac{4\widetilde{\cL}_{\max}}{n} + \mu (\omega_{\max}+1)}$.

To finish the proof, it remains to plug~\eqref{eq:dnjajnd} into~\eqref{eq:fohuanj}, use that $\gamma\leq   \frac{1}{\frac{4\widetilde{\cL}_{\max}}{n} + 2L} $ and unroll the recurrence. 

\end{proof}

\clearpage

\section{Variance Reduction with Bi-directional Compression: DIANA++}\label{apx-thm:DIANA++}

In this method, the master server applies compression in its turn with sketch $\MC$ independently. Thus, we maintain an additional control vector $H^k$, which helps to reduce the variance coming from the master's sparsification. Moreover, nodes keep track of $H^k$ just like the central server.

\begin{algorithm}[H]
\begin{algorithmic}[1]
\STATE \textbf{Input:} Initial point $x^0\in\R^d$, initial shifts $h_i^0\in\range(\ML_i),\; H^0\in\range(\ML)$, current point $x^k$, step size parameter $\gamma,\alpha$ and $\beta$, sketch $\MC_i^k$ and $\overbar{\MC}_i^k \eqdef \ML_i^\half\MC_i^k\ML_i^\phalf$, current shifts $h_1^k,\dots,h_n^k,H^k$ and $h^k \eqdef \frac{1}{n}\sum_{i=1}^n h_i^k$.
\STATE \textbf{on} each node
\STATE \quad {\bf send} sparse update $\Delta_i^k = \MC_i^k\ML_i^\phalf (\nabla f_i(x^k) - h_i^k)$
\STATE \quad $\overbar{\Delta}_i^k = \ML_i^\half\Delta_i^k,\; g_i^k = h_i^k + \overbar{\Delta}_i^k, h_i^{k+1} = h_i^k + \alpha\overbar{\Delta}_i^k$
\STATE \textbf{on} server
\STATE \quad {\bf get} sparse updates $\Delta_i^k$ from each node
\STATE \quad $\overbar{\Delta}^k = \frac{1}{n}\sum_{i=1}^n \overbar{\Delta}_i^k = \frac{1}{n}\sum_{i=1}^n \ML_i^\half\Delta_i^k$
\STATE \quad $g^k = \overbar{\Delta}^k + h^k = \frac{1}{n}\sum_{i=1}^n \overbar{\MC}_i^k \(\nabla f_i(x^k) - h_i^k\) + h_i^k$

\STATE \quad {\bf send} sparse update $\delta^k = \MC^k\ML^\phalf(g^k - H^k)$
\STATE \quad $\overbar{\delta}^k = \ML^\half\delta^k,\; \hat{g}^k = H^k + \overbar{\delta}^k = H^k + \overbar{\MC}^k \(g^k - H^k\)$

\STATE \quad $x^{k+1} = \prox_{\gamma R}(x^k - \gamma \hat{g}^k)$
\STATE \quad $h^{k+1} = h^k + \alpha\overbar{\Delta}^k$
\STATE \quad $H^{k+1} = H^k + \beta\overbar{\delta}^k$

\STATE \textbf{on} each node
\STATE \quad {\bf get} $\delta^k$ from the server
\STATE \quad reconstruct $\overbar{\delta}^k = \ML^\half\delta^k,\; \hat{g}^k = H^k + \overbar{\delta}^k = H^k + \overbar{\MC}^k \(g^k - H^k\)$
\STATE \quad $x^{k+1} = \prox_{\gamma R}(x^k - \gamma \hat{g}^k)$
\STATE \quad $H^{k+1} = H^k + \beta\overbar{\delta}^k$
\end{algorithmic}
\caption{\sc DIANA++}
\label{alg:DIANA++}
\end{algorithm}

\begin{theorem}\label{thm:DIANA++}
Let Assumptions \ref{asm:Li-smooth-convex} and \ref{asm:mu-convex} hold and assume that each node generates its own diagonal sketch $\MC_i$ independently from others. The master server, in its turn, generates $\MC$ independently from the nodes. Then, Algorithm \ref{alg:DIANA++} has the following iteration complexity
$$
\cO\( \frac{1}{\min\( \alpha -  \beta\theta', \beta \)} + \frac{\alpha + \beta\theta+\beta\theta'}{\min\( \alpha -  \beta\theta', \beta \)} \( \frac{L}{\mu} + \frac{\widetilde{\cL}}{\mu} + \frac{\widetilde{\cL}\widetilde{\cL}'_{\max}}{n\mu} + \frac{\widetilde{\cL}_{\max}}{n\mu} \) \),
$$
where we made the following notations
\begin{align*}
\begin{split}
\theta &\eqdef \frac{ n\widetilde{\cL} }{ \widetilde{\cL}_{\max} + 2\widetilde{\cL}\widetilde{\cL}'_{\max} } \le \frac{n}{2\widetilde{\cL}'_{\max}}, \quad
\theta' \eqdef \frac{2\theta}{n}\widetilde{\cL}'_{\max} \le 1 \in [0,1]\\
\widetilde{\cL}'_{\max} & \eqdef \max_{1\le i \le n} \lambda_{\max}\( \widetilde{\MP}_i\circ(\ML_i^\half\ML^\dagger\ML_i^\half) \), \quad
\widetilde{\cL} \eqdef \lambda_{\max}\( \widetilde{\MP}\circ\ML \)
\end{split}
\end{align*}

with bounds $\alpha \le \frac{1}{1+\omega_{\max}} = \max_{i\in[n]}\max_{j\in[d]}\frac{1}{p_{i;j}}$ and $\beta \le \frac{1}{1+\omega} = \max_{j\in[d]}\frac{1}{p_{j}}$.
\end{theorem}

\begin{remark}
Note that, when master does not compress the messages, then we have $\widetilde{\MP} = \bm{0}$. This implies the same complexity we had for DIANA+ as quantities $\widetilde{\cL},\; \theta,\; \theta'$ are all become zeros.
\end{remark}

\begin{proof}
The proof follows the same structure as for DIANA+, with additional variance reduction process introduced for the master server. Analogously, we start bounding the following second moment:
\begin{align}\label{bound-diana++-1}
\begin{split}
\E&\[\|\hat{g}^k - \nabla f(x^*)\|^2\]
=
     \E\[\|\hat{g}^k - g^k\|^2\]
     + \E\[\|g^k - \nabla f(x^*)\|^2\].
\end{split}
\end{align}

We can bound the second term as it was done in (\ref{bound-diana+}):
$$
\E\[\|g^k - \nabla f(x^*)\|^2\] \le
     2\(L + \frac{2\widetilde{\cL}_{\max}}{n}\)D_f(x^k,x^*)
     + \frac{2\widetilde{\cL}_{\max}}{n^2}\sum_{i=1}^n \left\| h_i^k - \nabla f_i(x^*) \right\|^2_{\ML_i^{\dagger}}.
$$

Then we decompose the first term $\E\[\|\hat{g}^k - g^k\|^2\]$ into two as follows:
\begin{align}\label{bound-diana++-2}
\begin{split}
\E\[\|\hat{g}^k - g^k\|^2\]
&=
     \E\[\| \overbar{\MC}^k(g^k - H^k) - (g^k - H^k) \|^2\] \\
&=
     \| g^k - H^k \|^2_{\E\[(\MI - \overbar{\MC}^k)^\top(\MI - \overbar{\MC}^k)\]} \\
&=
     \| g^k - H^k \|^2_{\ML^\phalf(\widetilde{\MP}\circ\ML)\ML^\phalf} \\
&\le
     \widetilde{\cL}\| g^k - H^k \|^2_{\ML^\dagger} \\
&\le
     2\widetilde{\cL}\| g^k - \nabla f(x^*) \|^2_{\ML^\dagger}
     + 2\widetilde{\cL}\| H^k - \nabla f(x^*) \|^2_{\ML^\dagger}.
\end{split}
\end{align}

To bound each of the two summands in (\ref{bound-diana++-2}),  we derive the analogue of (\ref{eq:transform}).

\begin{align}\label{bound-diana++-3}
\begin{split}
\E&\[\ML_i^{\half} \(\overbar{\MC}_i - \MI\)^{\top}\ML^\dagger\(\overbar{\MC}_i - \MI\) \ML_i^{\half}\] \\
&= 
    \E\[ \ML_i^\half \(\ML_i^{\phalf}\MC_i\ML_i^\half - \MI\)\ML^\dagger\(\ML_i^\half\MC_i\ML_i^{\phalf} - \MI\) \ML_i^{\half} \] \\
&= 
    \E\[ \ML_i^{\nicefrac{1}{2}} \(\ML_i^{\phalf}\MC_i (\ML_i^\half\ML^\dagger\ML_i^\half) \MC_i\ML_i^{\phalf}
    - \ML_i^{\phalf}\MC_i\ML_i^{\nicefrac{1}{2}}\ML^\dagger
    - \ML^\dagger\ML_i^{\nicefrac{1}{2}}\MC_i\ML_i^{\phalf}
    + \ML^\dagger\) \ML_i^{\nicefrac{1}{2}} \] \\
&
\overset{(\ref{E[CLC]})}{=} 
    \ML_i^{\nicefrac{1}{2}} \(\ML_i^{\phalf}\(\overbar{\MP}_i \circ (\ML_i^\half\ML^\dagger\ML_i^\half)\)\ML_i^{\phalf}
    - \ML_i^{\phalf}\ML_i^{\half}\ML^\dagger
    - \ML^\dagger\ML_i^{\nicefrac{1}{2}}\ML_i^{\phalf}
    + \ML^\dagger\) \ML_i^{\nicefrac{1}{2}} \\
&=  \ML_i^{\nicefrac{1}{2}} \ML_i^{\phalf}\(\overbar{\MP}_i \circ (\ML_i^\half\ML^\dagger\ML_i^\half)\)\ML_i^{\phalf}
    \ML_i^{\nicefrac{1}{2}} 
    - \ML_i^\half\ML^\dagger\ML_i^\half \\
&=  \ML_i^{\nicefrac{1}{2}} \ML_i^{\phalf}\(\widetilde{\MP}_i \circ (\ML_i^\half\ML^\dagger\ML_i^\half)\)\ML_i^{\phalf}
    \ML_i^{\nicefrac{1}{2}}.
\end{split}
\end{align}

Then we bound them as follows. First, we have

\begin{align}\label{bound-diana++-4}
\begin{split}
\E&\[\|g^k - \nabla f(x^*)\|^2_{\ML^\dagger}\]
=
     \|\nabla f(x^k) - \nabla f(x^*)\|^2_{\ML^\dagger}
     + \E\[\|g^k - \nabla f(x^k)\|^2_{\ML^\dagger}\] \\
&\le
     2 D_f(x^k,x^*)
     + \E\[\left\| \frac{1}{n}\sum_{i=1}^n \overbar{\MC}_i^k(\nabla f_i(x^k)-h_i^k) + h_i^k - \nabla f_i(x^k) \right\|^2_{\ML^\dagger}\] \\
&=
     2 D_f(x^k,x^*)
     + \frac{1}{n^2}\sum_{i=1}^n \E\[\left\| (\overbar{\MC}_i^k-\MI)\ML_i^{\nicefrac{1}{2}} r_i^k \right\|^2_{\ML^\dagger}\] \\
&=
     2 D_f(x^k,x^*)
     + \frac{1}{n^2}\sum_{i=1}^n \left\| r_i^k \right\|^2_{\E\[\ML_i^{\nicefrac{1}{2}}(\overbar{\MC}_i^k-\MI)^{\top}\ML^\dagger(\overbar{\MC}_i^k-\MI)\ML_i^{\nicefrac{1}{2}}\]} \\
&\overset{(\ref{bound-diana++-3})}{=}
     2 D_f(x^k,x^*)
     + \frac{1}{n^2}\sum_{i=1}^n \left\| r_i^k \right\|^2_{\ML_i^{\nicefrac{1}{2}}\ML_i^{\dagger\nicefrac{1}{2}} (\widetilde{\MP}_i\circ(\ML_i^\half\ML^\dagger\ML_i^\half)) \ML_i^{\dagger\nicefrac{1}{2}}\ML_i^{\nicefrac{1}{2}}} \\
&=
     2 D_f(x^k,x^*)
     + \frac{1}{n^2}\sum_{i=1}^n \left\| \ML_i^{\dagger\nicefrac{1}{2}}(\nabla f_i(x^k)-h_i^k) \right\|^2_{\widetilde{\MP}_i\circ(\ML_i^\half\ML^\dagger\ML_i^\half)} \\
&\le
     2 D_f(x^k,x^*)
     + \frac{\widetilde{\cL}'_{\max}}{n^2}\sum_{i=1}^n \left\| \nabla f_i(x^k)-h_i^k \right\|^2_{\ML_i^{\dagger}} \\
&\le
     2 D_f(x^k,x^*)
     + \frac{2\widetilde{\cL}'_{\max}}{n^2}\sum_{i=1}^n \left\| \nabla f_i(x^k)-f_i(x^*) \right\|^2_{\ML_i^{\dagger}}
     + \frac{2\widetilde{\cL}'_{\max}}{n^2}\sum_{i=1}^n \left\| h_i^k - \nabla f_i(x^*) \right\|^2_{\ML_i^{\dagger}} \\
&\le
     2 D_f(x^k,x^*)
     + \frac{4\widetilde{\cL}'_{\max}}{n} D_f(x^k,x^*)
     + \frac{2\widetilde{\cL}'_{\max}}{n^2}\sum_{i=1}^n \left\| h_i^k - \nabla f_i(x^*) \right\|^2_{\ML_i^{\dagger}} \\
&=
     2\(1 + \frac{2\widetilde{\cL}'_{\max}}{n}\)D_f(x^k,x^*)
     + \frac{2\widetilde{\cL}'_{\max}}{n^2}\sum_{i=1}^n \left\| h_i^k - \nabla f_i(x^*) \right\|^2_{\ML_i^{\dagger}} \\
\end{split}
\end{align}

Then, for the control vectors $H^k$ at the master, we have

\begin{align*}
\begin{split}
\E_k&\[\left\| H^{k+1} - \nabla f(x^*) \right\|^2_{\ML^{\dagger}}\] \\
&=
    \E_k\[\left\| H^k - \nabla f(x^*) + \beta\overbar{\delta}^k \right\|^2_{\ML^{\dagger}}\] \\
&=
    \left\| H^k - \nabla f(x^*) \right\|^2_{\ML^{\dagger}}
    + 2\beta\E\[\<H^k - \nabla f(x^*), g^k-H^k\>_{\ML^{\dagger}}\]
    + \beta^2 \E_k\[\left\| \overbar{\MC}^k(g^k-H^k) \right\|^2_{\ML^{\dagger}}\] \\
&=
    \left\| H^k - \nabla f(x^*) \right\|^2_{\ML^{\dagger}}
    + 2\beta\E_k\[\<H^k - \nabla f(x^*), g^k-H^k\>_{\ML^\dagger}\]
    + \beta^2 \E_k\[\left\| g^k-H^k \right\|^2_{\E\[(\overbar{\MC}^k)^{\top}\ML^{\dagger}\overbar{\MC}^k\]}\] \\
&\le
    \left\| H^k - \nabla f(x^*) \right\|^2_{\ML^{\dagger}}
    + 2\beta\E_k\[\<H^k - \nabla f(x^*), g^k-H^k\>_{\ML^{\dagger}}\]
    + \beta^2 \E_k\[\left\| g^k-H^k \right\|^2_{\ML^\phalf\E\[(\MC^k)^2\]\ML^\phalf}\] \\
&\le
    \left\| H^k - \nabla f(x^*) \right\|^2_{\ML^{\dagger}}
    + 2\beta\E_k\[\<H^k - \nabla f(x^*), g^k-H^k\>_{\ML^{\dagger}}\]
    + \beta^2(1+\omega) \E_k\[\left\| g^k-H^k \right\|^2_{\ML^{\dagger}}\] \\
&\le
    \left\| H^k - \nabla f(x^*) \right\|^2_{\ML^{\dagger}}
    + 2\beta\E_k\[\<H^k - \nabla f(x^*), g^k-H^k\>_{\ML^{\dagger}}\]
    + \beta \E_k\[\left\| g^k-H^k \right\|^2_{\ML^{\dagger}}\] \\
&=
    (1-\beta)\left\| H^k - \nabla f(x^*) \right\|^2_{\ML^{\dagger}}
    + \beta\E_k\[\left\| g^k- \nabla f(x^*) \right\|^2_{\ML^{\dagger}}\] \\
&\le
    (1-\beta)\left\| H^k - \nabla f(x^*) \right\|^2_{\ML^{\dagger}}
    + 2\beta\(1 + \frac{2\widetilde{\cL}'_{\max}}{n}\)D_f(x^k,x^*)
    + \frac{2\beta\widetilde{\cL}'_{\max}}{n^2}\sum_{i=1}^n \left\| h_i^k - \nabla f_i(x^*) \right\|^2_{\ML_i^{\dagger}} \\
\end{split}
\end{align*}

Now, for some $\theta$ (to be defined later), let
$$
\sigma^k \eqdef \frac{1}{n}\sum_{i=1}^n\|h_i^k - \nabla f_i(x^*)\|^2_{\ML_i^\dagger} + \theta \|H^k - \nabla f(x^*)\|^2_{\ML^\dagger}.
$$

Then, we have
\begin{align*}
\begin{split}
   \E&\[\|\hat{g}^k - \nabla f(x^*)\|^2\] \\
&\overset{(\ref{bound-diana++-1})}{=} 
     \E\[\|\hat{g}^k - g^k\|^2\]
     + \E\[\|g^k - \nabla f(x^*)\|^2\] \\
&\overset{(\ref{bound-diana++-2})}{\le} 
     2\widetilde{\cL}\E\[\| g^k - \nabla f(x^*) \|^2_{\ML^\dagger}\]
     + 2\widetilde{\cL}\| H^k - \nabla f(x^*) \|^2_{\ML^\dagger}
     + \E\[\|g^k - \nabla f(x^*)\|^2\] \\
&\overset{(\ref{bound-diana++-4})}{\le} 
     4\widetilde{\cL}\(1 + \frac{2\widetilde{\cL}'_{\max}}{n}\)D_f(x^k,x^*)
     + \frac{4\widetilde{\cL}\widetilde{\cL}'_{\max}}{n^2}\sum_{i=1}^n \left\| h_i^k - \nabla f_i(x^*) \right\|^2_{\ML_i^{\dagger}} \\
     &\quad
     + 2\(L + \frac{2\widetilde{\cL}_{\max}}{n}\)D_f(x^k,x^*)
     + \frac{2\widetilde{\cL}_{\max}}{n^2}\sum_{i=1}^n \left\| h_i^k - \nabla f_i(x^*) \right\|^2_{\ML_i^{\dagger}} \\
     &\quad
     + 2\widetilde{\cL}\| H^k - \nabla f(x^*) \|^2_{\ML^\dagger} \\
&=
     2\(L + 2\widetilde{\cL} + \frac{4\widetilde{\cL}\widetilde{\cL}'_{\max}}{n} + \frac{2\widetilde{\cL}_{\max}}{n}\)D_f(x^k,x^*) \\
     &\quad
     + \(\frac{4\widetilde{\cL}\widetilde{\cL}'_{\max}}{n} + \frac{2\widetilde{\cL}_{\max}}{n}\) \frac{1}{n}\sum_{i=1}^n \left\| h_i^k - \nabla f_i(x^*) \right\|^2_{\ML_i^{\dagger}}
     + 2\widetilde{\cL}\| H^k - \nabla f(x^*) \|^2_{\ML^\dagger} \\
&=
     2\(L + 2\widetilde{\cL} + \frac{4\widetilde{\cL}\widetilde{\cL}'_{\max}}{n} + \frac{2\widetilde{\cL}_{\max}}{n}\)D_f(x^k,x^*)
     + \(\frac{4\widetilde{\cL}\widetilde{\cL}'_{\max}}{n} + \frac{2\widetilde{\cL}_{\max}}{n}\) \sigma^k,
\end{split}
\end{align*}
with the following choice of $\theta$:
$$
\theta \eqdef \frac{ n\widetilde{\cL} }{ \widetilde{\cL}_{\max} + 2\widetilde{\cL}\widetilde{\cL}'_{\max} } \le \frac{n}{2\widetilde{\cL}'_{\max}}, \quad
\theta' \eqdef \frac{2\theta}{n}\widetilde{\cL}'_{\max} \le 1.
$$

For the control vectors $h_i^k$ and $H^k$, we deduce
\begin{align*}
\begin{split}
   \E&\[\sigma^{k+1}\] \\
&\le    (1-\alpha)\frac{1}{n}\sum_{i=1}^n\|h_i^k - \nabla f_i(x^*)\|^2_{\ML_i^\dagger} + 2\alpha D_f(x^k,x^*) \\
        &\quad
        + (1-\beta)\theta\left\| H^k - \nabla f(x^*) \right\|^2_{\ML^{\dagger}}
        + 2\beta\theta\(1 + \frac{2\widetilde{\cL}'_{\max}}{n}\)D_f(x^k,x^*)
        + \frac{2\beta\theta\widetilde{\cL}'_{\max}}{n^2}\sum_{i=1}^n \left\| h_i^k - \nabla f_i(x^*) \right\|^2_{\ML_i^{\dagger}} \\
&=
        \( 1 - \alpha +  \frac{2\beta\theta\widetilde{\cL}'_{\max}}{n} \) \frac{1}{n}\sum_{i=1}^n\|h_i^k - \nabla f_i(x^*)\|^2_{\ML_i^\dagger}
        + (1-\beta)\theta\left\| H^k - \nabla f(x^*) \right\|^2_{\ML^{\dagger}} \\
        &\quad
        + 2\( \alpha + \beta\theta\(1 + \frac{2\widetilde{\cL}'_{\max}}{n}\) \) D_f(x^k,x^*) \\
&\le
        \max\( 1 - \alpha +  \frac{2\beta\theta\widetilde{\cL}'_{\max}}{n}, 1 - \beta \) \sigma^k
        + 2\( \alpha + \beta\theta\(1 + \frac{2\widetilde{\cL}'_{\max}}{n}\) \) D_f(x^k,x^*) \\
&=
        \max\( 1 - \alpha +  \beta\theta', 1 - \beta \) \sigma^k
        + 2\( \alpha + \beta\theta+\beta\theta' \) D_f(x^k,x^*).
\end{split}
\end{align*}

Thus the constants from \citep{sigma_k} are as follows

\begin{align*}
\begin{split}
A &= L + 2\widetilde{\cL} + \frac{4\widetilde{\cL}\widetilde{\cL}'_{\max}}{n} + \frac{2\widetilde{\cL}_{\max}}{n} \\
B &= \frac{4\widetilde{\cL}\widetilde{\cL}'_{\max}}{n} + \frac{2\widetilde{\cL}_{\max}}{n} = \frac{2\widetilde{\cL}}{\theta}\\
C & = \alpha + \beta\theta+\beta\theta' \\
\rho &= \min\( \alpha -  \beta\theta', \beta \).
\end{split}
\end{align*}

Let $M = \frac{2B}{\rho}$, and note that $B\theta = 2\widetilde{\cL}$ and $B\theta'=\frac{4\widetilde{\cL}\widetilde{\cL}'_{\max}}{n}$. Then
\begin{align*}
\begin{split}
A+CM
&= A + 2B\frac{\alpha + \beta\theta+\beta\theta'}{\min\( \alpha -  \beta\theta', \beta \)} \\
&= \cO\( \frac{\alpha + \beta\theta+\beta\theta'}{\min\( \alpha -  \beta\theta', \beta \)} \( L + \widetilde{\cL} + \frac{\widetilde{\cL}\widetilde{\cL}'_{\max}}{n} + \frac{\widetilde{\cL}_{\max}}{n} \) \). \\
1+\frac{B}{M}-\rho &= 1 - \frac{\rho}{2} = 1 - \frac{1}{2}\min\( \alpha -  \beta\theta', \beta \).
\end{split}
\end{align*}
\end{proof}


\end{document}